\documentclass[12pt]{article}
\usepackage{amsmath}
\usepackage{graphicx}
\usepackage{enumerate}
\usepackage{url} 
\usepackage{lineno}	

\newcommand{\blind}{0}

\usepackage{graphicx}
\usepackage{tikz}
\usetikzlibrary{fit,positioning,arrows,automata,calc}
\tikzset{
  main/.style={circle, minimum size = 5mm, thick, draw =black!80, node distance = 10mm},
  connect/.style={-latex, thick},
  box/.style={rectangle, draw=black!100}
}
\usepackage{amsmath}
\usepackage{amsfonts}
\usepackage{amsthm}
\usepackage{mathtools}
\mathtoolsset{showonlyrefs}
\usepackage[colorinlistoftodos]{todonotes}
\usepackage{bm}
\usepackage{breqn}
\usepackage{physics}
\usepackage{enumitem}
\usepackage{bbold}
\usepackage{epigraph}

\usepackage{csquotes}
\usepackage{hyperref}       
\usepackage{url}            
\usepackage{booktabs}       
\usepackage{multirow}
\usepackage{amsfonts}       
\usepackage{nicefrac}       
\usepackage{microtype}      
\usepackage{color}
\usepackage{scalerel}
\usepackage[toc,page]{appendix}
\usepackage[font=small,labelfont=bf]{caption}
\usepackage{blindtext}
\usepackage{amssymb}
\usepackage{mathtools}
\usepackage{dsfont}
\allowdisplaybreaks
\usepackage{float}
\usepackage{graphicx}
\usepackage{subcaption}
\graphicspath{ {images/} }
\usepackage{letltxmacro}%
\usepackage{thmtools,thm-restate}
\usepackage{relsize}
\usepackage{algpseudocode}
\usepackage{enumitem}
\usepackage[ruled]{algorithm2e}
\usepackage[moderate]{savetrees} 

\newtheorem{definition}{Definition}[section]
\newtheorem{cor}{Corollary}[section]
\newtheorem{lemma}{Lemma}[section]
\newtheorem{remark}{Remark}[section]
\newtheorem{corollary}{Corollary}[section]
\newtheorem{prop}{Proposition}[section]
\newtheorem{proposition}{Proposition}[section]
\newtheorem{theorem}{Theorem}[section]
\newtheorem{property}{Property}[section]

\newcommand{\imag}{\boldsymbol{i}}
\newcommand{\poly}{\mathsf{poly}}

\newcommand{\Reals}{\mathbb{R}}
\newcommand{\Ind}[1]{\mathbb{1}\left\{ #1 \right\} }
\newcommand{\Prob}[1]{\mathbb{P}\left( #1 \right) }
\newcommand{\bA}{\boldsymbol{A}}
\newcommand{\bB}{\boldsymbol{B}}

\newcommand{\bE}{\boldsymbol{E}}
\newcommand{\bH}{\boldsymbol{H}}

\newcommand{\bR}{\boldsymbol{R}}
\newcommand{\bS}{\boldsymbol{S}}

\newcommand{\bU}{\boldsymbol{U}}
\newcommand{\bV}{\boldsymbol{V}}
\newcommand{\bW}{\boldsymbol{W}}
\newcommand{\bX}{\boldsymbol{X}}

\newcommand{\bZ}{\boldsymbol{Z}}

\newcommand{\bP}{\boldsymbol{P}}
\newcommand{\bQ}{\boldsymbol{Q}}
\newcommand{\bAO}{\boldsymbol{A}^{\Omega}}
\newcommand{\bBO}{\boldsymbol{B}^{\Omega}}

\newcommand{\YO}{Y^{\Omega}}
\newcommand{\bhAO}{\widehat{\bA}^{\Omega}}

\newcommand{\cE}{\mathcal{E}}
\newcommand{\hcE}{\widehat{\cE}}
\newcommand{\cG}{\mathcal{G}}
\newcommand{\cH}{\mathcal{H}}

\newcommand{\tZ}{\widetilde{Z}}

\newcommand{\btZ}{\widetilde{\bZ}}

\newcommand{\hY}{\widehat{Y}}
\newcommand{\bhA}{\widehat{\bA}}

\newcommand{\Econd}{\cE}
\newcommand{\F}{\mathcal{F}}
\newcommand{\D}{\mathcal{D}}
\newcommand{\cL}{\mathcal{L}}
\newcommand{\distas}[1]{\mathbin{\overset{#1}{\kern\z@\sim}}}%
\newsavebox{\mybox}\newsavebox{\mysim}
\newcommand{\distras}[1]{%
  \savebox{\mybox}{\hbox{\kern3pt$\scriptstyle#1$\kern3pt}}%
  \savebox{\mysim}{\hbox{$\sim$}}%
  \mathbin{\overset{#1}{\kern\z@\resizebox{\wd\mybox}{\ht\mysim}{$\sim$}}}%
}
\newcommand{\hrho}{\widehat{\rho}}

\newcommand{\bSigma}{\boldsymbol{\Sigma}}

\newcommand{\hbeta}{\widehat{\beta}}

\newcommand{\bbA}{\bA}
\newcommand{\bbE}{\bE}
\newcommand{\bbH}{\bH}
\newcommand{\bbP}{\bP}
\newcommand{\bbX}{\bX}
\newcommand{\bbZ}{\bZ}
\newcommand{\bbhA}{\bhA}

\newcommand{\Ex}{\mathbb{E}}
\newcommand{\Pb}{\mathbb{P}}
\newcommand{\Rb}{\mathbb{R}}
\newcommand{\Nb}{\mathbb{N}}

\newcommand{\pcr}{\text{PCR}}
\newcommand{\hsvt}{\text{HSVT}}
\newtoggle{FIGURES}
\toggletrue{FIGURES}

\begin{document}

\def\spacingset#1{\renewcommand{\baselinestretch}%
{#1}\small\normalsize} \spacingset{1}


\if0\blind
{
  \title{\bf On Robustness of Principal Component Regression}
  \author{Anish Agarwal,
    Devavrat Shah,
    Dennis Shen,
    Dogyoon Song \\
    MIT}
   \date{}
  \maketitle 
  \vspace{-5mm}
} \fi

\if1\blind
{
  \title{\bf On Robustness of Principal Component Regression}
  \date{}
  \maketitle
} \fi

\begin{abstract}
	Principal component regression (PCR) is a simple, but powerful and ubiquitously utilized method. 
Its effectiveness is well established when the covariates exhibit low-rank structure.
However, its ability to handle settings with noisy, missing, and mixed-valued, i.e., discrete and continuous, covariates is not understood and remains an important open challenge. 
As the main contribution of this work we establish the robustness of PCR, without any change, in this respect and provide meaningful finite-sample analysis. 

To do so, we establish that PCR is equivalent to performing linear regression after pre-processing the covariate matrix via hard singular value thresholding (HSVT). 
As a result, in the context of counterfactual analysis using observational data, we show PCR is equivalent to the recently proposed robust variant of the synthetic control method, known as robust synthetic control (RSC). 
As an immediate consequence, we obtain finite-sample analysis of the RSC estimator that was previously absent. 
As an important contribution to the synthetic controls literature, we establish that an (approximate) linear synthetic control exists in the setting of a generalized 
factor model, or latent variable model; traditionally in the literature, the existence of a synthetic control needs to be assumed to exist as an axiom.
We further discuss a surprising implication of the robustness property of PCR with respect to noise, i.e., PCR can learn a good predictive model even if 
the covariates are tactfully transformed to preserve differential privacy. 

Finally, this work advances the state-of-the-art analysis for HSVT by establishing stronger guarantees with respect to the $\ell_{2, \infty}$-norm rather than the frobenius norm as is commonly done in the matrix estimation literature, which may be of interest in its own right.

\end{abstract}

\noindent%
\vfill

\newpage

\spacingset{2} 

\section{Introduction} \label{sec:intro}
A common thread of many modern datasets is that they are high-dimensional, and often noisy and partially observed. 
When such datasets are used for regression, this means that {\em both} the response variables (also known as the label of target) and the covariates (also known as features) are corrupted.
This setting is known in the statistics literature as error-in-variables regression. 
Another common feature of most real-world datasets are that they are mixed valued, i.e., contain both discrete and continuous data, which further complicates the regression procedure.
Within this context,  we are interested in developing a better understanding of a popular prediction method known as principal component regression (PCR).
Indeed, PCR's ability to handle settings with noisy, missing, and mixed-valued covariates is not understood and remains an important open challenge \cite{recent-survey}. 

A further motivation of this work is to connect the error-in-variables setting to the exciting and growing literature on synthetic controls (SC), a standard framework in econometrics (and beyond) to make counterfactual predictions utilizing only observational data (\cite{abadie1, abadie2, Hsiao12, imbens, Xu2016GeneralizedSC, athey, athey1, amjad, LiBell17, Li18, hsiao2018, asc, ark}).
Broadly speaking, there is a notion of a ``target'' and ``donor'' units, for which we collect observations over time.
While the donors units remain under control, the target undergoes an intervention at some time period.
Here, the goal is to estimate what would have happened to the target unit had it also remained under control.
Towards answering this question, standard SC methods build a synthetic model of the target unit using observations associated with the donor units.
In the language of regression, the target unit observations represent the response variables and the donor unit observations represent the covariates.
In the SC literature, the observations associated with both the target and donor units are assumed to be {\em noisily} observed due to the presence of idiosyncratic shocks at each time step.
As a result, SC can be seen as an instance of error-in-variables regression; more generally, panel data settings, where one collects measurements over time, can also be viewed through this error-in-variables lens.

As the main contribution of this work, we establish the effectiveness of PCR, without any change, for  error-in-variables regression and provide meaningful finite-sample analysis for both in- and out-of-sample prediction error.
Given the connection between error-in-variables regression and SC, our analysis also implies that using PCR in this context leads to it implicitly de-noising the observations we have of the donor units, which are corrupted by idiosyncratic shocks.
Thus, we advocate for PCR's usage in panel data settings.

\subsection{Problem Statement}\label{sec:intro_aim}

In a typical prediction problem setup, we are given access to a labeled dataset $\{(Y_i, \bA_{i, \cdot})\}$ over $i \ge 1$; 
here, $Y_i \in \Reals$ represents the response variable we wish to predict, and $\bA_{i, \cdot} \in \Reals^{1 \times p}$ represents the associated covariate to be utilized in the prediction process. 
Let $N \geq 1$ denote the total number of observations, where the number of predictors $p$ can possibly exceed $N$.
Let $\bbA \in \Reals^{N \times p}$ denote the matrix of true covariates.

\smallskip \noindent 
{\bf Error-in-variables.} Rather than perfectly observing the covariates $\bbA$, 
the error-in-variables setting only reveals 
a corrupted version denoted as  $\bbZ \in \Reals^{N \times p}$. 
That is,
the $(i,j)$-th entry of $\bbZ$, denoted as $Z_{ij}$, is defined as $A_{ij} + \eta_{ij}$ with probability $\rho$ and $\star$ with probability $1 - \rho$, for some $\rho \in (0,1]$; here, $\star$ denotes a missing value and $\eta_{ij}$ denotes the noise in the $(i,j)$-th entry. 
In other words, each entry $Z_{ij}$ is observed with probability $\rho$, independently of other entries; 
however, even when observed, $Z_{ij}$ is still only a noisy instance of the true $A_{ij}$. 

\smallskip \noindent 
{\bf Approximate linear model.}
We assume the response variables are generated as follows:  
for $i \in [N]$, the random response $Y_i$ is associated with the covariate $\bbA_{i, \cdot} $ via 
\begin{align}\label{eq:regression_model_general}  
	Y_i & = \bbA_{i, \cdot} \beta^* + \epsilon_i + \phi_i,	
\end{align}
where $\beta^* \in \mathbb{R}^{p}$ is the unknown latent model parameter, 
$\epsilon_i \in \Reals$ denotes zero mean response noise with variance bounded by $\sigma^2$, 
and $\phi_i \in \Reals$ is the linear model misspecification, or mismatch, error; 
for simplicity, we assume the mismatch error is deterministic.
Additionally, the observed response variables $Y_i$ are restricted to a subset of the $N$ observations. 
More formally, we denote $\Omega \subset [N]$, with $ | \Omega | = n < N$, as the index set of observed responses, i.e., we observe $Y_i$ for $i \in \Omega$. 

\smallskip \noindent 
{\bf Goal.} Given noisy observations of all $N$ covariates $\{\bZ_{1, \cdot}, \dots, \bZ_{N, \cdot} \}$ and a subset of  response variables $\{Y_i: \, i \in \Omega\}$, our aim is to produce an estimate $\widehat{Y} \in \Reals^N$ so that the prediction error is minimized. 
Specifically, we measure performance in terms of the {\em training error}
\begin{align} \label{eq:train_error}
	 \text{MSE}_{\Omega}(\widehat{Y}) = \frac{1}{n} \,\,  \Ex \left[ \sum_{i \in \Omega} ( \widehat{Y}_{i} - \bA_{i, \cdot} \beta^*)^2 \right]
	\end{align}
and {\em testing error}
\begin{align} \label{eq:test_error}
		\text{MSE}(\widehat{Y}) = \frac{1}{N} \,\, \Ex \left[ \sum_{i=1}^N ( \widehat{Y}_{i} - \bA_{i, \cdot} \beta^* )^2 \right].
\end{align}
We note that for the bound $\text{MSE}(\hY)$ to be meaningful, $|\Omega^c| = N - n$ (the size of the test set) should be of the same order as that of the the training set $|\Omega| = n$.

\smallskip
\noindent
{\em Transductive semi-supervised learning.} 
It is worth remarking that in \eqref{eq:test_error}, the algorithm is given access to the observations associated with the covariates for {\em both} training and testing data during the training procedure. 
Of course, however, the algorithm does {\em not} access the test response variables. 
This is commonly referred to in the literature as transductive semi-supervised learning; here, we want to infer the response variables for the specific unlabeled data.
%
Traditionally, it is assumed that a statistical estimator only has access to the training covariates and response variables during the model learning process. 
The reason we consider a transductive learning setting is a consequence of the nature of the algorithm of interest, PCR. 
Specifically, PCR pre-processes the covariates using PCA, which changes the training procedure if only a subset of the covariates are utilized. 
Therefore, to allow for a meaningful evaluation, it is natural to allow the algorithm to have access to {\em all} available covariate information.
%
Indeed, as we will discuss in Section \ref{sec:pcr_sc}, as well as Appendices \ref{sec:private} and \ref{sec:mixed}, it is natural to have access to all covariates in many important real-world applications. 

\subsection{Contributions}


\noindent 
\textbf{PCR implicitly de-noises.} 
As the main contribution of this work, we argue that PCR, without any change, is robust to noise and missing values in the observed covariates. 
In particular, despite only having access to $\bZ$, 
we show the training error of PCR scales (up to logarithmic factors) as  $ \rho^{-4} r / \min(n, p) + \|\phi\|_2^2 / n $, where $\rho$ denotes the fraction of observed (noisy) covariates and $r$ is the rank of $\bA$ (Corollary \ref{thm:training_pcr}).
That is, PCR {\em implicitly de-noises $\bZ$} by projecting it onto the subspace spanned by the top $r$ right singular vectors.
We note that the prediction error rate of $r / n$ for the training data matches (up to log factors) the minimax rate achievable by ordinary least squares (OLS) if one had perfectly observed the true underlying covariate matrix $\bA$ (see \cite{wainwright2019high} and references therein).
%

%
We extend our results to the case where $\bA$ is only approximately low-rank (Theorem \ref{thm:training_pcr_generic} and Corollaries \ref{cor:training_pcr_generic} and \ref{cor:training_pcr_generic_LVM}). 
To the best of our knowledge, under this setting,
there do not exist prediction consistency results for OLS or regularized variants thereof such as Lasso and Ridge, without making additional assumptions on the sparsity of $\beta^*$.
This remains true even if $\bA$ is perfectly observed.
Thus, the first step in PCR of finding a low-dimensional representation is likely crucial for this setting, and further motivated if the covariates are noisily observed. 
Given the ubiquity of approximately low-rank matrices in real-world datasets, it reinforces the utility and robustness of applying PCR in practice.

Moreover, we note that PCR does {\em not} require any knowledge about the underlying noise model that corrupts the covariates in order to to have vanishing train and test errors. 
Despite the exciting recent advancement in the high-dimensional error-in-variables literature, such as in \cite{loh_wainwright, cocolasso, tsybakov_2}, the current inventory of methods require knowledge of the underlying covariate noise model (in particular, exact knowledge of its second moment of matrix) and do not provide finite sample guarantees for train or test error.
We do note, however, that the aim of these previous papers is to estimate the latent linear model parameter $\beta^*$ (assuming it is sparse), rather than to analyze prediction errors.
For a detailed comparison, see Appendix \ref{sec:lit_review}.


\smallskip
\noindent 
\textbf{PCR implicitly regularizes.}
We define an appropriate notion of generalization error for the transductive learning setting we consider. 
We establish that the testing prediction error of PCR is bounded above by the training error plus a term that scales as $k^{5/2}/\sqrt{n}$, where $k$ is the number of retained principal components (Theorem \ref{thm:test_pcr}).
Our testing error result provides a systematic way to select the correct number of principal components in a data-driven manner, i.e., to choose the value of $k$ that minimizes the training error plus the generalization penalty term $k^{5/2} / \sqrt{n}$.

Our test error analysis utilizes the standard framework of Rademacher complexity (see \cite{Bartlett_2003} and references therein). 
However, there are two crucial differences that we need to overcome in order to obtain sharp, meaningful bounds. 
First, our notion of generalization is different from that of the traditional setup since the noisy test covariates (but not responses) are included in the training process, which requires careful analysis.
Second, we argue that the Rademacher complexity under PCR scales with the dimensionality of the number of principle components utilized, denoted as $k$, rather than the ambient covariate dimension $p$. 
To do so, we identify the Rademacher complexity class of PCR with $k$-sparse $\beta$'s.

\smallskip
\noindent 
{\bf PCR applications.}
%
We discuss the robustness of PCR to contaminated covariates by analyzing its ability to learn a predictive model when only differentially private covariates are available. 
In particular, we find that it is feasible for PCR to achieve good prediction accuracy and simultaneously maintain differential privacy of the covariates (Appendix \ref{sec:private}). 
We also describe how the robustness of PCR allows it to seamlessly utilize mixed valued covariates under a general probabilistic model (Appendix \ref{sec:mixed}).

%
\smallskip \noindent 
{\bf SC literature.}
First, we note that regardless of method used to construct synthetic controls, the fundamental hypothesis that drives these prior works is the existence of a linear relationship between the  target and donors;
in fact, the original proposal of \cite{abadie1, abadie2} suggests restricting the linear model coefficients to be non-negative and sum to one, i.e., a convex combination.
However, it is not clear when such a hypothesis holds.
Second, meaningful finite-sample analysis of the mean-squared post-intervention error of SC has remained elusive.  
We tackle these two questions via our results on PCR. 
Towards the first question, we establish that (approximate) synthetic controls exist under a generalized factor model (also known as a latent variable model).
Here, the measurement associated with a given unit and time is a sufficiently smooth function of the latent unit and time factors.
Therefore in a general sense, a synthetic control almost always exists and need not be assumed as a hypothesis or axiom (see Proposition \ref{prop:linear_comb}).
Towards the second question, we show that PCR is identical to a recently proposed SC estimator known as robust synthetic control (RSC) \cite{amjad}.
Hence, we immediately establish meaningful training (pre-intervention) and testing (post-intervention) error guarantees for RSC (see Theorem \ref{thm:mse_sc}). 

\subsection{Organization of Paper}
In Section \ref{sec:PCR}, we describe the PCR algorithm. 
Section \ref{sec:results} then details the various training and test prediction error bounds for PCR and the conditions under which they hold. 
In Section \ref{sec:pcr_sc}, we formally connect PCR to SC. 
%
In Appendix \ref{sec:lit_review}, we do a detailed comparison with previous related works.
In Appendices \ref{sec:private} and \ref{sec:mixed}, we discuss the application of PCR for differentially private regression and mixed valued covariates, respectively.
The remaining appendices are to prove our  theoretical results.

\section{Principal Component Regression} \label{sec:PCR} 
%
We recall the description of PCR, as in \cite{pcr_jolliffe}. 
We suggest a minor modification of PCR in the presence of missing data where we simply re-scale the observed covariates by the inverse of the fraction of observed data.

\noindent 
\textbf{Algorithm.} 
Let $\hrho$ denote the fraction of observed entries in $\bbZ$, i.e., $\hrho = 1/(Np) \sum_{i=1}^{N} \sum_{j=1}^p \mathbb{1}(Z_{ij} \neq \star) \vee 1/(Np)$. 
Let $\btZ \in \mathbb{R}^{N \times p}$ represent the rescaled version of $\bbZ$, where every unobserved value $\star$ is replaced by $0$, i.e., $\tZ_{ij}  = Z_{ij} / \hrho$ if $Z_{ij} \neq \star$ and $0$ otherwise.

The singular value decomposition (SVD) of $\btZ$ is denoted as $\btZ = \bU \bS \bV^T = \sum_{i=1}^N s_i u_i v_i^T$, where $\bU \in \Reals^{N \times N}$, $\bS \in \Reals^{N \times p}$, and $\bV \in \Reals^{p \times p}$. 
Without loss of generality, assume that the singular values $s_i$'s are arranged in decreasing order, i.e., $s_1 \ge \dots \ge s_N \ge 0$. 
Note that $\bU = [u_1, \dots, u_N]$ and $\bV = [v_1, \dots, v_p]$ are orthogonal matrices, i.e., the $u_i$'s and $v_j$'s are orthonormal vectors. 

For any $k \in [N]$, let $\bU_k = [u_1, \dots, u_k]$, $\bV_k = [v_1, \dots, v_k]$, and $\bS_k = \text{diag}(s_1, \dots, s_k)$. 
Then, the $k$-dimensional representation of $\btZ$, as per PCA, is given by $\bZ^{\text{PCR}, k} = \btZ \bV_k$.
Let $\beta^{\text{PCR}, k} \in \mathbb{R}^k$ be the solution to the linear regression problem under $\bZ^{\pcr, k}$, i.e., $\beta^{\pcr, k}$ is the minimizer of
$$
{\sf minimize} ~ \sum_{i\in \Omega} \left(Y_i -  \bZ^{\text{PCR}, k}_{i\cdot} w \right)^2 ~\text{over}~w \in \mathbb{R}^k.
$$
Then, the estimated $N$-dimensional response vector $\hY^{\text{PCR}, k} = \bZ^{\text{PCR}, k} \beta^{\text{PCR}, k}$. 

\smallskip
\noindent 
\textbf{Intuition.}
%
%
Using all the noisily observed observed covariates, PCR first finds a $k$ dimensional representation of the covariate matrix using the method of principal component analysis (PCA), where $k$ might be much smaller than $p$.
Specifically, PCA projects every covariate $\bZ_{i, \cdot}$ onto the subspace spanned by the top $k$ right singular vectors of the observed covariate matrix, $\bZ$. 
PCR then uses the $k$-dimensional features to perform linear regression.

\subsection{Connecting PCR to the Matrix Estimation Literature}\label{sec:intro_pcr_ME}
To establish our results, we study PCR via its equivalence with performing linear regression after pre-processing covariates via hard singular value thresholding (HSVT), as described below.

\smallskip
\noindent
\textbf{Linear regression with covariate pre-processing via HSVT.}
Given any $\lambda > 0$, we define the map $\text{HSVT}_{\lambda}: \mathbb{R}^{N \times p} \to \mathbb{R}^{N \times p}$, which simply shaves off the input matrix's singular values that are below the threshold $\lambda$. 
Precisely, given a matrix $\bB \in \mathbb{R}^{N \times p}$, denote its SVD as $\bB = \sum_{i=1}^{N} \sigma_i x_i y_i^T$, and let $\text{HSVT}_{\lambda}(\bB) = \sum_{i = 1}^{N} \sigma_i \mathbb{1}(\sigma_i \ge \lambda) x_i y_i^T$.
For any $k \in [N]$, given $\btZ$ as before, define $\bZ^{\text{HSVT}, k} = \text{HSVT}_{s_{k}}(\btZ)$.
Let $\beta^{\text{HSVT}, k} \in \mathbb{R}^p$ be a solution of linear regression under $\bZ^{\hsvt, k}$, i.e., $\beta^{\hsvt, k}$ is the minimizer of
\begin{align}
{\sf minimize} ~&~ \sum_{i\in \Omega} \left(Y_i -  \bZ^{\text{HSVT}, k}_{i\cdot} w \right)^2 ~\text{over}~w \in \mathbb{R}^p.
\end{align}
Then, the estimated $N$-dimensional response vector $\hY^{\text{HSVT}, k} = \bZ^{\text{HSVT}, k} \beta^{\text{HSVT}, k}$. 

\smallskip
\noindent
\textbf{Equivalence with PCR.}
We now state a simple, yet key relation between PCR and the algorithm above.
Precisely, the two algorithms produce identical estimated response vectors. 
\begin{proposition}\label{prop:equivalence}
For any $k \leq N$, 
$ \hY^{\emph{PCR}, k}  = \hY^{\emph{HSVT}, k} .$
\end{proposition}
\noindent 
By establishing the equivalence above, it allows us to analyze PCR through the growing matrix estimation/completion literature, of which HSVT is one of the most commonly analyzed methods.
In fact, there is significant literature establishing that HSVT is a noise-model-agnostic method that recovers the ground-truth matrix given a sparse, noisy observation of it, e.g., see \cite{Chatterjee15}

\smallskip
\noindent
\textbf{$\|\cdot\|_{2, \infty}$-norm error bound for HSVT.}
The limitation of the current results concerning HSVT is that they only establish its estimation accuracy in terms of the mean-squared error or expected squared Frobenius norm of the error matrix. 
To establish our above mentioned results on the prediction error of PCR, it seems necessary to bound the expected squared $\ell_{2, \infty}$-norm of the error matrix (see Lemmas \ref{lemma:mcse_hvst} and \ref{lemma:mcse_hvst_LVM}), which is a stronger guarantee than the Frobenius norm. 
To see this, let $\bE=[e_{ij}] \in \Reals^{n \times p}$ denote the error 
matrix; then, 
\[
\frac{1}{np} \| \bE \|_F^2 \,= \frac{1}{np} \sum_{i=1}^n \sum_{j=1}^p e_{ij}^2 \le \frac{1}{n} \max_{j \in [p]} \sum_{i=1}^n e_{ij}^2 = \frac{1}{n} \| \bE \|_{2, \infty}^2.
\]
Given the ubiquity of HSVT, the $\|\cdot\|_{2, \infty}$-norm result for HSVT may be of interest in its own right.

\subsection{Connecting PCR to Synthetic Controls} 
We briefly describe the application of the analysis of PCR to SC, which has become a standard method in econometrics (and beyond) to make counterfactual predictions utilizing only observational data. 
%


%
\smallskip
\noindent
\textbf{Robust synthetic control.}
In \cite{amjad}, the authors propose the robust synthetic control (RSC) method, which pre-processes observations using HSVT before performing linear regression to learn the model. 
They observed empirically that the resulting synthetic control had attractive robustness properties such as robustness to noisy and partially observed data, and thus suggested an alternative model to the convex weights originally proposed by \cite{abadie1, abadie2}; 
compare Figures \ref{fig:basque_mar} and \ref{fig:cali_mar} with Figures \ref{fig:basque_sc_mar} and \ref{fig:cali_sc_mar}.
Using Proposition \ref{prop:equivalence}, we establish PCR is identical to the RSC estimator.
This provides empirical evidence of the importance of pre-processing the covariates (in the setting of SC, this is the donor pool data) by finding its low-dimensional representation. See Section \ref{sec:pcr_sc} for details.
%
\vspace{-5mm}
\section{Main Results} \label{sec:results}

%
{\bf Notations.} For any matrix $\bB \in \Reals^{N \times p}$, let $\| \bB \|_F, \| \bB \|_2, \| \bB \|_\infty$ denote the Frobenius norm, operator norm, and max norm (i.e., largest absolute value among all entries) of a matrix $\bB$, respectively; let $\| \bB \|_{2, \infty}$ denote the max column $\ell_2$-norm of  $\bB$. 
For an index set $\Omega \subset [N]$, let $\bBO$ denote the $|\Omega| \times p$  submatrix of $\bB$ formed by stacking the rows of $\bB$ according to $\Omega$, i.e., $\bBO$ is the concatenation of $\{ \bB_{i, \cdot}: i \in \Omega\}$. 
The superscript $\Omega$ is sometimes omitted if the matrix representation is clear from context. 
%
Let $x \vee y = \max(x, y)$ and $x \wedge y = \min(x, y)$ for any $x, y \in \Reals$. Lastly, let $\mathbb{1}$ denote the indicator function.

\vspace{-5mm}
\subsection{Key Modeling Assumptions}\label{sec:modeling_assumptions}
We recall the approximate linear model given by \eqref{eq:regression_model_general}. 

\noindent {\bf Bounded covariates.}
We assume the entries of $\bbA$ are bounded.
Without loss of generality, we assume the entries are bounded by $1$.
\begin{property}  \label{prop:bounded_covariates}
    The entries of $\bA$ are bounded by one in absolute value, i.e., $\| \bA \|_\infty ~\le 1$.
\end{property}

\noindent {\bf Noise on response variables.}
We make the standard assumption that the noise on the response variable, denoted by $\epsilon_i$, is mean zero and has bounded variance.
\begin{property}  \label{prop:observation_noise_structure}
The response noise $\epsilon = [\epsilon_i] \in \mathbb{R}^N$ is a random vector with independent, mean zero entries such that each of its components has variance bounded above by $\sigma^2$.
\end{property}

\noindent 
{\bf Noise on covariates.}
Recall that rather than observing $\bA$, we are given access to its partially observed and noisy version $\bZ$. 
%
%
Let $\bbH = [\eta_{ij}] \in \Reals^{N \times p}$ denote the covariate noise matrix. 
We define $\bbX = \bbA + \bbH$ as the noisy perturbation of the covariate matrix, without missing values. 
We assume the following property about the noise matrix $\bbH$ (see Definition \ref{def:psialpha} for the definition of $\psi_\alpha$-random variables/vectors).

\begin{property}  \label{prop:covariate_noise_structure}
Let $\bbH$ be a matrix of independent, mean zero $\psi_{\alpha}$-rows for some $\alpha \geq 1$, i.e., 
there exists an $\alpha \geq 1$ and $K_{\alpha} < \infty$ such that $\norm{\eta_{i, \cdot}}_{\psi_{\alpha}} \le K_{\alpha}$ for all $i \in [N]$. 
Further, assume there exists a $\gamma^2 > 0$ such that  $\big\| \Ex \eta_{i,\cdot}^T \eta_{i, \cdot} \big\|_2 \leq \gamma^2$ for all $i \in [N]$. 
Lastly, for all $i \in [N], j \in [p]$, assume variance of $\eta_{i,j}$ is bounded above $\sigma^2$.
\end{property}

\begin{remark}
One can verify that if the entries of $\eta_{i,\cdot}$ are independent, then $\gamma^2 = \mathcal{O}(1)$ (e.g., for independent standard normal random variables, $\gamma^2 = 1$). 
In general, $\gamma^2$ will scale linearly with the number of correlated entries in $\eta_{i, \cdot}$;
similarly, $K_{\alpha}$ scales with the square root of the correlated entries.
\end{remark} 

\begin{remark}
We assume that the response noise $\epsilon$ and covariate noise $\bbH$ are independent of each other. Further, if we denote $D_{ij} \in \{0, 1\}$ as the random variable indicating whether $Z_{ij}$ is missing or not, we assume $D_{ij}$ is independent of $\epsilon$ and $\bbH$.
Relaxing these assumptions and allowing for dependencies between these three sources of noise remains interesting future work.
\end{remark}

\vspace{-5mm}
\subsection{Training Prediction Error} 
In this section, we present bounds on the training error under different settings. 
\vspace{-3mm}
\subsubsection{General Results} 
We first state Theorem \ref{thm:training_pcr_generic} (proof in Appendix \ref{sec:appendix_noisy_regression_via_MCSE}), which bounds the training error of PCR in terms of three natural quantities, as described below. 
\begin{theorem}[Training Error of PCR: Generic Result]\label{thm:training_pcr_generic}
Consider PCR with parameter $k \geq 1$. Suppose Property \ref{prop:observation_noise_structure} holds. Then,
under the model described by \eqref{eq:regression_model_general}, 
\begin{align} \label{eq:mse_upper_generic}
\emph{MSE}_{\Omega}(\hY) &\le \frac{4 \sigma^2 k}{n} + \frac{3 \| \beta^* \|_1^2 }{n} \Ex \| (\bZ^{\emph{HSVT}, k, \Omega} - \bAO)\|_{2, \infty}^2  \,+\, \frac{20 \|\phi\|_2^2}{n} 
\end{align} 
\end{theorem}
\noindent {\em Interpretation.} 
The bound in \eqref{eq:mse_upper_generic} has three terms on the right hand side: 
(a) $ \sigma^2 k / n$ represents the standard ``regression'' prediction error, which scales with the model complexity $k$ and inversely with number of samples $n$;
(b) $(1/n) \, \| \beta^* \|_1^2 \, \Ex \| \bZ^{\text{HSVT}, k, \Omega} - \bAO\|_{2, \infty}^2 $, which is a consequence of the corruption of $\bA$ (if $\bA$ was fully observed and rank $k$, then this error term would vanish); 
(c) {$(1/n) \|\phi\|^2_2$ represents the (inevitable) impact of the model mismatch. }

\medskip
\noindent 
{\em Quantification.} 
To quantify \eqref{eq:mse_upper_generic}, we need to evaluate $\Ex[\|\bZ^{\text{HSVT}, k, \Omega} - \bbA^\Omega\|_{2,\infty}^2]$, where $\bZ^{\text{HSVT}, k}$ is the estimate of $\bbA$ produced from the sparse, noisy observation of it, $\bbZ$.
Our interest is in evaluating the estimation error with respect to the $\ell_{2, \infty}$-error. 
As stated earlier, the estimation error for HSVT is typically evaluated with respect to the Frobenius norm and this quantity is well understood, e.g., see \cite{Chatterjee15}. 
On the other hand, the error bound with respect to $\ell_{2, \infty}$-norm is unknown. 
To that end, we provide a novel characterization of this error in Lemma \ref{lemma:mcse_hvst} below (proof in Appendix \ref{sec:appendix_mcse_hsvt}). 

Let $\bbA = \sum_{i=1}^N \tau_i u_i v_i^T$ with its singular values $\tau_i$ arranged in descending order. 
Let $\bbA^k = \sum_{i=1}^k \tau_i u_i v_i^T$ denote the truncation of $\bbA$ obtained by retaining the top $k$ components. 
\begin{lemma}[$\ell_{2,\infty}$-error bound for HSVT] \label{lemma:mcse_hvst}
Let Properties \ref{prop:bounded_covariates}, \ref{prop:observation_noise_structure},  \ref{prop:covariate_noise_structure} hold.  
If $\rho \geq 64 \log(Np) / (Np)$, 
then 
\begin{align}\label{eq:MCSE_train_bound}
\Ex[\|\bZ^{\emph{HSVT}, k} - \bbA\|_{2,\infty}^2]	 
&\le \frac{C'}{\rho^4} \left( \frac{N (N \vee p) }{ (\tau_k - \tau_{k+1})^2} + k \right) \log^5(Np) + 
2 \| \bA^k - \bA \|_{2, \infty}^2,
\end{align}
where $C' = C (1+\sigma^2)(1+\gamma^2)(1+K^4_\alpha)$ and $C>0$ is an absolute constant.
\end{lemma}

\vspace{-3mm}
\subsubsection{Low-Rank Covariates, Well-Balanced Spectra}\label{sec:well_balanced_low_rank}
We state the following result for PCR when the covariate matrix is low-rank, i.e., $\bA$ admits a low-dimensional representation, and PCR chooses the correct number of principal components. 
\begin{property} \label{property:spectra} 
Let $r$ denote the rank of $\bbA$. The $r$-th largest singular value (i.e., the smallest nonzero singular value) of $A$ satisfies $\tau_r = \Omega( \sqrt{Np/ r} )$.
%
%
\end{property} 
\noindent Property \ref{property:spectra} combined with Property \ref{prop:bounded_covariates} imply the singular spectrum of $\bA$ is 
``well-balanced'' in the sense that $\frac{\tau_1}{\tau_r} =  O( \sqrt{r})$.
Below, we describe another natural setting under which Property \ref{property:spectra} holds. 
%
\begin{remark}
A natural setting in which Property \ref{property:spectra} holds is if $\bA =\Theta(1)$ and the non-zero singular values of $\bA$ satisfy $\tau^2_i = \Theta(\zeta)$ for some $\zeta$. 
Then, $C r \zeta = \| \bA \|_F^2 \, = \Theta( Np )$ for some constant $C$, i.e., $\tau^2_i = \Theta(Np / r)$.
See Proposition \ref{prop:gaussian_example} below for a canonical probabilistic generating process used to analyze probabilistic PCA in \cite{bayesianpca, probpca}, under which Property \ref{property:spectra} holds.  
\end{remark}
\begin{corollary}\label{thm:training_pcr}
Let Properties \ref{prop:bounded_covariates}, \ref{prop:observation_noise_structure}, \ref{prop:covariate_noise_structure},  \ref{property:spectra} hold. 
Suppose PCR chooses the correct number of principal components $k = r = \emph{rank}(\bA)$. 
Let $\rho \geq 64 \log(Np) / (Np)$ and $n = \Theta(N)$. 
Then for any given $\Omega \subset [N]$,
\begin{align} \label{eq:mse_train_hsvt_simple}
\emph{MSE}_{\Omega}(\hY) 
&\le \frac{4 \sigma^2 r}{n} + \frac{C' \| \beta^*\|_1^2}{\rho^4} \, \frac{r \log^5(np)}{n \wedge p} + \frac{20 \|\phi\|_2^2}{n}  ,
\end{align}
where $C' = C (1+\sigma^2)(1+\gamma^2)(1+K^4_\alpha)$  and $C>0$ is an absolute constant. 
\end{corollary}
\begin{proof}
Corollary \ref{thm:training_pcr} follows from Theorem \ref{thm:training_pcr_generic} and Lemma \ref{lemma:mcse_hvst} by setting $k = r$, $\bbA^k = \bbA$, $\tau_{k+1} = 0$. 
\end{proof}

\medskip
\noindent 
{\em Interpretation.} 
The statement of Corollary \ref{thm:training_pcr} requires that the {\em correct} number of principal components
are chosen in PCR. 
In settings where all $r$ singular values of $\bbA$ are roughly equal (Property \ref{property:spectra}), the training prediction decays (up to logarithmic factors) as $ \rho^{-4} r / (n \wedge p) + \|\phi\|_2^2 / n $.
We note that for this exact low-rank setting, analyzing the case where $k > r$ (e.g., as done in \cite{MoonLinearRegression}) is interesting future work.

\medskip 
\noindent 
{\bf Example: embedded Gaussian features.}
We present a classical data generating process under which PCR (and PCA) is justified. 
Consider the setting where $\bA \in \Reals^{N \times p}$ is generated by sampling its rows from a distribution on $\Reals^p$, which in turn, is an embedding of some underlying latent distribution on $\Reals^r$; this is similar in spirit to the probabilistic model for PCA, cf. \cite{bayesianpca, probpca}.
\begin{proposition}\label{prop:gaussian_example}
Let $\bA = \tilde{\bA} \tilde{\bR}$, where the entries of $\tilde{\bA} \in \Reals^{N \times r} $ are independent standard normal random variables, i.e., $\tilde{A}_{ij} \sim \mathcal{N}(0,1)$ and $\tilde{\bR} \in \Reals^{r \times p}$ is another random matrix with independent entries drawn uniformly at random from $\{-1/\sqrt{r}, ~1/\sqrt{r}\}$.
Suppose, $r \leq \frac{\sqrt{p}}{4\sqrt{2\log p}}+ 1$ and $r = o(N)$
{ Then $\| \bA \|_{\infty} \leq 4 \sqrt{ \log(Np) }$ and Property \ref{property:spectra} holds with probability at least $1 - \frac{2}{N^2p} - 2 \exp(-c \sqrt{Nr})$ for some constant $c > 0$. }
\end{proposition}
\noindent{In a strict sense, $\bA$ in Proposition \ref{prop:gaussian_example} does not satisfy Property \ref{prop:bounded_covariates} because of the extra $\log$ factor. Taking a closer look at the proof of Lemma \ref{lemma:mcse_hvst}, we can see that this slack only makes the exponent of the log slightly larger ($5 \to 6$) in Lemma \ref{lemma:mcse_hvst} and Corollary \ref{thm:training_pcr}.} Proof of Proposition \ref{prop:gaussian_example} can be found in Appendix \ref{sec:gaussian_features}.

\vspace{-3mm}
\subsubsection{Beyond Low-Rank Covariates---Low-Rank Approximation in $\| \cdot \|_2$-norm}
\label{sec:geom_decaying_singular_values}
\noindent In Corollary \ref{cor:training_pcr_generic}, we generalize the result of Corollary \ref{thm:training_pcr} to the setting where the low-rank model is misspecified, i.e., $\bA$ does not equal $\bA^{k}$.
\begin{corollary} \label{cor:training_pcr_generic}
Let Properties \ref{prop:bounded_covariates}, \ref{prop:observation_noise_structure},  \ref{prop:covariate_noise_structure} hold.  
Suppose $\rho \geq 64 \log(Np) / (Np)$.  
Let $n = \Theta(N)$. 
Then, 
\begin{align} \label{eq:mse_upper_generic_refined}
	\emph{MSE}_{\Omega}(\hY) & \le
	\frac{4\sigma^2 k}{n}  
	+ \frac{C' \| \beta^* \|_1^2 }{\rho^4} \left( \frac{n \vee p}{ (\tau_k - \tau_{k+1})^2} + \frac{k}{n} \right) \log^5(np) 
	+ \frac{3 \| \beta^* \|_1^2}{n}\| \bA^k - \bA \|_{2, \infty}^2
	+  \frac{20}{n} \| \phi \|_2^2,
\end{align} 
where $C' = C (1+\sigma^2)(1+\gamma^2)(1+K^4_\alpha)$   and $C>0$ is an absolute constant. 
\end{corollary}
\begin{proof}
Corollary \ref{cor:training_pcr_generic} follows immediately from Theorem \ref{thm:training_pcr_generic} and Lemma \ref{lemma:mcse_hvst}. 
\end{proof}

\noindent 
{\em Interpretation.} 
Corollary \ref{cor:training_pcr_generic} implies the training prediction error, not including the linear model mismatch $\phi$, decays to zero if: 
(i) the gap between the $k$-th and $(k+1)$-st singular values of $\bA$ grows faster than $n \vee p$ (ignoring log factors); 
(ii) $k = o(n)$;
(iii) $\| \bA^k - \bA \|_{2, \infty}^2 = o(n)$. 
Below, we show that if the spectrum of $\bA$ is geometrically decaying, then there exists a range of $k$ such that Properties  \ref{prop:bounded_covariates}, \ref{prop:observation_noise_structure}, and  \ref{prop:covariate_noise_structure} are satisfied.

\medskip 
\noindent 
{\bf Example: geometrically decaying singular values.} 
To explain the utility of Corollary \ref{cor:training_pcr_generic}, we consider a setting where $\bbA$ has geometrically decaying singular values, and is thus {\em approximately} low-rank. 
We note that such a setting is representative of many real-world datasets; as an example, see Figures \ref{fig:basque_spectrum} and \ref{fig:cali_spectrum}.
Further, matrices with geometrically decaying singular values are also ubiquitous models in the study of a variety of domains including graphon estimation and signal processing.

Let $e_{\cdot, j} \in \Reals^p$ denote the $j$-th canonical basis vector. 
Recall that $u_i, v_i$, and $\tau_i$ denote the left singular vectors, right singular vectors, and singular values of $\bbA$, respectively.
\begin{proposition} \label{prop:geo_decay_finite_sample}
Let Properties \ref{prop:bounded_covariates}, \ref{prop:observation_noise_structure},  \ref{prop:covariate_noise_structure} hold.  
Suppose $\rho \geq 64 \log(Np) / (Np)$.  
Let $n = \Theta(N)$. 
Let $\tau_1 = C_1 \sqrt{Np}$ and  $\tau_k = \tau_1 \theta^{k-1}$ for all $k \in [N]$ with $\theta \in (0,1)$.
Further, let $v_i^T e_j = O(1/\sqrt{p})$ for all $i, j \in [p]$.
Consider PCR with parameter $k = \frac{1}{4} \cdot \frac{ \log(n \wedge p)}{ \log ( 1/\theta )}$. 
Then,
\begin{align}\label{eq:geo_decay}
\emph{MSE}_{\Omega}(\widehat{Y}) 
& \leq \frac{C' C(\theta) \| \beta^* \|_1^2 }{\rho^4}  \frac{\log^{6}(np)}{ (n \wedge p)^{1/2}} + { \frac{20}{n} \|\phi\|_2^2}, 
\end{align}
where $C' = C (1+\sigma^2)(1+\gamma^2)(1+K^4_\alpha)$, $C(\theta) > 0$ depends only on $\theta$, and $C_{1}> 0$ is an absolute constant.  
\end{proposition}
\noindent Proof of Proposition \ref{prop:geo_decay_finite_sample} can be found in Appendix \ref{sec:appendix_geo_decay_finite_sample}.

\noindent {\em Interpretation.}
The conditions on the spectrum of $\bA$ in Proposition \ref{prop:geo_decay_finite_sample}  are self-explanatory with potentially one exception, $v_i^T e_j = O(1/\sqrt{p})$. 
In effect, this assumption states that the right singular vectors of $\bbA$ satisfy an ``incoherence'' condition, 
cf. \cite{candes2007sparsity}, with the canonical basis of $\mathbb{R}^p$; 
or, equivalently, all entries of the right singular vectors are roughly of the same magnitude, $O(1/\sqrt{p})$. 
See Appendix \ref{sec:proof_geo_decay} for an explicit construction of such a matrix from signal processing.
The bound in Proposition \ref{prop:geo_decay_finite_sample} implies that if the number of principal components is chosen as $4\left(  \log (np) / \log \left( 1/\theta \right) \right)$ and $(n \wedge p) = \Omega(\rho^{-4} \poly(\log p))$, then the training prediction error is dominated by $(1/n) \|\phi\|_2^2$. 
This is precisely the unavoidable linear model mismatch error. 

\vspace{-3mm}
\subsubsection{Beyond Low-Rank Covariates---Low-Rank Approximation in $\| \cdot \|_\infty$-norm}\label{sec:PCR_GLM}
Thus far, in Sections \ref{sec:well_balanced_low_rank} and \ref{sec:geom_decaying_singular_values}, $\bA$ has been assumed to be well-approximated by a {\em specific} low-rank matrix $\bbA^k$ 
that is induced by retaining the top $k$ singular values of $\bA$. 
Such an approximation is optimal with respect to Frobenius and spectral norm. 
However, for approximating with respect to other norms, e.g., $\ell_{2,\infty}$ or $\ell_{\infty}$, exciting progress has been made to obtain different styles of low-rank approximations (see for example \cite{udell2019big, xu2017rates} and references therein). 
Indeed, such low-rank approximations of $\bA$ may not correspond to $\bbA^k$. 
For that reason, we provide an analogous result to Lemma \ref{lemma:mcse_hvst} and Corollary \ref{cor:training_pcr_generic} for the setting when $\bA$ is well-approximated by some {\em arbritrary} low-rank matrix.

Specifically, let $\bbA = \bbA^{\text{(lr)}} + \bbE^{\text{(lr)}}$. 
In words, $\bbA^{\text{(lr)}}$ denotes a low-rank matrix and $\bbE^{\text{(lr)}}$ denotes the approximation error between $\bbA$ and $\bbA^{\text{(lr)}}$. 
Let $r = \text{rank}(\bbA^{\text{(lr)}})$ and let the SVD of $ \bbA^{\text{(lr)}} = \sum_{i=1}^r \tau_i u_i v_i^T$ (again, with the singular values $\tau_i$ arranged in descending order). 
\begin{lemma}\label{lemma:mcse_hvst_LVM}
Let Properties \ref{prop:bounded_covariates}, \ref{prop:observation_noise_structure},  \ref{prop:covariate_noise_structure} hold.  
Consider PCR with parameter $k = r$.  
Let $\rho \geq 64 \log(Np) / (Np)$. 
Then, 
\begin{align}\label{eq:MCSE_train_bound}
\Ex[\|\bZ^{\emph{HSVT}, k} - \bbA\|_{2,\infty}^2]
&\le \frac{C'}{\rho^4} \left( \frac{N(N \vee p \vee \| \bE^{\emph{(lr)}} \|_2^2)}{\tau_r^2}  + r \right) \log^5(Np) + 2 \| \bE^{\emph{(lr)}} \|^2_{2, \infty}, 
\end{align}
where $C' = C (1+\sigma^2)(1+\gamma^2)(1+K^4_\alpha)$ and $C>0$ is an absolute constant. 
\end{lemma} 
\noindent Proof of Lemma \ref{lemma:mcse_hvst_LVM} can be found in Appendix \ref{sec:appendix_mcse_hsvt_LVM}.

\noindent 
{\em Interpretation.} 
Lemma \ref{lemma:mcse_hvst_LVM} is similar to the result of Lemma \ref{lemma:mcse_hvst};
however, because we now only assume that $\bA$ is well-approximated by {\em an} arbitrary low-rank matrix $\bbA^{\text{(lr)}}$ rather than $\bA^{k}$ as done in Section \ref{sec:geom_decaying_singular_values}, this introduces an additional $\| \bE^{\text{(lr)}} \|_2^2 / \tau_r^2$ term compared to the bound in Lemma \ref{lemma:mcse_hvst}.

\begin{corollary} \label{cor:training_pcr_generic_LVM}
Let Properties \ref{prop:bounded_covariates}, \ref{prop:observation_noise_structure},  \ref{prop:covariate_noise_structure} hold.  
Consider PCR with parameter $k = r$.  
Let $\rho \geq 64 \log(Np) / (Np)$. 
Let $n = \Theta(N)$. 
Then, 
\begin{align} \label{eq:mse_upper_generic_refined_LVM}
\emph{MSE}_{\Omega}(\hY) 
& \le \frac{4\sigma^2 r}{n}  
+ \frac{C' \| \beta^*\|_1^2}{\rho^4} \left( \frac{n \vee p \vee \| \bE^{\emph{(lr)}} \|_2^2}{\tau_r^2} + \frac{r}{n} \right) \log^5(np) 
+ \frac{6 \|\beta^*\|_1^2}{n} \| \bE^{\emph{(lr)}} \|^2_{2, \infty} 
\,+\,  \frac{20}{n} \| \phi \|_2^2,
\end{align} 
where $C' = CK^2_\alpha (1+\sigma^2)(1+\gamma^2)(1+K^4_\alpha)$ and $C>0$ is an absolute constant. 
\end{corollary}
\begin{proof}
Corollary \ref{cor:training_pcr_generic_LVM} follows immediately from Theorem \ref{thm:training_pcr_generic} and Lemma \ref{lemma:mcse_hvst_LVM}. 
\end{proof}

\noindent 
{\em Interpretation.} 
If $\bbA^{\text{(lr)}}$ satisfies Property \ref{property:spectra}, i.e., the well-balanced spectra condition, then one can verify (again, ignoring log factors) the prediction error scales as 
$\rho^{-4} r / (n \wedge p) + \|\phi\|_2^2 / n +  \rho^{-4} r  \|\bE^{(\text{lr})} \|^{2}_{\infty}$.
This is identical to the bound in Corollary \ref{thm:training_pcr} with an additional $\rho^{-4} r \|\bE^{(\text{lr})} \|^{2}_{\infty}$ term, which arises since $\bA$ is not assumed to be low-rank but rather is well-approximated by an arbitrary low-rank matrix $\bbA^{\text{(lr)}}$.
Below, we show that under a generalized factor model, $r \|\bE^{(\text{lr})} \|^{2}_{\infty}$ vanishes to zero as $n$ grows.

\medskip 
\noindent {\bf Example: generalized factor model.}
We say the $\bA$ is generated as per a generalized factor model or latent variable model (LVM) if 
\begin{align}  \label{eq:lvm_2}
	A_{ij} &= g(\theta_i, \rho_j), 
\end{align}
where $\theta_i \in \Reals^{d_1}$ and $\rho_j \in \Reals^{d_2}$ are latent features that capture measurement $i$ and feature $j$ specific information, respectively, for some $d_1, d_2 \ge 1$; 
and the latent function $g: \Reals^{d_1} \times \Reals^{d_2} \to \Reals$ captures the model relationship. 
If $g$ is ``well-behaved'', e.g., H\"older continuous, and the latent spaces are compact, then Proposition \ref{prop:lvm} shows $\bA$\ is well approximated by a low-rank matrix with respect to the $\| \cdot \|_\infty$-norm, where the approximation error vanishes as more data is collected. 

\medskip
\noindent
{\em H\"older continuous functions.} 
We now define the H\"older class of functions, which is widely adopted in the non-parametric regression literature (see \cite{xu2017rates, Tsybakov:2008:INE:1522486}). 
Given a function $g: [0, 1)^K \to \Rb$, and a multi-index $\kappa \in \mathbb{N}^K$, let the partial derivate of $g$ at $x \in [0, 1)^K$ (if it exists) be denoted as
\begin{align}\label{eq:def_partial_derivative}
\triangledown_\kappa g(x) = \frac{\partial^{|\kappa|} g(x)}{(\partial x)^\kappa} =  \frac{\partial^{|\kappa|} g(x)}{(\partial x_1)^{\kappa_1} (\partial x_2)^{\kappa_2} \dots (\partial x_K)^{\kappa_K} }. 
\end{align}

\begin{definition}[\textbf{$(\zeta, \cL)$-H\"older Class}]\label{def:holder}
Let $\zeta, \cL$ be two positive numbers. The H\"older class $\cH(\zeta, \cL)$ on $[0, 1)^K$
%
%
is defined as the set of functions $g: [0, 1)^K \to \Rb$ whose partial derivatives satisfy
\begin{align}\label{eq:holder}
\sum_{\kappa: |\kappa| = \lfloor \zeta \rfloor} \frac{1}{\kappa !} |\triangledown_\kappa g(x) - \triangledown_\kappa g(x') | \le \cL \norm{x - x'}_\infty^{\zeta - \lfloor \zeta \rfloor}
\quad \text{for all } x, x' \in [0, 1)^K.
\end{align}
Here, $\lfloor \zeta \rfloor$ denotes the largest integer strictly smaller than $\zeta$.
We note that the domain is easily extended to any compact subset of $\Rb^K$.
\end{definition}

\begin{remark}
Note if $\zeta \in (0, 1]$, then \eqref{eq:holder} is equivalent to the $(\zeta, \cL)$-Lipschitz condition, i.e.,
\[
|g(x) - g(x')| \le \cL \norm{x - x'}_\infty^{\zeta - \lfloor \zeta \rfloor} \quad
\text{for all } x, x' \in [0, 1)^K.
\]
However, for $\zeta > 1$, $(\zeta, \cL)$-H\"older smoothness no longer implies $(\zeta, \cL)$-Lipschitz smoothness. 
\end{remark}

\begin{proposition}\label{prop:lvm}
Let $\bA$ satisfy \eqref{eq:lvm_2} with $\theta_i, \rho_j  \in [0, 1)^K$ as latent parameters.
Further, for all $\rho_j$, let $g(\cdot, \rho_j) \in \cH(\zeta, \cL)$ as defined in \eqref{eq:holder}.
Then, for any $\delta > 0$, there exists a low-rank matrix $\bA^{\emph{(lr)}}$ of rank $r \le C(\zeta, K) \delta^{-K}$ such that
$
	\norm{\bA - \bA^{\emph{(lr)}}}_{\infty} \le \cL \cdot \delta^\zeta. 
$
Here, $C(\zeta, K)$ is a term that depends only on $\zeta$ and $K$.
\end{proposition} 
\noindent The proof of Proposition \ref{prop:lvm} can be found in Appendix \ref{sec:lvm_low_rank_proof}. 
\begin{remark}
We remark on the H\"older continuity of a typical linear factor model, i.e.,
$ g(\theta_i, \rho_j) = \langle \theta_i,  \rho_j \rangle$ 
for some latent vectors $\theta_i, \rho_j \in \Rb^K$.
It is easily seen that such a model satisfies Definition \ref{def:holder} for all $\zeta \in \Nb$, and $\cL = C$, for some absolute positive constant, $C$. 
Thus, one can think of H\"older continuous functions as generalizations of typical linear factor models to sufficiently smooth non-linear functions.
\end{remark}
\begin{corollary}\label{cor:LVM-spectra-error}
Let Properties \ref{prop:bounded_covariates}, \ref{prop:observation_noise_structure},  \ref{prop:covariate_noise_structure} hold.  
Consider PCR with $k = r$.  
Let $\rho \geq 64 \log(Np) / (Np)$. 
Let $n = \Theta(N)$. 
Let the conditions of Proposition \ref{prop:lvm} hold and further assume $\bA^{\emph{(lr)}}$ (as defined in Proposition \ref{prop:lvm}) satisfies Property \ref{property:spectra}.
Then, 
\begin{align} \label{eq:mse_upper_generic_refined_LVM_detailed}
\emph{MSE}_{\Omega}(\hY) 
& \le 
\frac{C' C(\zeta, K) \cL^{2} \|  \beta^*\|_1^2}{\rho^4} \left(\frac{1}{(n \wedge p)^{{1 - \frac{K}{2\zeta}}}}  \right) \log^5(np) 
\,+\,  \frac{20}{n} \| \phi \|_2^2,
\end{align} 
where $C' = C (1+\sigma^2)(1+\gamma^2)(1+K^4_\alpha)$ and $C>0$ is an absolute constant. 
\end{corollary}
\noindent The proof of Corollary \ref{cor:LVM-spectra-error} can be found in Appendix \ref{sec:proof_LVM_MSE}.
Note that as long as $\zeta > K/2$, this leads to vanishing training error.

\vspace{-5mm}
\subsection{Test Prediction Error} 
We now evaluate the generalization performance of PCR. 
As previously mentioned, the emphasis of this work is to provide a rigorous analysis on the prediction properties of the PCR algorithm through the lens of HSVT. 
Recall from Proposition \ref{prop:equivalence}, PCR with parameter $k$ is equivalent to linear regression with pre-processing of the noisy covariates using HSVT where the top $k$ singular values are retained. 
To that end, we study candidate vectors $\beta^{\hsvt,k} = \bV_k \cdot \beta^{\pcr,k} \in \Reals^p$. 
In light of this observation, we establish the following simple but useful result that suggests restricting our model class to sparse linear models only (the proof of which can be found in Appendix \ref{sec:low_rank_prop}). 
\begin{proposition} \label{prop:low_rank_sparsity}
Let $\bX \in \Reals^{n \times p}$ and $\emph{rank}(\bX) = k$.
Without loss of generality, let $\{\bX_{\cdot, 1}, \dots, \bX_{\cdot, k}\}$ form a collection of $k$ linearly independent vectors, i.e., for any  $i \in \{k+1, \dots, p\}$, there exists some $c(i) \in \mathbb{R}^k$ such that $\bX_{\cdot, i} = \sum_{\ell=1}^k c_l(i) \bX_{\cdot, \ell}$.
Assume the following condition on $\bX$ holds: 
\begin{align}\label{eq:test_error_condition}
\max_{i \in \{k+1, \dots, p\}}\| c(i) \|_{\infty} \le C''.
\end{align}
\noindent Then if, $M = \bX v$ for some $v \in \mathbb{R}^p$, there exists $v^* \in \mathbb{R}^p$ such that $M = \bX v^*$,  $\norm{v^*}_0 = k$, and $\norm{v^*}_1 \le C'' k \norm{v}_1$.
\end{proposition}
\noindent {\em Interpretation.}
By Proposition \ref{prop:low_rank_sparsity}, for any $\bZ^{\hsvt, k}$ and $\beta^{\hsvt,k} = \bV_k  \beta^{\pcr,k}$, there exists a $\beta' \in \Reals^p$ such that $\bZ^{\hsvt, k}  \beta^{\hsvt,k} = \bZ^{\hsvt, k}  \beta'$ where $\| \beta' \|_0 \, \le k$ and $\| \beta' \|_1 \, \le k \| \beta^{\hsvt,k} \|_1$. 
Thus, for the purposes of bounding the test error of PCR with parameter $k$ via the toolkit of Rademacher complexity, {\em it suffices to restrict our hypothesis class to linear predictors with sparsity $k$}. 
Condition (i) in Proposition \ref{prop:low_rank_sparsity} is a mild assumption circumventing the pathological case that the linear coefficients used to represent columns $\bX_{\cdot, i}$ for $i \in \{k+1, \dots, p\}$ in terms of $\{\bX_{\cdot, 1}, \dots, \bX_{\cdot, k}\}$ are unbounded.
\begin{theorem} [Test Error of PCR] \label{thm:test_pcr}
Let Property \ref{prop:bounded_covariates} hold.  
Let $n = \Theta(N)$. 
Consider PCR with parameter $k \ge 1$ and assume $\bZ^{\emph{HSVT}, k}$ satisfies \eqref{eq:test_error_condition} in Proposition \ref{prop:low_rank_sparsity}.
Then
%
\begin{align} \label{eq:mse_test_hsvt}
\Ex_\Omega \left[\emph{MSE}(\hY) \right] 
&\le \Ex_\Omega \left[ \emph{MSE}_\Omega(\hY) \right] + \frac{ C'''  k^{5/2} } {\sqrt{n}} \| \beta^* \|_1,
\end{align} 
\noindent where $C''' = C \cdot C'' \cdot \Ex[ \| \beta^{\emph{HSVT},k} \|^{2}_{1} \cdot \| \bZ^{\emph{HSVT}, k} \|_\infty^2 ]$, with
$C''$ defined as in Proposition \ref{prop:low_rank_sparsity} and $C > 0$ is an absolute constant;
$\Ex_\Omega$ denotes the expectation taken with respect to $\Omega \subset [N]$ (of size $n$), which is chosen uniformly at random without replacement. 
\end{theorem}
\noindent 
See Appendix \ref{sec:mse_test_hsvt} for a proof of Theorem \ref{thm:test_pcr}.

\smallskip
\noindent {\em Interpretation.}
%
%
We note all our training error bounds do not depend on $\Omega$; see \eqref{eq:mse_upper_generic}, \eqref{eq:mse_train_hsvt_simple}, \eqref{eq:mse_upper_generic_refined}, \eqref{eq:geo_decay}, \eqref{eq:mse_upper_generic_refined_LVM}, \eqref{eq:mse_upper_generic_refined_LVM_detailed}.
Hence, the bound on $\Ex_\Omega[ \text{MSE}_{\Omega}(\hY)]$ also does not depend on $\Omega$ for these settings. 
Note the test error decays at a rate $1 / \sqrt{n}$, in comparison with $1 / n$ for the training error.
This ``slow rate'' of $1 / \sqrt{n}$ for test error is indeed the best achievable using the standard Rademacher complexity analysis (see Chapter 4 of \cite{wainwright2019high}).
An important open problem is to achieve the fast rate of $1 / n$ for the test error in the error-in-variables setting;
we remark that a  related work \cite{agarwal2020principal} takes key steps towards solving this.

\medskip 
\noindent 
{\bf Choosing $k$.} 
\label{remark:picking_k}
We describe how the test prediction error 
can help in choosing the parameter for PCR (model complexity) in a data-driven manner. 
Specifically, Theorem \ref{thm:test_pcr} suggests that the overall error is at most the training error plus a term that scales as $k^{5/2} / \sqrt{n}$. Therefore, one should choose the $k$ that minimizes this bound. 
Naturally, as $k$ increases, the training error is likely to decrease, but the additional term $k^{5/2}/\sqrt{n}$ will increase; an optimal $k$ can thus be found in a data-driven manner. 

\vspace{-5mm}
\subsection{Discussion} \label{ssec:ols_pcr}
{\bf Comparison with ordinary least squares (OLS).} 
It is known that OLS implicitly performs regularization if the covariates $\bA$ are exactly low-rank, noiseless, and fully observed (see Lemma 3.1 of \cite{rigollet2011}). 
In most real-world settings, however, data is never precisely low-rank, but is rather {\em approximately} low-rank, such as in the examples detailed in Sections \ref{sec:geom_decaying_singular_values} and \ref{sec:PCR_GLM} (see \cite{Udell} and references therein for further theoretical justification for approximately low-rank covariate matrices). 
In such a setting, it is not established, nor is it likely, that OLS has the same implicit regularization effect as before. 
Indeed, in the example shown in Figure \ref{fig:basque_lr}, OLS has very poor empirical generalization performance even though over $99\%$ of the spectral energy is captured in the top singular value, i.e., the covariate matrix is very-well approximated by a rank-one matrix. 
In contrast, if the principal components are chosen correctly, then PCR continues to have the desired regularization property, even in the approximate low-rank case. The contrast can be seen in Figure \ref{fig:basque_pcr}. 
Additionally, we provide the explicit tradeoff between training and testing error based on the number of selected principal components $k$.

\medskip
\noindent {\bf ``Information'' spread across covariates is necessary.}
Within the high-dimensional (error-in-variables) regression literature, there are several different structural assumptions required of the covariate matrix to achieve vanishing prediction or parameter estimation error (see \cite{re_condition} and references therein for some detailed examples). 
Intuitively, these assumptions state that the signal is ``well-spread'' across the various columns of the covariate matrix. 
Below, we consider a simple yet illustrative example for which both PCR and traditional methods from the literature do not seem to provide meaningful answers. 

\medskip
\noindent{\em Example.} Suppose $\bA_{\cdot,1} = e_1$ and $\bA_{\cdot, 2} = e_2$, where $e_1, e_2$ are the canonical basis vectors in $\Reals^N$, and $\bA = [\bA_{\cdot,1}, \bA_{\cdot,2}, \dots, \bA_{\cdot,2}] \in \Reals^{N \times p}$. Then, it is clear that $r = \text{rank}(\bA) = 2$.

\medskip
\noindent{\em What happens to PCR.} 
To estimate $\bA$, even with the additional (oracle) knowledge of the positions of $\bA_{\cdot, 1}$ and $\bA_{\cdot,2}$, one can verify the optimal estimators for $\bA_{\cdot, 1}$ and $\bA_{\cdot,2}$ are $\bZ_{\cdot, 1}$ and $1/(p-1) \sum_{j=2}^p \bZ_{\cdot, j}$, respectively. 
This results in the following lower bound on the recovery error
\[ 
\| \bhA - \bA \|_{2, \infty}^2 \, \ge \| \bZ_{\cdot, 1} - \bA_{\cdot, 1} \|_2^2 \, = \| \eta_{\cdot, 1} \|_2^2 \,\stackrel{\mathbb{E}}{=} N,
 \]
yielding 
\[ 
\frac{1}{N} \Ex \left[ \| \hY - \bA \beta^* \|_2^2 \right] \le \frac{\sigma^2 r}{N} + \| \beta^*\|_1, 
\]
which does not lead to prediction consistency. 
In fact, the second term, $\| \beta^* \|_1$, is exactly what arises if the bias is not corrected in the error-in-variables regression setting of \cite{loh_wainwright, tsybakov_1, tsybakov_3, tsybakov_4}. 

\medskip
\noindent{\em What happens to traditional error-in-variables regression estimators.}  
Now, consider the same setup as above but let $\bA$ be fully observed, i.e., $\bZ = \bA$.
In such settings where $\bA$ is uncontaminated, it is known that the restricted eigenvalue (RE) condition (see Definition \ref{def:re} of Appendix \ref{sec:definitions}), which is the de-facto assumption in the literature, guarantees $\ell_2$-recovery of the underlying $\beta^*$ via the Lasso method. 
However, this particular $\bA$ breaks the RE condition and thus $\beta^*$ cannot be accurately estimated. 
To see this, let $\Delta = e_3 \in \Reals^p$ in Definition \ref{def:re}. 
Then, $( 1/N) \| \bA \Delta \|_2^2 \,= 0 $, hence violating the RE condition needed for all the existing analyses of methods for the error-in-variables regression setting. 

\medskip
\noindent{\em In summary.}
From this simple example, we observe that a lack of information spread across the columns of $\bA$ seem to yield poor prediction and parameter estimation errors. 
However, it has been well established that a large ensemble of covariate matrices $\bA$ satisfy the RE condition; 
specifically, when the entries (or rows) of $\bA$ are sampled independently from a sub-gaussian distribution. 
Analogously, we show that two canonical generating processes for the covariate matrix, namely embedded Gaussian features and geometrically decaying singular values, satisfy the desired properties needed to achieve vanishing prediction error.
That is, the singular value gap $\tau_k - \tau_{k+1}$ is sufficiently large under such generating processes for appropriate $k$, where $k$ is the number of chosen principal components. 
\vspace{-5mm}
\section{PCR and Synthetic Controls} \label{sec:pcr_sc}
\subsection{Synthetic Controls Setup}\label{sec:sc_setup}
{\bf Pre- \& post-intervention periods.}
As is standard in the SC literature, let there be $p + 1$ different time series over $N$ periods associated with a target unit and $p$ donor units. 
%
Suppose the target unit receives the intervention at time period $n$, where $1 \le n < N$. 
We will refer to the pre- and post- intervention periods as the time periods prior to and after the intervention point.

\smallskip
\noindent
{\bf Donor observations under control.}
Let $\bA \in \Reals^{N \times p}$ represent the true utilities of the $p$ donor units across the entire time horizon $N$ in the absence of intervention;
 i.e., $\bA_{\cdot, j} \in \Reals^N$ represents the time series over $N$ periods for donor $j \in [p]$. 
Rather than observing $\bA$, we assume we are only given access to $\bZ \in \Reals^{N \times p}$, a sparse, noisy instantiation of $\bA$. 
In words, $\bZ$ denotes the corrupted donor pool observations; as made precise later in the section, we assume $\bZ$ follows the distributional characteristics described in Section \ref{sec:modeling_assumptions}.

\smallskip
\noindent
{\bf Target unit observations under control.}
For every $i \in [N]$, let $Y_i$ denote the noisy utility associated with the target unit in the absence of intervention (control). 
However, since the target unit experiences an intervention for all time instances $n < i \le N$, we only have access to a noisy version of the target unit's utility for the pre-intervention period, i.e., we only observe $Y^{\text{pre}} = [Y_i]$ for $i \in [n]$.
Analogously, we denote $Y^{\text{post}}= [Y_i]$ for $i \in N \setminus [n]$ as the target's (noisy) utility in the post-intervention period.
We will denote $\Ex[Y_i] \in \Reals$ as the true, latent utility at time $i$ for the target unit, if the intervention never occurred. 
In summary, given data $(Y^{\text{pre}}, \bZ)$, the aim is to recover $\Ex[Y^{\text{post}}]$, the counterfactual trajectory of the target unit under control in the post-intervention period. 
For a pictorial view of the setup of the problem, please refer to Figure \ref{fig:sc}.

\vspace{-5mm}
\subsection{(Approximate) Linear Synthetic Controls Exist}\label{sec:linear_model_sc}
{\bf Existence of linear synthetic controls.} 
In the SC literature, two standard assumptions are made: 
first, there exists a linear relationship between the target and donor units---in \cite{abadie2, abadie1}, a more restrictive assumption is made that a convex relationship between the target and units exists; 
second, the underlying utilities follow a low-rank factor model. 

Below, we show that if the underlying utilities of the target and donor units follow a generalized factor model or latent variable model (LVM) as in \eqref{eq:lvm_2}, then an (approximate) linear relationship between the target and donor units is actually {\em implied} by such a model.
That is, the existence of an {(approximate)} linear synthetic control does not need to be additionally assumed. 
Further, we establish that the linear approximation error goes to zero the more data that is collected.
As stated in Section \ref{sec:PCR_GLM}, LVMs are a natural nonlinear generalization of the typical {\em factor model} ubiquitous in studying panel data in econometrics.

To that end, let $\bA' = [A'_{ij}] \in \Reals^{N \times (p+1)}$ denote the concatenation of $\bA$, the latent donor pool utilities, with $\Ex[Y]$, the vector of underlying utilities for the target unit in the absence of an intervention. 
We denote $\bA_{\cdot, 0}' = \Ex[Y]$ as the latent true utility vector for the target unit, and $\bA_{\cdot, j}' = \bA_{\cdot, j}$ for all $j \in [p]$ as the latent true utilities for the donor pool.
We assume $\bA'$ follows a LVM as detailed below:

\begin{property}\label{glm_linearity}
Let $\bA'$ follow a LVM as defined in \eqref{eq:lvm_2} -- as established in Proposition \ref{prop:lvm}, for any $\delta > 0$,  there exists ${\bA'}^{\emph{(lr)}}$ of rank $r \le C(\zeta, K) \delta^{-K}$ such that $\norm{\bA' - {\bA'}^{\emph{(lr)}}}_{\infty} \le \cL \cdot \delta^\zeta$.
Let $\zeta > K$.
Denote ${\bA'}^{\emph{(lr)}} = \bU \bV^{T}$ as its singular value decomposition where $\bU \in \Rb^{N \times r}, \bV \in \Rb^{(p + 1) \times r}$ and $v_{i}$ denotes the $i$-th row of $\bV$.
Let $v_{0}$ lie within $\text{span}(\{v_{i}\}_{i \in [p]})$.
%
%
\end{property}
\noindent {\em Interpretation.}
If $\bA'$ satisfies a LVM as defined in \eqref{eq:lvm_2}, then the existence of ${\bA'}^{\text{(lr)}}$ as in Property \ref{glm_linearity} is simply a restatement of Proposition \ref{prop:lvm}.
To analyze SC, we make the additional mild assumption in Property \ref{glm_linearity} that $v_{0}$ lies within $\text{span}(\{v_{i}\}_{i \in [p]})$.
This assumption helps avoid the ``pathological'' case where the right singular vector associated with the target unit, $v_{0}$, does not lie within the span of the right singular vectors associated with the donor units, $\{v_{i}\}_{i \in [p]}$. 
Even in the worst case, by the definition of a rank of a matrix, there can only exist $r$ out of the $p$ right singular vectors $v_{i}$ that do not lie within the span of the remaining singular vectors.
Indeed, we can pick $\delta$ as defined in Property \ref{glm_linearity}, such that $r = \text{rank}({\bA'}^{\text{(lr)}}) \le C(\zeta, K) \delta^{-K} = o(p)$, rendering this pathological case overwhelmingly unlikely to hold.
Lastly, we note this assumption that $v_{0}$ lies within $\text{span}(\{v_{i}\}_{i \in [p]})$ is implicitly always made in the SC literature.
\begin{prop}\label{prop:linear_comb}
Assume $\bA'$ satisfies Property \ref{glm_linearity}. 
Then there exists a $\beta^* \in \Reals^{p}$ such that the target unit (represented by index $0$) satisfies for all $i \in [N]$, and for any $\delta > 0$,
\begin{align} \label{eq:sc_existence}
	| A'_{i0} - \sum_{k=1}^{p} \beta^*_k \cdot A'_{ik}| \le C(\zeta, K) \cdot \cL \cdot \delta^{(\zeta - K)}. 
\end{align}
Here, $C(\zeta, K)$ is defined as in Property \ref{glm_linearity}. 
\end{prop}
\noindent Proof of Proposition \ref{prop:linear_comb} can be found in Appendix \ref{sec:linearity_exists}.

\smallskip
\noindent {\em Interpretation.}
Proposition \ref{prop:linear_comb} shows that if $\bA'$ follows a LVM as in \eqref{eq:lvm_2}, then a (approximate) linear synthetic control exists, where the linear misspecification error decays to zero for appropriate choice of $\delta$ in \eqref{eq:sc_existence}.
Moreover, empirically a LVM is well-motivated -- across many real-world datasets, including the canonical SC case studies of California Proposition 99 and terrorism in Basque Country of \cite{abadie2, abadie1}, we see they exhibit an approximate (very) low-rank structure (see Figures \ref{fig:basque_spectrum}, \ref{fig:basque_energy}, \ref{fig:cali_spectrum}, and \ref{fig:cali_energy}). 

\vspace{-5mm}
\subsection{Synthetic Controls and Error-in-variables Regression} \label{ssec:pcr_rsc_eiv}
\noindent
{\bf SC framework fits error-in-variables regression with model mismatch.}  
If $\bA'$ satisfies Property \ref{glm_linearity}, Proposition \ref{prop:linear_comb} establishes that we can express the underlying utility of the target unit under no intervention for all $i \in [N]$ as
\begin{align}\label{eq:sc_linear_model}
\Ex[Y_i] = \bA'_{i0} = \bA_{i, \cdot} \beta^* + \phi_i.
\end{align}
Here $\beta^* \in \Reals^p$ is defined as in Proposition \ref{prop:linear_comb}, and $\phi _i$ is the model mismatch bounded by $C(\zeta, K) \cdot \cL \cdot \delta^{(\zeta - K)}$.
That is, in the SC framework, \eqref{eq:regression_model_general} holds under Property \ref{glm_linearity}. 
In summary Proposition \ref{prop:linear_comb} reduces the question of interest in SC of estimating $\Ex[Y^{\text{post}}]$ -- the counterfactual trajectory of the target unit under no intervention in the post-intervention period -- to that of linear regression with model mismatch.
We note that we are in the error-in-variable setting as instead of observing $(\Ex[Y^{\text{pre}}_i], \bA)$, we only get to observe $(Y^{\text{pre}}, \bZ)$.
  
\smallskip
\noindent
{\bf Restating objective in SC framework.} 
Given \eqref{eq:sc_linear_model}, we can write the {\em pre-intervention error} as
\begin{align} \label{eq:pre_int_error}
	\text{MSE}_{\text{pre}}(\hY) &= \frac{1}{n} \, \Ex\left[ \sum_{i \in [n]} (\hY_i - \bA_{i, \cdot} \beta^*)^2 \right],
\end{align}
and the {\em post-intervention error} as
\begin{align} \label{eq:post_int_error}
	\text{MSE}_{\text{post}}(\hY) &= \frac{1}{N-n} \, \Ex\left[ \sum_{i \in [N] \setminus  [n]} (\hY_i - \bA_{i, \cdot} \beta^*)^2 \right].
\end{align} 
\eqref{eq:pre_int_error} is precisely the training error defined in \eqref{eq:train_error} and
\eqref{eq:post_int_error} is a slightly modified form of \eqref{eq:test_error} since the objective now is to accurately estimate the counterfactual in the absence of any intervention only during the post-intervention stage. 
Observe that the objective in SC exactly fits the setting of transductive semi-supervised learning as described in Section \ref{sec:intro_aim}. 

\vspace{-5mm}
\subsection{Finite-sample Analysis of RSC via PCR} \label{ssec:pcr_rsc}
{\bf RSC is equivalent to PCR.}
The RSC method proposed by \cite{amjad} has exhibited empirical success in duplicating the celebrated results of \cite{abadie2, abadie1} for the California Proposition 99 and terrorism in Basque Country case studies, respectively, using: 
(i) only the outcome data, i.e., without any usage of auxiliary covariates
(ii) in the presence of noisy data.
Under these two conditions, the classical SC algorithm of \cite{abadie2, abadie1} provides poor post-intervention predictions (see Figures \ref{fig:basque_sc_mar} and \ref{fig:cali_sc_mar}).

The RSC method is a three step procedure: 
(i) perform HSVT on the donor matrix (include both pre- and post-intervention data); 
(ii) linearly regress thresholded donor matrix with pre-intervention data of the target unit to learn linear weights for each of the donors;
(iii) apply these linear weights on the post-intervention donor data to estimate the counterfactual trajectory for the target unit.

Observe that the RSC method is precisely the algorithm detailed in Section \ref{sec:intro_pcr_ME}, of doing HSVT followed by OLS. 
Empirically, \cite{amjad} demonstrated that the RSC method's first step of pre-processing via HSVT effectively de-noises and imputes missing values in the donor observations, which is crucial in building a robust linear synthetic 
control that has good post-intervention performance. 
Pleasingly, by Proposition \ref{prop:equivalence}, we can equivalently interpret the RSC method as simply PCR. 

It is worth noting that one of the primary motivations for utilizing convex regression (as proposed in \cite{abadie2, abadie1}) was to impose sparsity in the number of donors chosen, i.e., enforcing most of the coefficients of the synthetic control to be zero. 
Rather than introducing sparsity in the original donor space, PCR can be interpreted as introducing sparsity in the subspace induced by the right singular vectors corresponding to the donors since only the top few right singular components are retained.
Indeed, as made precise by Proposition \ref{prop:low_rank_sparsity}, PCR performs implicit $\ell_0$-regularization on the learnt linear model. 

\smallskip
\noindent
{\bf Theoretical results.}
By viewing RSC via the lens of PCR, it allows us to bound the post-intervention prediction error for the target unit. 
We recall some necessary notation.
Recall the definition of $\bA$ and $\bZ$ from Section \ref{sec:sc_setup}; 
the definition of $\bA'$ from Section \ref{sec:linear_model_sc};
and denote $\bZ^{\hsvt, k}$ and $\beta^{\hsvt,k}$ as the de-noised donor matrix and the fitted linear model outputted from RSC, respectively.
\begin{theorem} \label{thm:mse_sc}
Let $\bA, \bZ$ satisfy Properties \ref{prop:bounded_covariates}, \ref{prop:observation_noise_structure}, \ref{prop:covariate_noise_structure}.
Let $\bA'$ satisfy: 
(i) \eqref{eq:lvm_2} and further assume $\theta_i$ for $i \in [N]$, the latent parameters associated with time, are sampled i.i.d from some latent distribution $\Theta$;
(ii) Property \ref{glm_linearity} and further assume ${\bA}^{\emph{(lr)}} \in \Rb^{N \times p}$, the restriction of ${\bA'}^{\emph{(lr)}}$ to the donor units, satisfies Property \ref{property:spectra}.
Let $\bZ^{\emph{HSVT}, k}$ satisfies \eqref{eq:test_error_condition} in Proposition \ref{prop:low_rank_sparsity}.
Let $N - n = \Theta(n)$.
Then,
\begin{align} \label{eq:mse_upper_generic_refined_LVM_detailed-SC}
\emph{MSE}_{\emph{post}}(\hY)
& \le 
\frac{C' C(\zeta, K) \cL^{2} \|  \beta^*\|_1^2}{\rho^4} \left(\frac{1}{(n \wedge p)^{{1 - \frac{K}{2\zeta}}}}  \right) \log^5(np) 
\,+\, \frac{ C'''  k^{5/2} } {\sqrt{n}} \| \beta^* \|_1,
\end{align} 
where:
$C' = C (1+\sigma^2)(1+\gamma^2)(1+K^4_\alpha)$;  
$C(\zeta, K)$ and $\cL$ are defined as in Property \ref{glm_linearity};
$C''' = C \cdot C'' \cdot \| \beta^{\emph{HSVT},k} \|_{1} \cdot \Ex[\| \bZ^{\emph{HSVT}, k} \|_\infty^2 ]$, with
$C''$ defined as in Proposition \ref{prop:low_rank_sparsity};
$C>0$ is an absolute constant.
\end{theorem}
\vspace{-3mm}
\noindent The proof of Theorem \ref{thm:mse_sc} can be found in Appendix \ref{sec:mse_sc_proof}.

\smallskip
\noindent {\em Interpretation.}
We highlight that Theorem \ref{thm:test_pcr} bounds $\Ex_\Omega [\text{MSE}(\hY)]$, while Theorem \ref{thm:mse_sc} bounds $\text{MSE}_{\text{post}}(\hY)$.
That is, Theorem \ref{thm:mse_sc} differs from Theorem \ref{thm:test_pcr} in that the set of observations for which we see labels (i.e., observations of the target in the pre-intervention period) is {\em not assumed to be drawn uniformly at random}.
Such an assumption obviously cannot hold in the setting of SC as the pre-intervention period chronologically occurs before the post-intervention period.
Instead, we make a more standard assumption that the latent features $\theta_i$, which correspond to different time periods, are sampled i.i.d. from some unknown distribution, $\Theta$. 
Lastly, we leave it as open problem of how to achieve confidence intervals for the post-intervention error for RSC; one could possibly do so by extending our results to hold in high-probability rather than in expectation.

\smallskip
\noindent \textbf{Comparison with related works in SC.}
The most relevant results to compare against are Corollary 4.1 (pre-intervention prediction error) and 
Theorem 4.6 (post-intervention prediction error) in \cite{amjad}. 
To begin with, as in standard in the SC literature, \cite{amjad} does not establish the existence of a synthetic control, rather it simply assumes one exists. 
Corollary 4.1 in \cite{amjad} does not show consistency of the RSC method with respect to pre-intervention error as there is an irreducible term, $\sigma^2$, the measurement noise in the donor pool, that does not vanish. 
Further, with respect to post-intervention error, Theorem 4.6 of \cite{amjad} suffers from the same irreducible $\sigma^2$ term. In addition, the authors do not show that the second term of their bound decays to zero as more data is collected. 
As importantly, in both bounds, it is assumed that the RSC method picks the correct number of singular components, 
i.e., the rank of the underlying matrix of donor utilities is correctly chosen and is of a lower order compared to the ambient dimensions. 
In contrast, in our setting, we allow the low-rank condition of the underlying donor matrix to be misspecified, i.e., follows a generalization factor model.
Finally, their result does not provide guidance for picking the right parameter $k$ for rank (or in PCR) as done by our result through the generalization or post-intervention error analysis. 

Additionally, \cite{ark} considers a similar setting where the observed covariates $\bZ$ are a corrupted version (additive noise model) of the true, underlying covariates $\bA$, which follow an approximately low-rank factor model, i.e., $\bA$ cannot have too many large singular values (specifically, refer to Assumption 3 of \cite{ark}).
However, they do not allow for missing data within $\bZ$. 
Here, the authors perform convex regression (with $\ell_2$-norm constraints) along both the unit and time axes (unlike standard SC methods, such as RSC, which only consider regression along the unit axis) to estimate the causal average treatment effect. 
As is classically done in the SC literature, the authors of \cite{ark} assume that convex weights exist amongst the rows of $\bA$; in contrast, we show that (approximate) linear weights are directly implied by a (approximate) low-rank factor model.
For this setting, they establish a rigorous asymptotic normality result for their causal estimand of interest, which is the average treatment effect of all treated units over the entire post-intervention period; in contrast, our target causal estimand is the entire post-intervention vector for each treated unit, for which we show mean squared error consistency at rate $1 / \sqrt{n}$.
The work of \cite{ark} complements our own in terms of clarifying the tradeoff between the assumptions made on $\bA$ (i.e., the low-rank approximation error), the constraints on the synthetic control weights (linear vs. convex), and the target causal estimand.
Building on these works to explicitly define the tradeoff between what can be assumed on the spectra of $\bA$, the synthetic control weights, and the subsequent results one can get for various target causal estimand is an interesting future research direction.  

Another work that is less related, but is worth commenting on, as it also heavily relies on matrix estimation techniques for SC, is \cite{athey1}. 
Here, the authors consider an underlying low-rank matrix of $N$ units and $T$ measurements per unit, and the entries of the observed matrix are considered ``missing'' once that unit has been exposed to a treatment. 
To estimate the counterfactuals, \cite{athey1} applies a nuclear norm regularized matrix estimation procedure. 
Some key points of difference are that their performance bounds are with respect to the Frobenius norm over all entries (i.e., units and measurements) in the matrix; meanwhile, we provide a stronger bound that is specific to the single treated unit and only during the post-intervention period. 
Additionally, the bound of \cite{athey1} depends on a parameter, which they denote as $p_c$, that represents the minimum probability of observe all $T$ measurements associated with a given unit. 
The authors establish consistency of their estimator provided that $p_c \gg 1/\sqrt{T}$. 
When data is randomly missing, even if the probability of observing each entry is $1-\varepsilon$ for any $\varepsilon > 0$, then $p_c < (1-\varepsilon)^T = o(1/\sqrt{T})$; thus, this result is not applicable for our setting.

%
\vspace{-5mm}
\subsection{Empirical Results}\label{sec:pcr_sc_empirical}
\vspace{-3mm}
We present empirical results using the RSC method on several well-known datasets in the literature to highlight its robustness properties in comparison with the traditional SC estimator and OLS.

\smallskip
\noindent {\bf Terrorism in Basque Country.} 
A canonical case study within the SC literature investigates the impact of terrorism on the economy in Basque Country (see \cite{abadie1}). 
Here, the target unit of interest is Basque Country, the donor pool consists of neighboring Spanish regions, and the intervention is represented by the first wave of terrorist activity in 1970. 
The aim in this study is to isolate the effect of terrorism on the GDP of Basque Country. 
In other words, to evaluate the effect of terrorism, SC-like methods aim to estimate the unobservable counterfactual GDP growth in the absence of terrorism for Basque Country using observations from various other Spanish regions, which are assumed to be unaffected by the terrorist activity. 

Since we do not have access to the counterfactual realities of the Basque Country GDP post 1970 in the absence of terrorism, we will use the celebrated estimates of \cite{abadie1} as our baseline; 
this is our chosen ``ground-truth'' because these counterfactual trajectories of Basque's GDP in the absence of terrorism have been widely accepted in the econometrics community. 
The resulting synthetic Basque is displayed in Figure \ref{fig:basque_cvx}. 

\smallskip
\noindent {\em Classical SC and OLS under missing data.} 
We randomly obfuscate data, ranging from 5-20\%, and in Figure \ref{fig:basque_sc_mar} we plot the resulting synthetic Basque GDPs predicted via convex regression on the outcome GDP data, i.e., the original SC method {\em without} auxiliary covariates --
the solid blue and orange lines represent the observed and synthetic Basque (predicted by \cite{abadie1}), respectively, while the dashed lines represent the synthetic Basques under varying levels of missing data. 
As clearly seen from the figure, the original SC method is not robust to sparse observations, which may explain its dependency on auxiliary covariates to learn its model.

Additionally, we construct a synthetic Basque using OLS, i.e., running linear regression without any pre-processing of the donor observations (on only the outcome GDPs). 
As seen in Figure \ref{fig:basque_lr}, OLS clearly overfits to the idiosyncratic noise of the pre-intervention data and fails to produce sensible post-intervention estimates. 
In fact, the synthetic Basque GDP as predicted by OLS suggests terrorism actually had a long-term benefit for the Basque economy!
This example motivates the importance of appropriately regularizing and de-noising the donor data, as PCR does, prior to learning a synthetic control.

\smallskip
\noindent {\em Importance of covariate pre-processing via PCA.} 
The first step of PCR (i.e., PCA) is even more starkly empirically motivated by inspecting the singular value spectrum and cumulative energy of the Basque dataset, which are shown in Figures \ref{fig:basque_spectrum} and \ref{fig:basque_energy}, respectively. 
The data exhibits low-dimensional structure with over 99\% of the spectral energy captured in just the {\em top} singular value, which fits the setting under which our theoretical results imply low pre- and post-intervention prediction errors. 
The resulting synthetic Basque as per PCR (or, equivalently, the RSC method) is shown in Figure \ref{fig:basque_pcr}, which pleasingly closely matches that of \cite{abadie1}. 
Similarly, in Figure \ref{fig:basque_mar}, we display various synthetic Basque GDPs after randomly obfuscating the donor observations. 
Across the varying levels of missing data from 5-20\%, the synthetic Basque GDPs continue to resemble the baseline estimates of \cite{abadie1} such that the same conclusion on the negative economic effects of terrorism can be drawn.

Importantly, we underscore that all of the results computed via the PCR method shown in Figures \ref{fig:basque_pcr} and \ref{fig:basque_mar} only use the outcome data (only the per-capita GDP values), i.e., the PCR estimator does {\em not} utilize the auxiliary covariate information that was required to achieve the results in \cite{abadie1}. 
Hence, PCR exhibits desirable robustness properties with respect to missing and noisy data, and with less stringent data requirements to achieve similar counterfactual estimates. 

\smallskip
\noindent {\bf California Proposition 99.} 
Another popular case study in the SC literature investigates the impact of California's Proposition 99, an anti-tobacco legislation, on the per-capita cigarette consumption in the state (see \cite{abadie2}). 
Similar to the Basque example, we will use the widely accepted counterfactual estimate of \cite{abadie2} as our baseline, which is shown in Figure \ref{fig:cali_pcr_cvx}. 
Here, the authors of \cite{abadie2} considered California as the target state, the collection of states in the U.S. that did not adopt some variant of a tobacco control program as the donor pool, and Proposition 99 (enacted in 1988) as the intervention.

We again plot the resulting Californias learned via convex regression {\em without} auxiliary covariates and under varying levels of missing data (5-20\%) in Figure \ref{fig:cali_sc_mar}; 
similar to the Basque case study, this highlights the poor performance of the original SC method in the presence of missing data. 
Further, we plot the singular value spectrum and energy of the California Proposition 99 dataset, seen in Figures \ref{fig:cali_spectrum} and \ref{fig:cali_energy}, and observe that over 99\% of the cumulative spectral energy is again captured by the top singular value, which fits the setting under which our theoretical results apply and motivates the application of PCR. 

Empirically, we observe that the resulting synthetic California predicted via PCR, also displayed in Figure \ref{fig:cali_pcr_cvx}, closely matches the baseline, again without using any of the auxiliary covariates considered in the work of \cite{abadie2}. 
This is indeed expected from the theoretical analysis given the extremely low-dimensional structure of the data. 
Much like the previous Basque example, across the varying levels of missing data from 5-20\%, the synthetic California per-capita cigarette consumption trajectories continue to mirror the baseline estimates of \cite{abadie2}---even in the presence of missing data, the counterfactual estimates produced by PCR suggest that Proposition 99 successfully cut smoking in California.

\section{Conclusion} \label{sec:conclusion}
\noindent {\bf Summary of contributions.}
As the main contribution of this work, we address a long-standing problem of showing PCR (as is) is surprisingly robust to a wide array of problems that plague large-scale modern datasets, including high-dimensional and noisy, sparse, and mixed valued covariates. 
We provide meaningful non-asymptotic bounds for both the training and testing (transductive semi-supervised setting) errors for these settings, even when the covariate matrix is only {\em approximately} low-rank and the linear model is {\em misspecified}. 
From a practical standpoint, our testing error bound further provides guidance as to how to choose the PCR hyper-parameter $k$ in a data-driven manner. 
To achieving our formal results, we establish a simple, but powerful equivalence between PCR and linear regression with covariate pre-processing via HSVT; 
in the process, we provide a novel error analysis of HSVT with respect to the $\ell_{2, \infty}$-norm. 
We then formally connect our theoretical results with three important applications to highlight the broad meaning of ``noisy'' covariates; 
namely, SC (measurement noise), differentially private regression (noise added by design), and mixed covariate regression (``structural'' noise). 
Of particular note, given the equivalence between PCR and the RSC estimator, it immediately leads to a finite-sample bound for the post-intervention error of RSC under a generalized factor model, which is currently absent from the literature.
We note that finite-sample analyses are absent for most SC estimators. 

\medskip
\noindent {\bf How to ``robustify'' an estimator.}
In essence, this work shows that the PCA component of PCR is an effective pre-processing tool in finding a linear low-dimensional embedding of the covariates, which carries the added benefits of implicit de-noising and $\ell_0$-regularization. 
We postulate that when the covariate data is ``unstructured'' (e.g., speech or video), finding meaningful nonlinear low-dimensional embeddings of the data can also achieve similar implicit benefits, e.g., via a variational auto-encoder or a general adversarial network.
We hope this work motivates a general statistical principle that to ``robustify'' a statistical estimator---first find a low-dimensional embedding of the data before fitting a prediction model.

{\smaller
\bibliographystyle{abbrv}
\bibliography{Bibliography-MM-MC}
}

\vspace{-5mm}
\section{Figures} \label{sec:figures}

\iftoggle{FIGURES}{
\begin{figure}[!htb]
	\centering
	\includegraphics[width=0.4\textwidth]{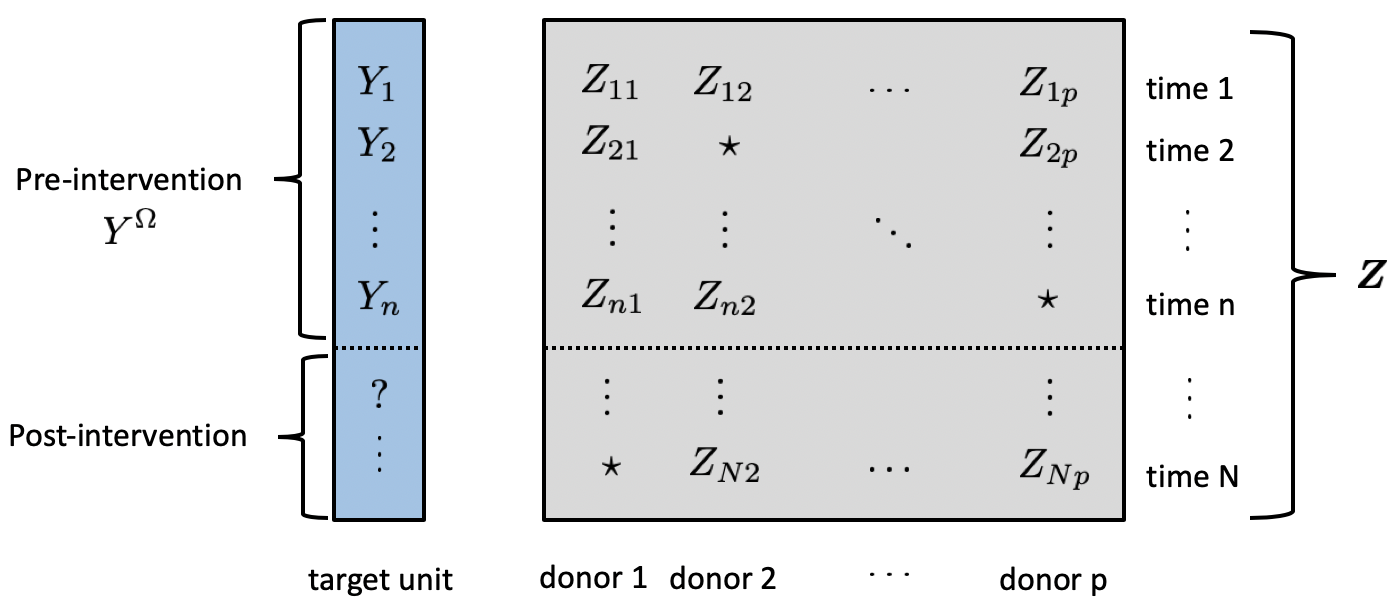}
	\caption{Caricature of observed data $(Y^{\text{pre}}, \bZ)$ in SC framework
	(with $\star$ denoting unobserved and/or missing data in the donor matrix). ``?'' represents the 
	counterfactual observations for the target unit in the absence of intervention, which is what we wish to estimate.}
	\label{fig:sc}
\end{figure}
}

\iftoggle{FIGURES}{
\begin{figure}[!htb]
	\centering
	\caption{\,}
	\begin{subfigure}[t]{0.3\textwidth}
		\centering
		\includegraphics[width=\linewidth]{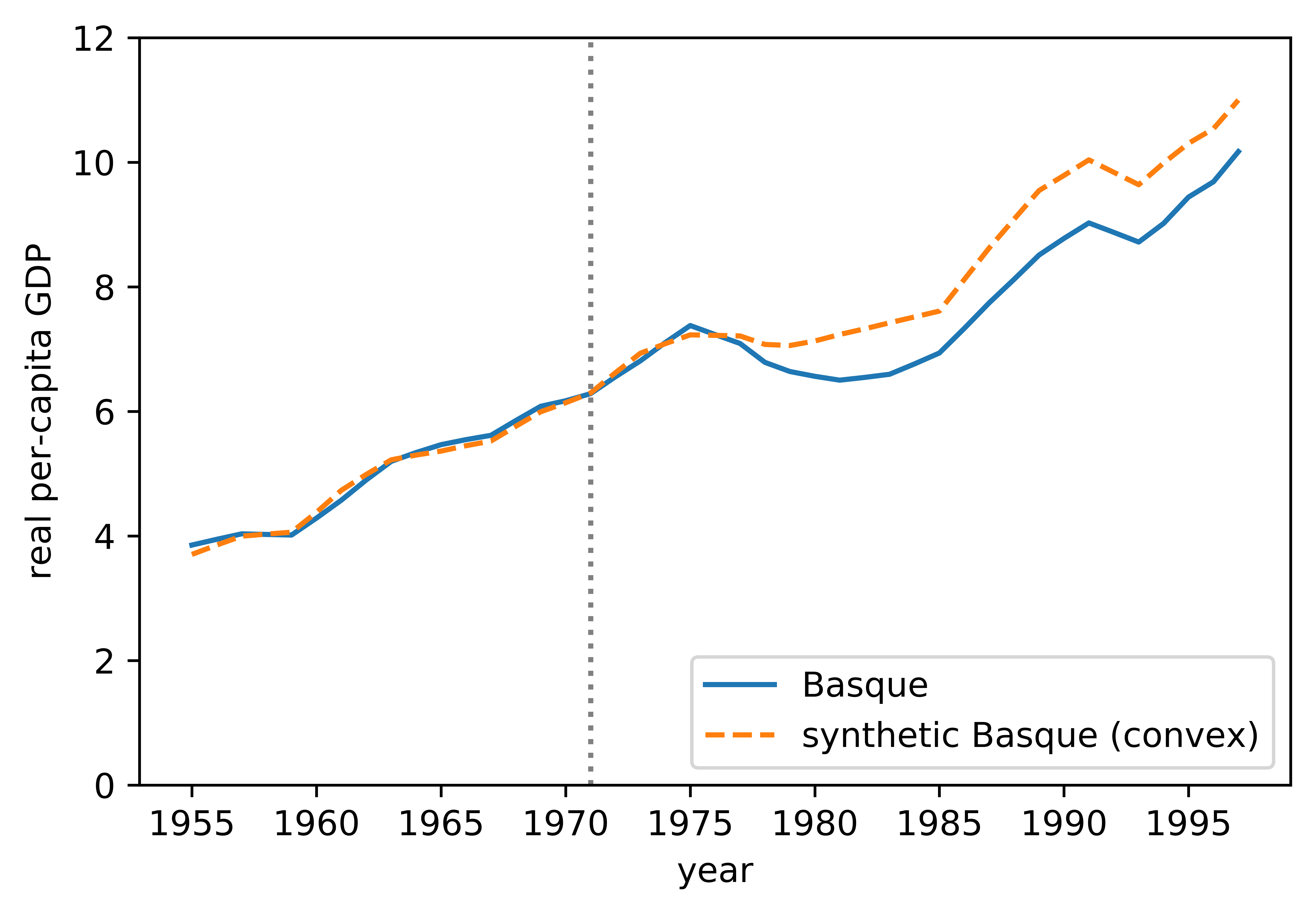}
		\caption{Synthetic Basque as predicted by \cite{abadie1}.}
		\label{fig:basque_cvx}
	\end{subfigure}
	~
	\begin{subfigure}[t]{0.3\textwidth}
		\centering
		\includegraphics[width=\linewidth]{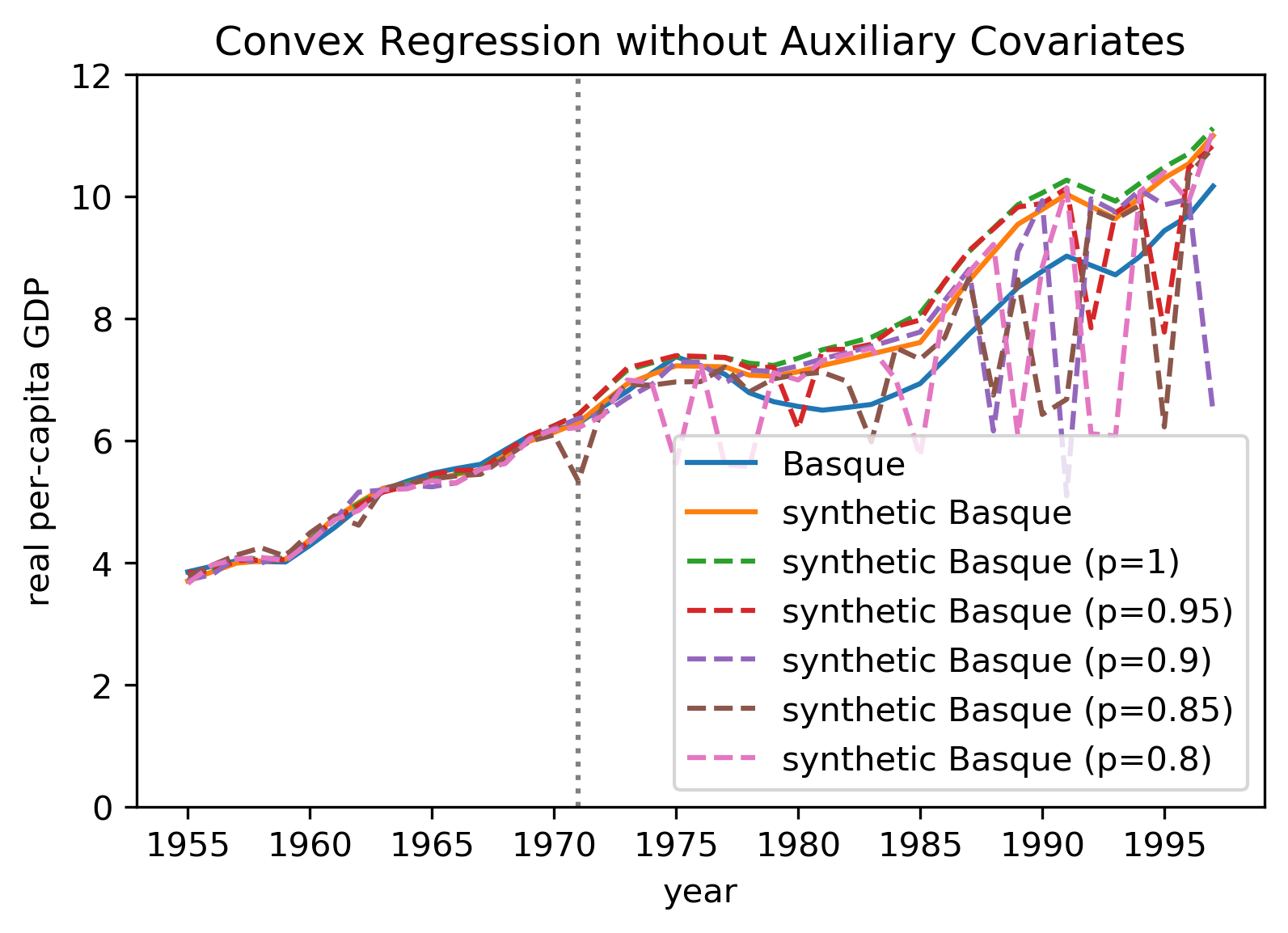}
		\caption{Synthetic Basque as predicted by \cite{abadie1} under varying levels of missing data.}
		\label{fig:basque_sc_mar}
	\end{subfigure}
	~
	\begin{subfigure}[t]{0.3\textwidth}
		\centering
		\includegraphics[width=\linewidth]{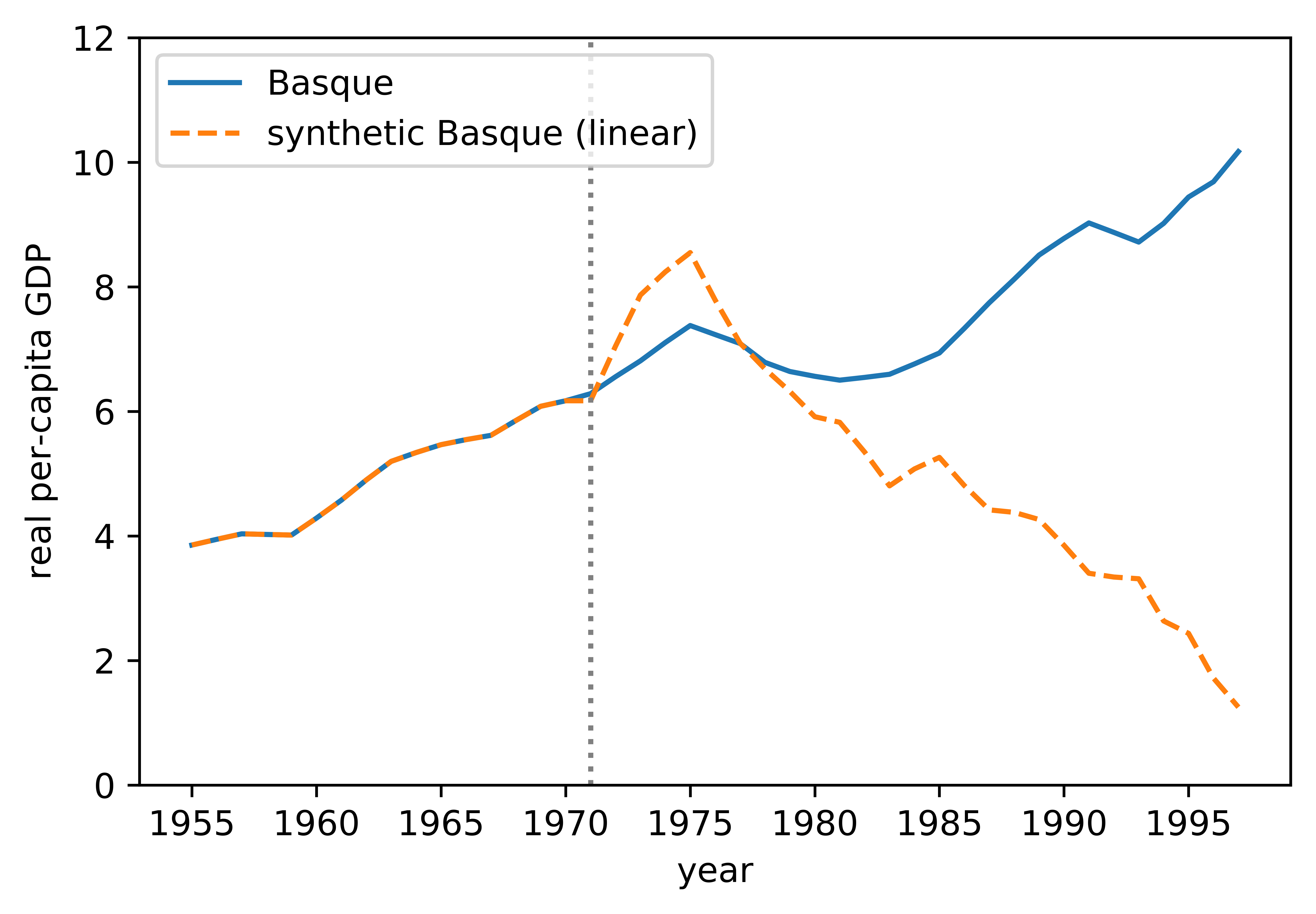}
		\caption{Synthetic Basque as predicted by Linear Regression.}
		\label{fig:basque_lr}
	\end{subfigure}
\end{figure}
}

\iftoggle{FIGURES}{
\begin{figure}[!htb]
	\centering
	\caption{\,}
	\begin{subfigure}[t]{0.4\textwidth}
		\centering
		\includegraphics[width=0.75\linewidth]{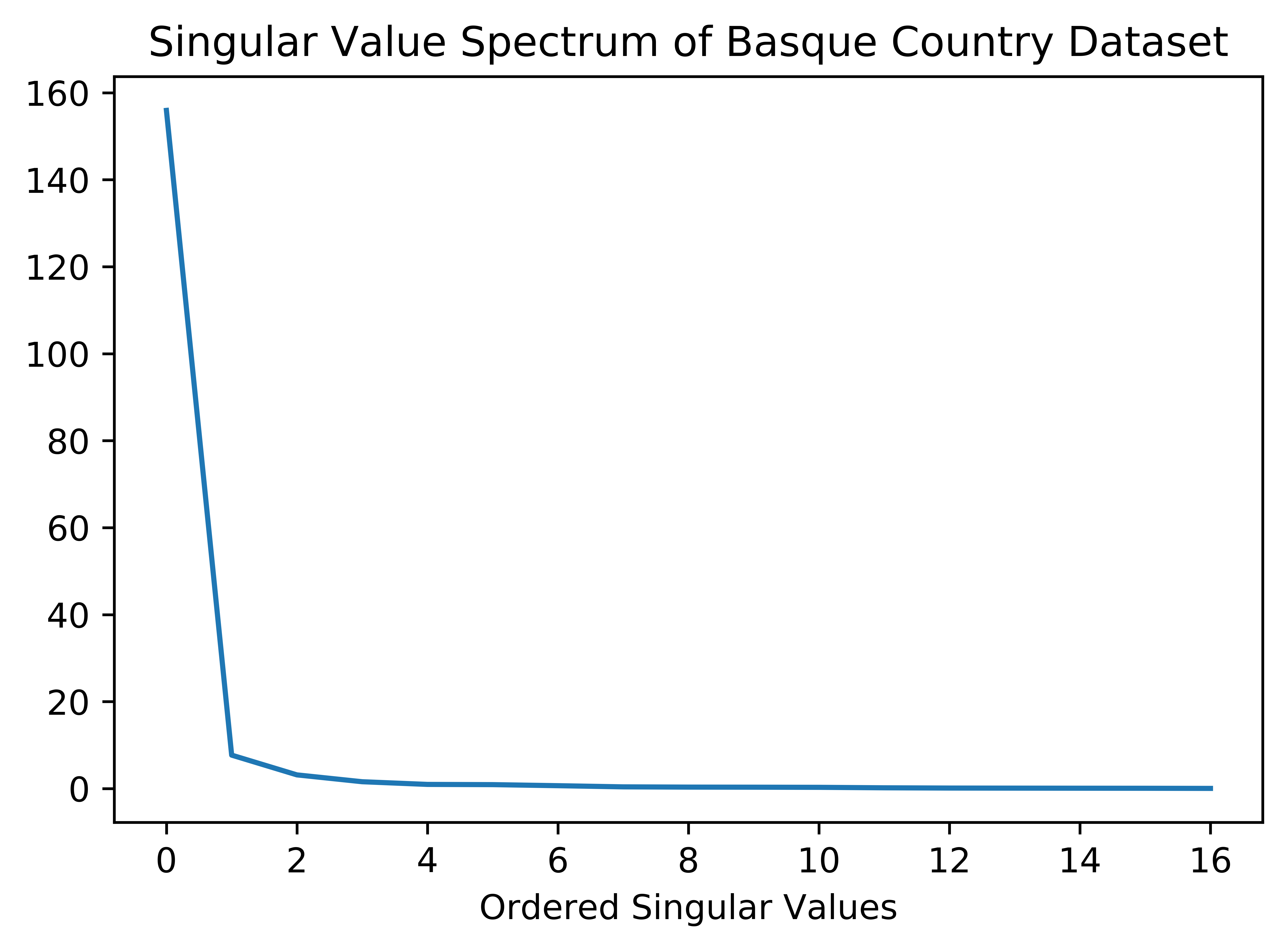}
		\caption{Singular value spectrum of Basque Country dataset.}
		\label{fig:basque_spectrum}
	\end{subfigure}
	~
	\begin{subfigure}[t]{0.4\textwidth}
		\centering
		\includegraphics[width=0.75\linewidth]{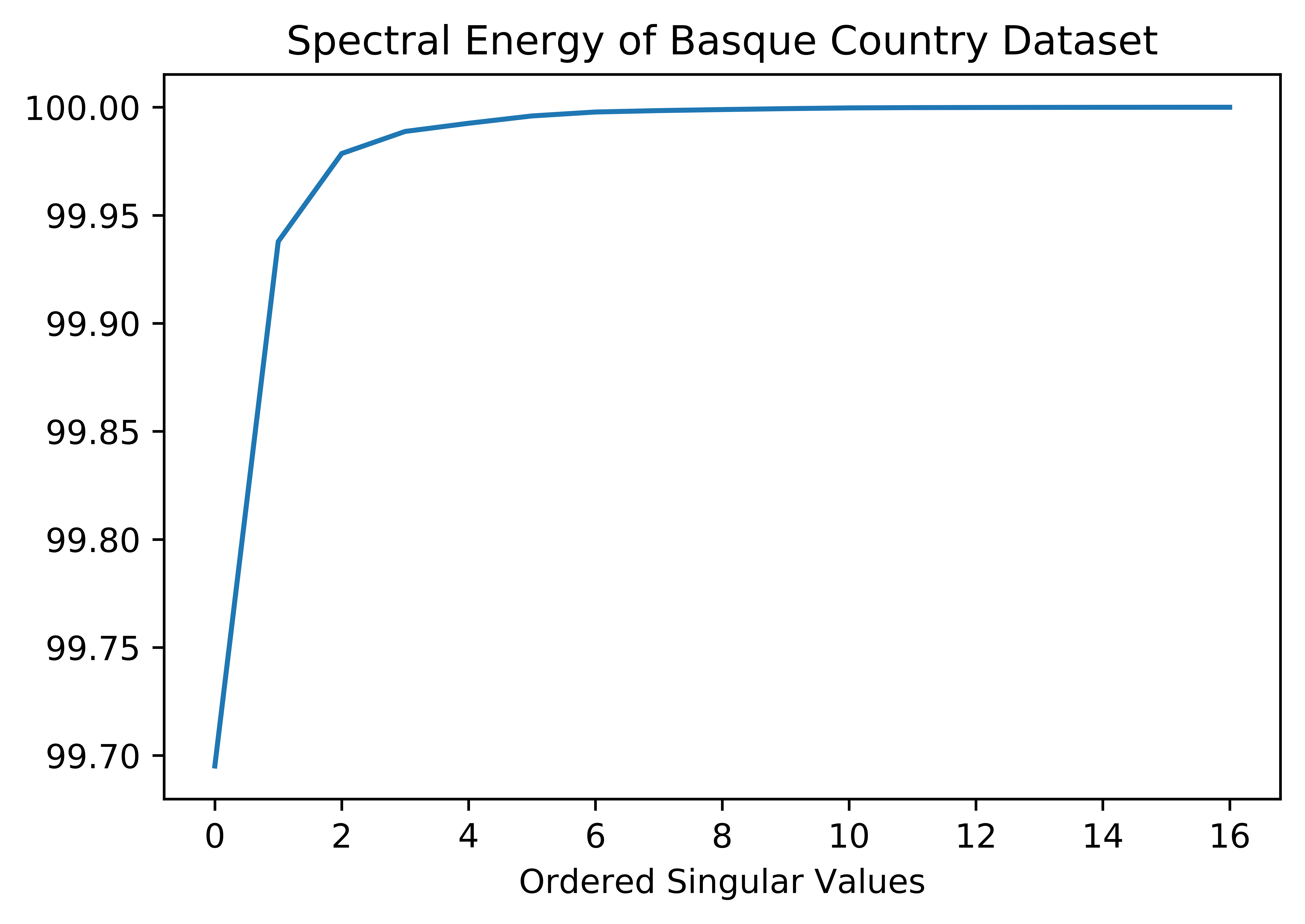}
		\caption{Spectral energy of Basque Country dataset.}
		\label{fig:basque_energy}
	\end{subfigure}
	\label{fig:basque2}
\end{figure}
}

\iftoggle{FIGURES}{
\begin{figure}[!htb]
	\centering
	\caption{\,}
	\begin{subfigure}[t]{0.4\textwidth}
		\centering
		\includegraphics[width=0.75\linewidth]{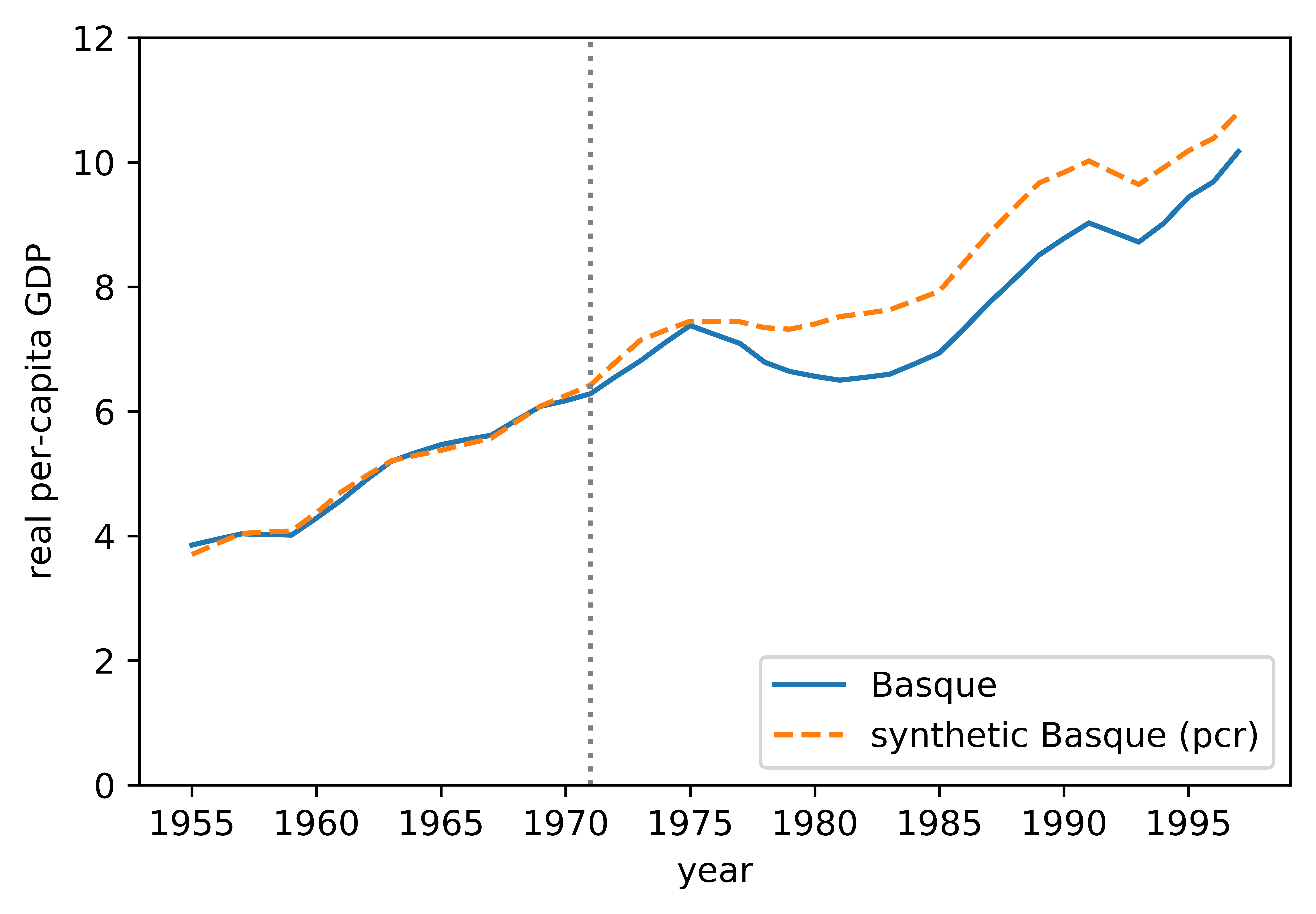}
		\caption{Synthetic Basque as predicted by PCR.}
		\label{fig:basque_pcr}
	\end{subfigure}
	~
	\begin{subfigure}[t]{0.4\textwidth}
		\centering
		\includegraphics[width=0.75\linewidth]{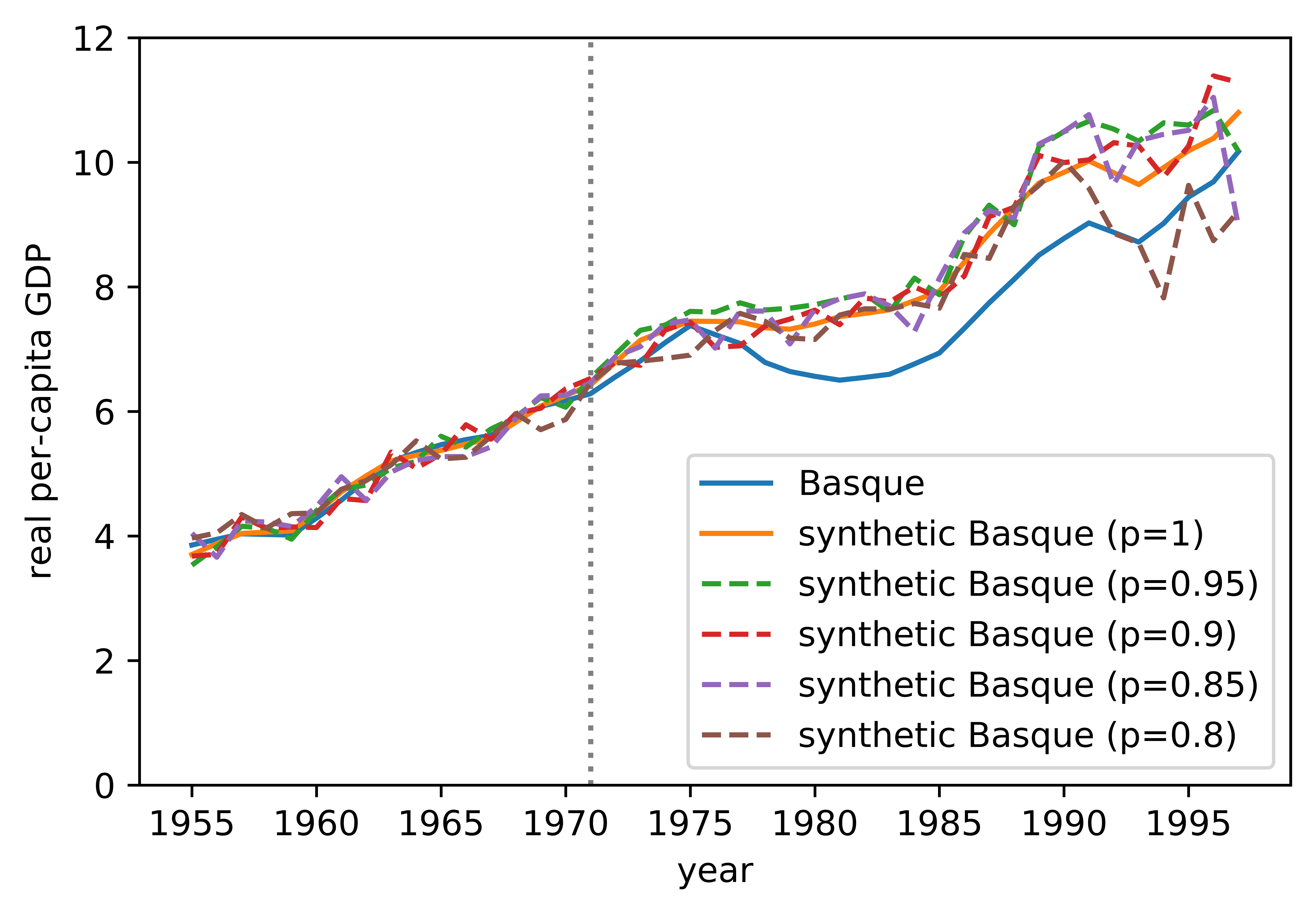}
		\caption{Synthetic Basque as predicted by PCR under varying levels of missing data.}
		\label{fig:basque_mar}
	\end{subfigure}
	\label{fig:basque3}
\end{figure}
}

\iftoggle{FIGURES}{
\begin{figure}[!htb]
	\centering
	\caption{\,}
	\begin{subfigure}[t]{0.4\textwidth}
		\centering
		\includegraphics[width=0.75\linewidth]{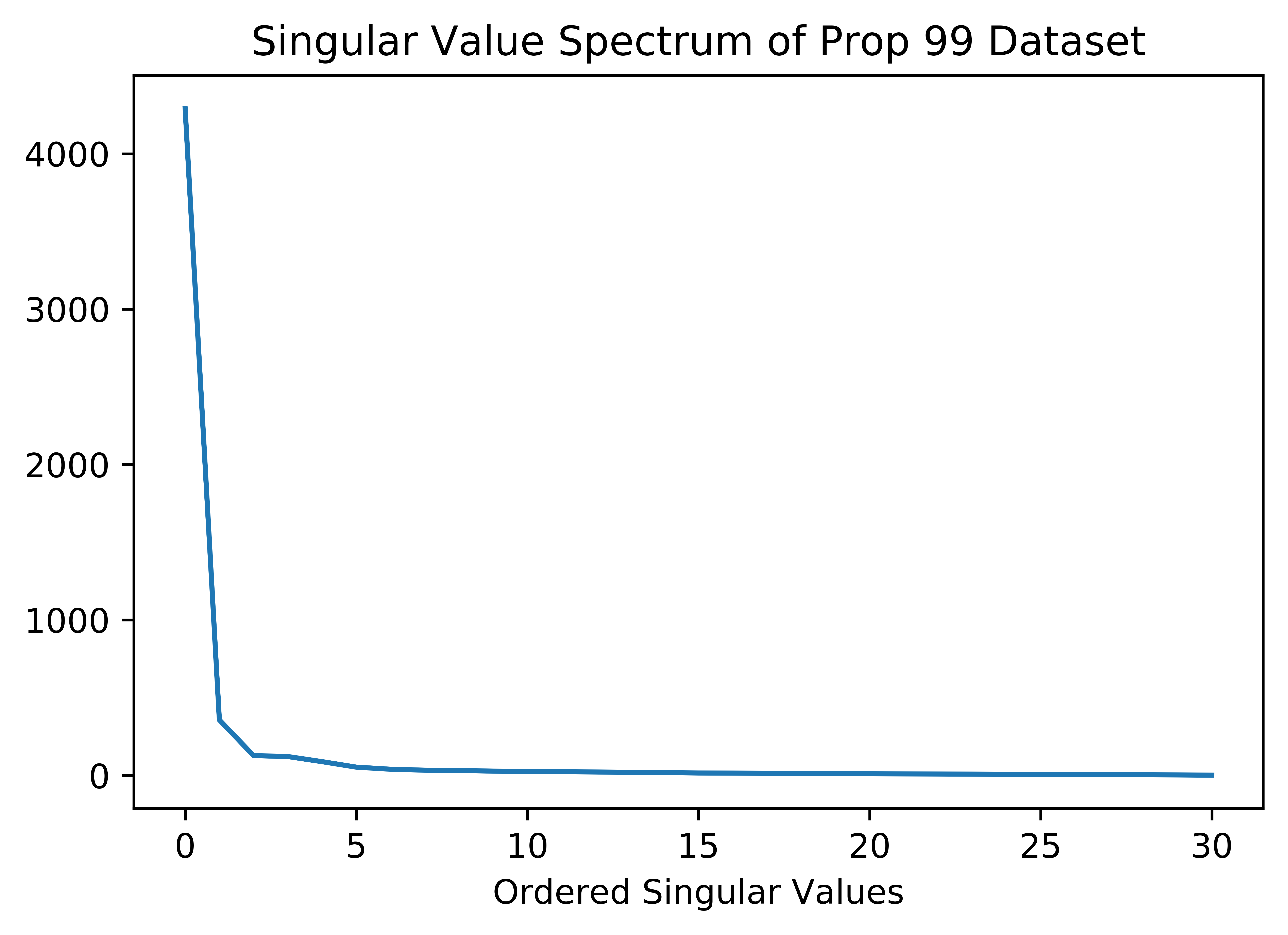}
		\caption{Singular value spectrum of California Prop 99 dataset.}
		\label{fig:cali_spectrum}
	\end{subfigure}
	~
	\begin{subfigure}[t]{0.4\textwidth}
		\centering
		\includegraphics[width=0.75\linewidth]{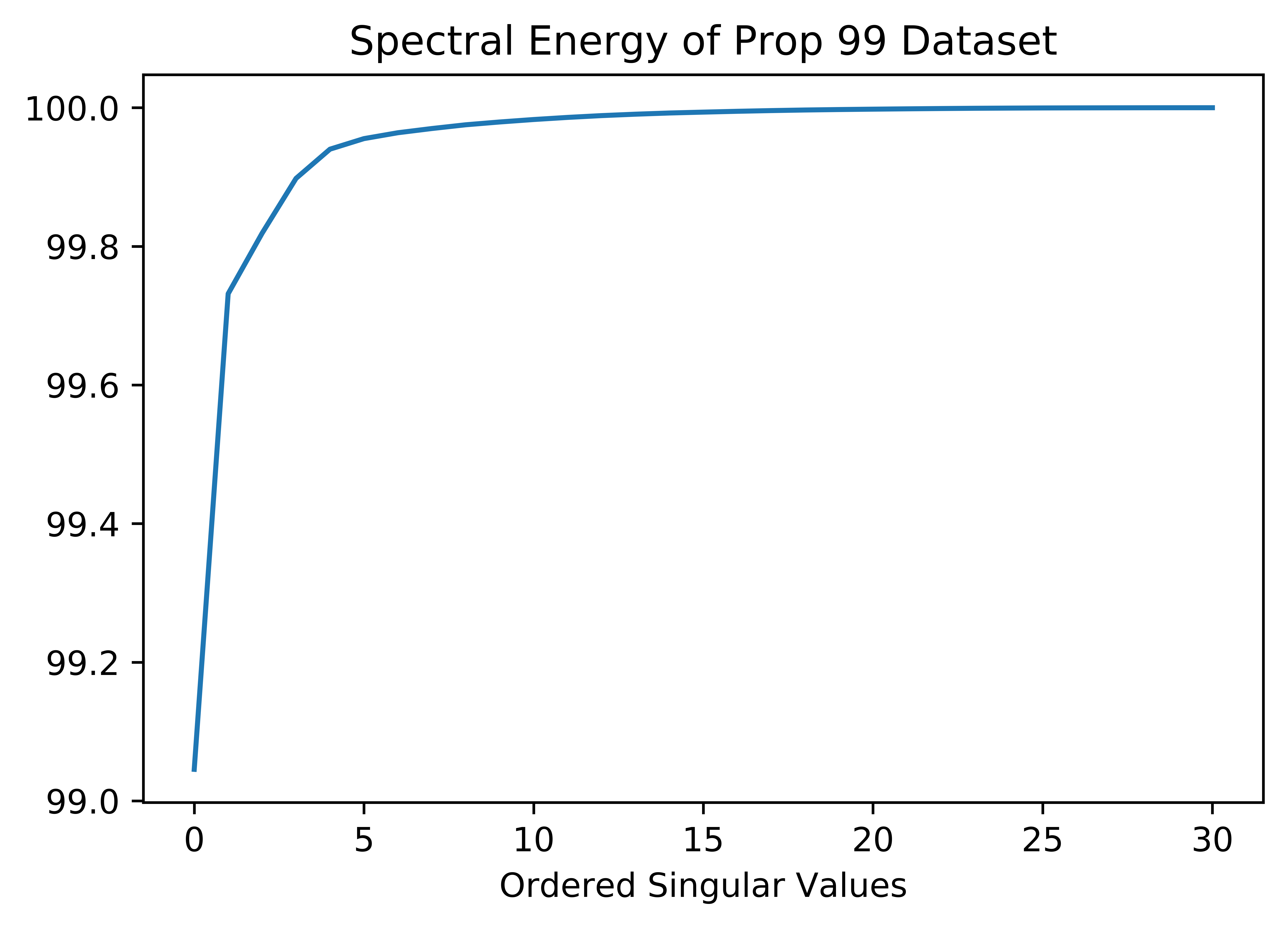}
		\caption{Spectral energy of California Prop 99 dataset.}
		\label{fig:cali_energy}
	\end{subfigure}
	\label{fig:cali1}
\end{figure} 
}

\iftoggle{FIGURES}{
\begin{figure}[!htb]
	\centering
	\caption{\,}
	\begin{subfigure}[t]{0.3\textwidth}
		\centering
		\includegraphics[width=\linewidth]{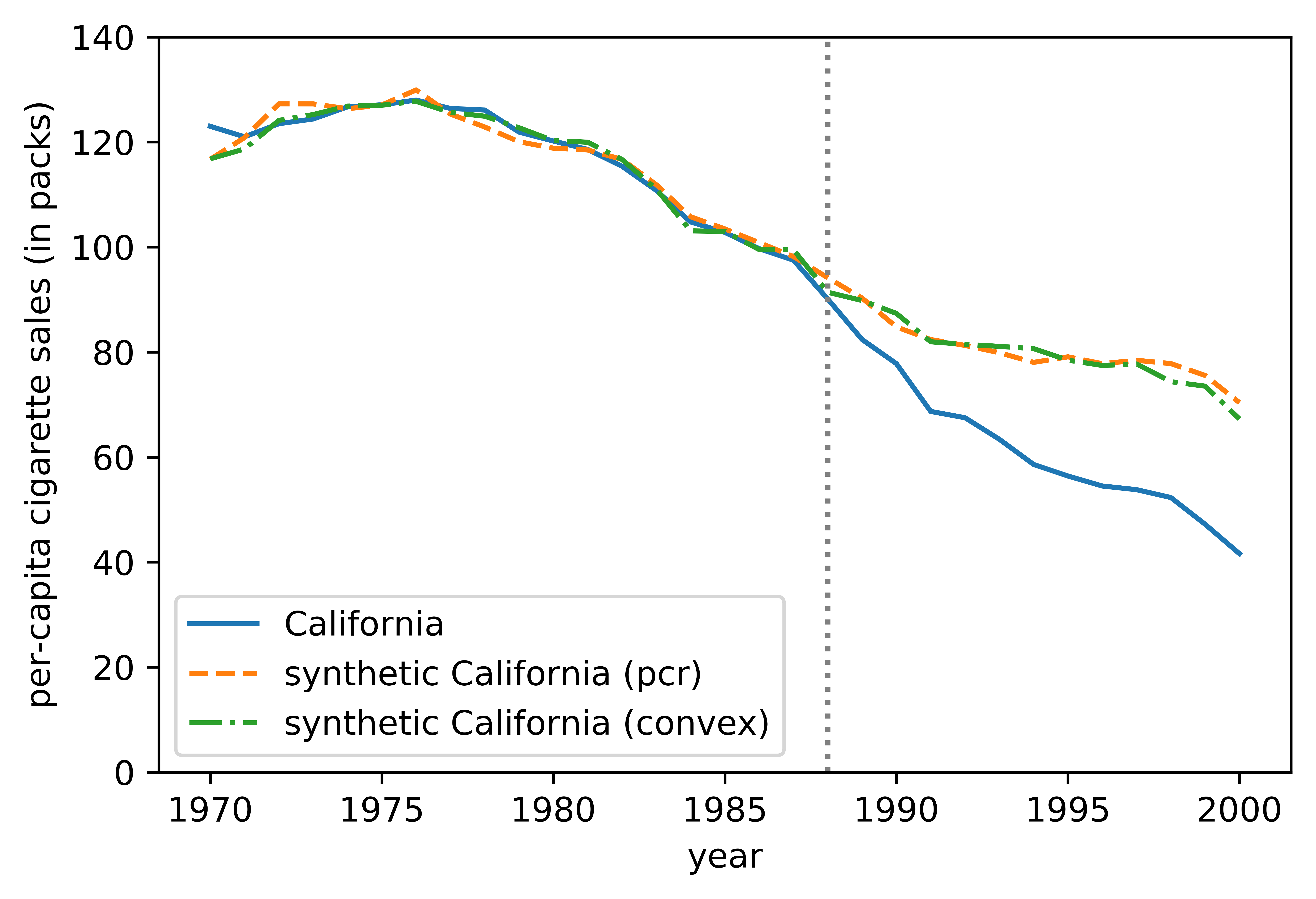}
		\caption{Synthetic California as predicted by PCR and \cite{abadie2}.}
		\label{fig:cali_pcr_cvx}
	\end{subfigure}
	~
	\begin{subfigure}[t]{0.3\textwidth}
		\centering
		\includegraphics[width=\linewidth]{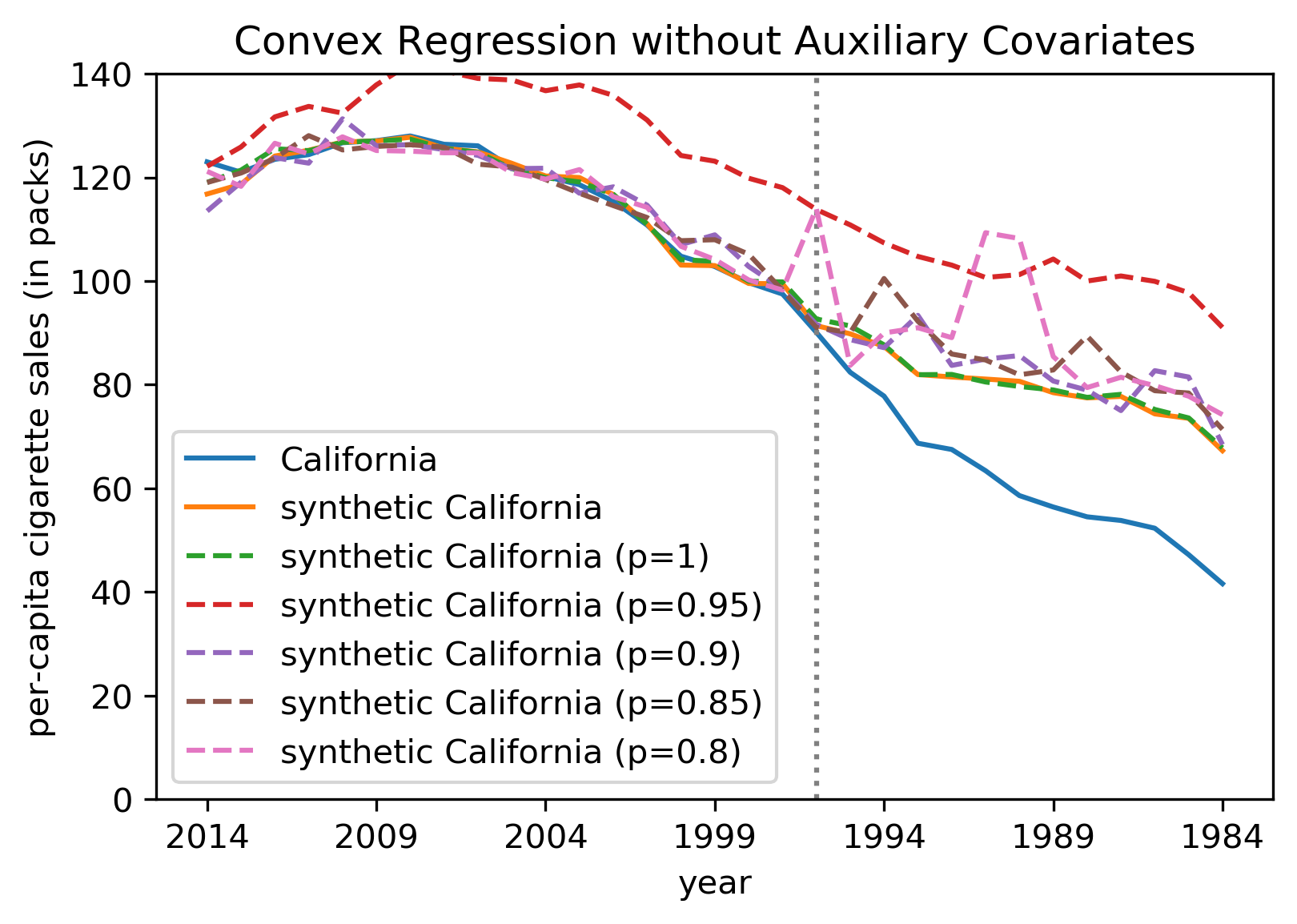}
		\caption{Synthetic California as predicted by \cite{abadie2} under varying levels of missing data.}
		\label{fig:cali_sc_mar}
	\end{subfigure}
	~
	\begin{subfigure}[t]{0.3\textwidth}
		\centering
		\includegraphics[width=\linewidth]{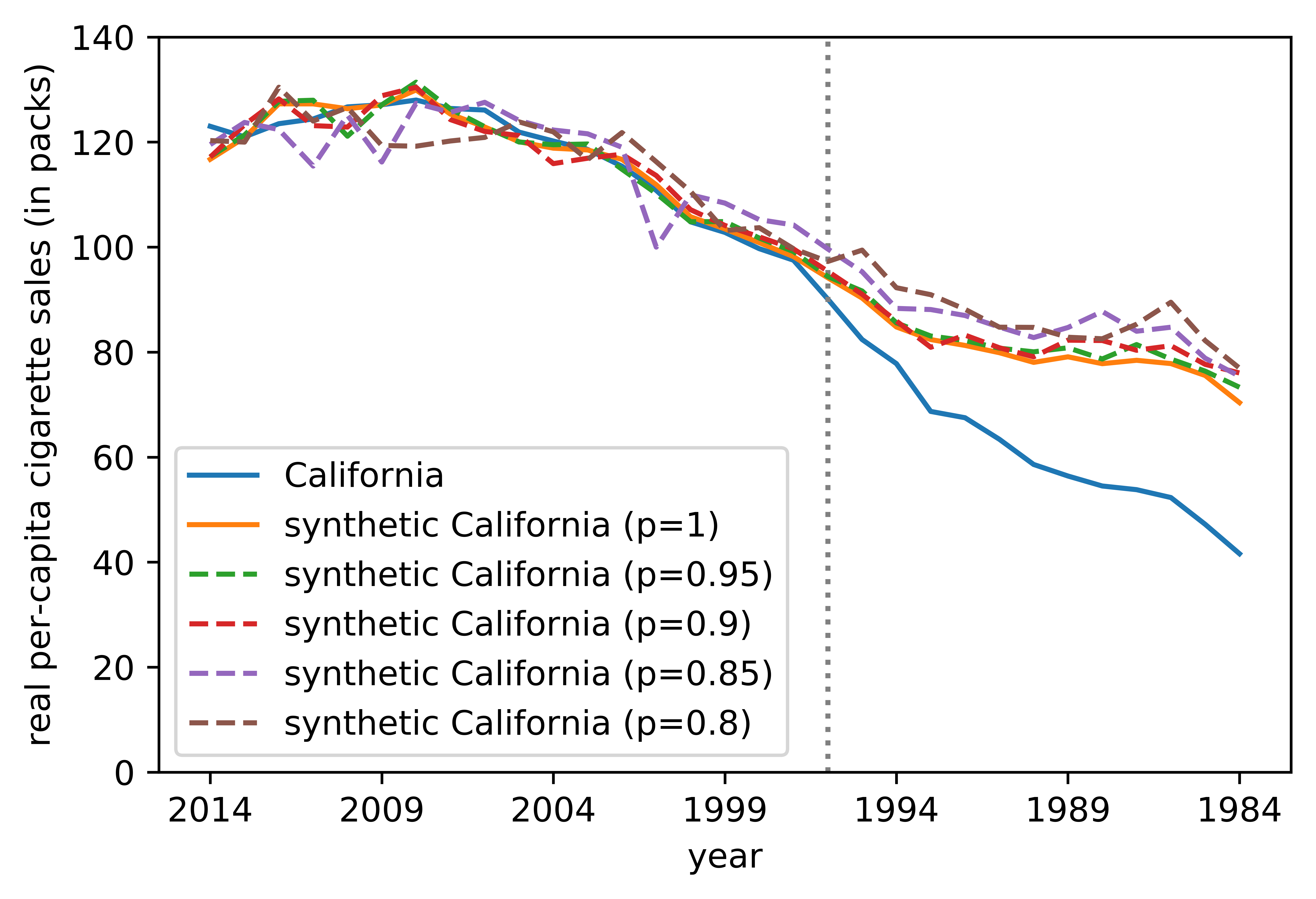}
		\caption{Synthetic California as predicted by PCR under varying levels of missing data.}
		\label{fig:cali_mar}
	\end{subfigure}
	\label{fig:cali2}
\end{figure} 
}


\begin{appendix}
\pagebreak

\clearpage
\begin{center}
{\LARGE \bf Online Supplement:} \\
{\LARGE \bf On Robustness of Principal Component Regression} \\
  Anish Agarwal, Devavrat Shah, Dennis Shen, Dogyoon Song \\
    MIT
\end{center}

\section{Related Works} \label{sec:lit_review}

We focus on the related literature pertaining to error-in-variable regression and PCR, but also 
include a brief discussion on the literature for matrix estimation/completion.

\medskip
\noindent
\textbf{Error-in-variables regression.}
There exists a rich body of work regarding high-dimensional error-in-variable regression (see \cite{loh_wainwright}, \cite{cocolasso}, \cite{tsybakov_1}, \cite{orthogonal_1}, \cite{weighted_l1}). Three common threads of these works include: (1) making a sparsity assumption on $\beta^*$; (2) establishing error bounds with convergence rates for estimating $\beta^*$ under different norms, i.e., $\| \hbeta - \beta^* \|_q$ where $\norm{\cdot}_q$ denotes the $\ell_q$-norm; (3) assuming the covariate matrix satisfying ``incoherence''-like condition such as the Restricted Eigenvalue Condition, cf. \cite{loh_wainwright}. In all of these works, the goal is to recover the underlying model, $\beta^*$. In contrast, as discussed, the goal of PCR is to primarily provide good prediction. Some notable works closest to our setup include \cite{loh_wainwright}, \cite{cocolasso}, \cite{tsybakov_2}, which are described in some detail next.

In \cite{loh_wainwright}, a non-convex $\ell_1$-penalization algorithm is proposed based on the plug-in principle to handle covariate measurement errors. This approach requires explicit knowledge of the unobserved noise covariance matrix $\Sigma_{\bH} = \Ex \bH^T \bH$ and the estimator designed \textit{changes} based on their assumption of $\Sigma_{\bH}$. They also require explicit knowledge of a bound on the $\norm{\cdot}_2$-norm of $\beta^*$,  the object they aim to estimate. In contrast, PCR does not require any such knowledge about the distribution of the noise matrix $\bH$ (i.e., the algorithm does not explicitly use this information to make predictions).

The work of \cite{cocolasso} builds upon \cite{loh_wainwright} by proposing a convex formulation of Lasso. Although the algorithm introduced does not require knowledge of $\norm{\beta^*}_2$, similar assumptions on $\bZ$ and $\bH$ (e.g., sub-gaussianity and access to $\Sigma_{\bH}$) are made. This renders their algorithm to not be {noise-}model agnostic. In fact, many works (e.g., \cite{tsybakov_1}, \cite{tsybakov_2}, \cite{tsybakov_3}) require either $\Sigma_{\bH}$ to be known or the structure of $\bH$ is such that it admits a data-driven estimator for its covariance matrix. This is so because these algorithms rely on correcting the bias for the matrix $\bZ^T \bZ$, which PCR does not need to compute. 

It is worth noting that all these works in error-in-variables regression focus only on parameter estimation (i.e., learning $\beta^*$) and not explicitly de-noising the noisy covariates. Thus, even with the knowledge of $\beta^*$, it is not clear how these methods can be used to produce predictions of the response variables associated with unseen, noisy covariates.



\vspace{12pt}
\noindent
\textbf{Principal Component Regression.} 
A notable work is that of \cite{pcr_tibshirani}, which suggests a variation of PCR to infer the direction of the principal components. However, it stops short of providing meaningful finite sample analysis beyond what is naturally implied by that of standard Linear Regression. The regularization property of PCR is also well known, at least empirically, due to its ability to reduce the variance. As a contribution, we provide rigorous finite sample guarantees of PCR: (i) under noisy, missing covariates; (ii) when the linear model is misspecified; (iii) when the low-rank model for covariate matrix is misspecified.

As a further contribution, we argue that PCR's regression model
has sparse support (established using the equivalence between PCR and Linear Regression with covariate pre-processing via HSVT); this sparsity allows for improved generalization as the Rademacher complexity of the resulting model class scales with the sparsity parameter (i.e., the rank of the covariate matrix pre-processed with HSVT). Hence, PCR not only addresses the challenge of noisy, missing covariates, but also, in effect, performs implicit regularization. 

\vspace{12pt}
\noindent
{\bf Matrix estimation.} Matrix estimation has spurred tremendous theoretical and empirical research across numerous fields (see \cite{CandesTao10, KeshavanMontanariOh10a, Recht11, Chatterjee15}), 
Traditionally, the end goal is to recover the underlying mean matrix from an incomplete, noisy sampling of its entries; the quality of the estimate is often measured through the Frobenius norm. Further, entry-wise independence and sub-gaussian noise is typically assumed. A key property of many matrix estimation methods is they are {noise-}model agnostic (i.e., the de-noising procedure does not change with the noise assumptions). We advance state-of-art for HSVT, arguably the most ubiquitous matrix estimation method, by (i) analyzing its error with respect to the $\ell_{2,\infty}$-norm and (ii) allowing for a broader class of noise distributions (e.g., sub-exponential). Such generalizations are necessary to enable the various applications detailed in Section \ref{sec:pcr_sc} and Appendices \ref{sec:private} and \ref{sec:mixed}. 


\section{Differentially Private Regression}\label{sec:private}

\medskip
\noindent {\bf Setup and Question.}
With the increasing use of machine learning for critical operations, analysts must maximize the accuracy of their predictions and simultaneously protect sensitive information (i.e., covariates). 
An important notion of privacy is that of differential privacy; this requires that the outcome of a database query cannot greatly change due to the presence or absence of any individual data record (see \cite{dwork_1} and references therein). 
More specifically, let $\delta$ be a positive real number, $\mathcal{D}$ be a collection of datasets, and $\mathcal{A}:\mathcal{D} \rightarrow {\rm im}(\mathcal{A})$ be a randomized algorithm that takes a dataset as input. 
The algorithm $\mathcal{A}$ is said to provide $\delta$-differential privacy if, for all datasets $\mathcal{D}_1$ and $\mathcal{D}_2$ in $\mathcal{D}$ that differ on a single element, and all subsets $\mathcal{S} \in {\rm im}(\mathcal{A})$, the following holds: 
\begin{align} \label{eq:diff_privacy}
	\Pb \left( \mathcal{A}(D_1) \in S \right) &\le \exp(\delta) \cdot \Pb \left( \mathcal{A}(D_2)\in S \right),
\end{align}
where the randomness lies in the algorithm. 
Thus, \eqref{eq:diff_privacy} guarantees that little can be learned about any particular record within the database. 

The canonical mechanism $\mathcal{A}$ to guarantee differential privacy is known as the Laplacian mechanism. 
In this setting, noise is drawn from a Laplacian distribution and added to query responses. 
In particular, introducing additive noise $W \sim {\rm Laplace}(0, \Delta_{f}/ \delta)$ to any database query guarantees $\delta$-privacy (see \cite{dwork_1} and references therein); here, $\Delta_{f} = \max_{\mathcal{D}_1, \mathcal{D}_2 \in \mathcal{D}} | f(\mathcal{D}_1) - f(\mathcal{D}_2) |$, where the maximum is taken over all pairs of datasets $\mathcal{D}_1$ and $\mathcal{D}_2$ in $\mathcal{D}$ differing in at most one element, and $f: \mathcal{D} \rightarrow \Reals^d$ is a vector-valued function denoting the true, latent query response. 
We now describe how PCR can be applied in the context of a differentially private framework. 

\vspace{12pt}
\noindent {\bf How it fits our framework.}
Let $\bA$ denote the true, fixed database of $N$ sensitive individual records and $p$ covariates. 
We consider the setting where an analyst is allowed to ask two types of queries of the data: (1) $f_{\bA}$ - querying for individual data records, i.e., $\bA_{i, \cdot}$ for $i \in [N]$; (2) $f_Y$ - querying for a linear combination of an individual's covariates, i.e. $\bA_{i, \cdot} \beta^*$. A typical example would be where $\bA_{i, \cdot}$ is the genomic information for patient $i$ and $\bA_{i, \cdot} \beta^*$ denotes patient $i$'s outcome for a clinical study. 

In order to provide $\delta$-differential privacy, the Laplacian mechanism will return query responses with additive Laplacian noise. For query type (1), let  $Z_{ij}$ for $i \in [N], j \in [p]$ be the returned response; here, $Z_{ij} = A_{ij} + \eta_{ij}$ with probability $\rho$ and $Z_{ij} = \star$ with probability $1 - \rho$, where $\eta_{i, \cdot} = [\eta_{ij}]$ for $j \in [p]$ is independent Laplacian noise with the variance parameter proportional to $\Delta_{f_\bA} / \delta$; we note that an auxiliary benefit of our setup is that it allows for a significant fraction of the query response to be masked, in addition to to the Laplacian noise corruption. For query type (2), when an analyst queries for the response variable $\bA_{i, \cdot} \beta^*$, she observes $Y_i = \bA_{i, \cdot} \beta^* + \epsilon_i$, where $\epsilon_i$ is again independent Laplacian noise with variance parameter proportional to  $\Delta_{f_Y} / \delta$. We note that the above setup naturally fits our framework since the Laplacian distribution belongs to the family of sub-exponential distributions, i.e., satisfying Property \ref{prop:covariate_noise_structure} with $\alpha = 1$. 

Finally, let $Y^\Omega$ denote the $n$ noisy observed responses (e.g., corresponding to the outcomes of $n$ patient clinical trials), and let $\bZ$ denote the noisy observed covariates (e.g., the collection of genomic information of all $N$ patients). Ultimately, the goal in such a setup is to accurately learn in- and out-of-sample global statistics (e.g., having low $\text{MSE}_{\Omega}(\hY)$ and $\text{MSE}(\hY)$, respectively) about the data, while preserving the individual privacy of the users.

\vspace{12pt}
\noindent 
\textit{Is privacy preserved?} 
Lemma \ref{lemma:mcse_hvst} demonstrates that the estimated covariate matrix $\bZ^{\text{HSVT}, k} $ via HSVT achieves small average $\|\cdot\|_{2, \infty}$-norm error (column-squared error); hence, for instance, HSVT can accurately learn the \textit{average} age of all patients. However, this does not translate to accurately estimating the age of any particular patient -- this would correspond to a small $\|\cdot\|_{\infty}$-norm error. Similarly, Corollary \ref{cor:training_pcr_differential} (stated below), establishes that PCR can estimate the vector $\bA \beta^*$ well on average, but not any particular element of this vector. 
We leave it as an open question as to whether or not de-noising the covariate matrix through HSVT can give a $\|\cdot\|_{\infty}$-norm bound.


\vspace{12pt}
\noindent {\bf Results.} 
We now state the following corollary, an instantiation of Corollary \ref{thm:training_pcr}, which demonstrates the efficacy of PCR (with respect to prediction) in the context of differential privacy. 
We note a similar bound could easily be produced for any of the results in Section \ref{sec:results} 
--  see \eqref{eq:mse_upper_generic}, \eqref{eq:mse_train_hsvt_simple}, \eqref{eq:mse_upper_generic_refined}, \eqref{eq:geo_decay}, \eqref{eq:mse_upper_generic_refined_LVM}, \eqref{eq:mse_upper_generic_refined_LVM_detailed}, \eqref{eq:mse_test_hsvt} -- 
by appropriately substituting $\gamma, K_\alpha$ with $\frac{\Delta_{f_\bA}}{\delta}$ and  $\sigma$ with $\frac{\Delta_{f_Y}}{\delta}$.

\begin{cor}\label{cor:training_pcr_differential}
Let the conditions of Corollary \ref{thm:training_pcr}.
Let $\eta_{ij}$ be sampled independently from $\sim {\rm Laplace}(0, \Delta_{f_\bA}/ \delta)$ for $i \in [N], j \in [p]$. 
Let  $\epsilon_i$ be sampled independently from $\sim {\rm Laplace}(0, \Delta_{f_Y}/ \delta)$. 
Let $n = \Theta(N)$. 
Then, PCR preserves $\delta$-differential privacy of $\bA$ and $\bA \beta^*$ with
\begin{align} \label{eq:mse_train_hsvt_simple}
\emph{MSE}_{\Omega}(\hY) 
&\le \frac{C' \| \beta^*\|_1^2}{\rho^4} \, \frac{r \log^5(np)}{n \wedge p} + \frac{20 \|\phi\|_2^2}{n} ,
\end{align}
where $C' = C(1+ (\Delta_{f_Y}/ \delta)^2)(1+ (\Delta_{f_\bA}/ \delta)^8)$ and $C>0$ is an absolute constant.
\end{cor}
\begin{proof}
Proof is immediate from Corollary \ref{thm:training_pcr} by substituting $\gamma, K_\alpha$ with $\frac{\Delta_{f_\bA}}{\delta}$ and  $\sigma$ with $\frac{\Delta_{f_Y}}{\delta}$.
\end{proof}
\noindent {\em Interpretation.}
From Corollary \ref{cor:training_pcr_differential}, we observe that PCR learns a predictive linear model in a differentially private framework, where the covariates are purposefully contaminated with Laplacian noise to maintain $\delta$-differential privacy. 

\section{Regression with Mixed Valued Covariates}\label{sec:mixed}

\medskip
\noindent {\bf Setup and Question.} Regression models with mixed discrete and continuous covariates are ubiquitous in practice. 
With respect to discrete covariates, a standard generative model assumes the covariates are generated from a categorical distribution (i.e., a multinomial distribution). Formally, a categorical distribution for a random variable $X$ is such that $X$ has support in $[G]$ and the probability mass function (pmf) is given by $\Pb(X = g) = \rho_g$ for $g \in [G]$ with $\sum_{g = 1}^G \rho_g= 1$. 

For simplicity, we focus on the case where the regression is being done with a collection of Bernoulli random variables (i.e., each $X$ has support in $\{0, 1\}$). The extension to general categorical random variables is straightforward and discussed below. 

A standard model in regression with Bernoulli random variables assumes that the response variable is a linear function of the latent parameters of the observed discrete outcomes. Formally, $\bA_{i, \cdot} =  [\rho^{(i)}_1, \rho^{(i)}_2, \dots, \rho^{(i)}_p] \in \Reals^{1 \times p}$, where $\rho^{(i)}_j$ for $j \in [p]$ is the latent Bernoulli parameter for the $j$-th feature and $i$-th measurement. Further, the mean of the response variable satisfies $\Ex[Y_i] = \sum_{j = 1}^p \rho^{(i)}_j \beta_j$. However, for each feature, we only get binary observations, i.e., $X_{ij} \in \{0, 1\}$.

As an example, consider $\Ex[Y_i]$ to be the expected health outcome of patient $i$. Let there be a total of $p$ possible observable binary symptoms (e.g., cold, fever, headache, etc.). Then $\bA_{i, \cdot}$ denotes the vector of (unobserved) probabilities that patient $i$ has some collection of symptoms (e.g., $A_{i1} = \mathbb{P}(\text{patient $i$ has a cold}), A_{i2} = \mathbb{P}(\text{patient $i$ has a fever}), \dots$). However, for each patient, we only observe the ``noisy'' binary outcome of these symptoms (i.e., $X_{i1} = \mathbb{1}(\text{patient $i$ has a cold}), X_{i2} = \mathbb{1}(\text{patient $i$ has a fever})$).
Ideally, we get to observe the underlying probabilities of the symptoms as that is what we assume the response is linearly related to. 
The objective in such a setting is to accurately recover $\bA \beta^*$ given $Y^{\Omega}$ and $\bX$.  

\vspace{12pt}
\noindent \textit{Current practice for mixed valued features.} A common practice for regression with categorical variables is to build a separate regression model for every possible combination of the categorical outcomes (i.e., to build a separate regression model conditioned on each outcome). In the healthcare example above, this would amount to building $2^p$ separate regression models corresponding to each combination of the observed $p$ binary symptoms. This is clearly not ideal for the following two major reasons: (i) the sample complexity is exponential in $p$; (ii) we do not have access to the underlying probabilities $\bA_{i, \cdot}$  (recall $\bX_{i, \cdot} \in \{0, 1\}^p$), which is what we actually want to regress $Y^{\Omega}$ against.

\vspace{12pt}
\noindent {\bf How it fits our framework.} Recall from Property \ref{prop:covariate_noise_structure} that the key structure we require of the covariate noise $\eta_{ij}$ is that $\Ex[\eta_{ij}] = 0$. Now even though $X_{ij} \in \{0, 1\}$, it still holds that $\Ex[X_{ij}] = \rho^{(i)}_j = A_{ij}$, which immediately implies $\Ex[\eta_{ij}] = \Ex[X_{ij} - A_{ij}] = 0$. Further, $\eta_{ij}$ is sub-Gaussian ($\alpha = 2$) since $| \eta_{ij} | \le 1$. Thus, the key conditions on the noise are satisfied for PCR to effectively (in the $\|\cdot\|_{2, \infty}$-norm) de-noise $\bX$ to recover the underlying probability matrix $\bA$; this, in turn, allows PCR to produce accurate estimates $\bhA \hbeta$ through regression, as seen by Theorem \ref{thm:test_pcr}. 


Pleasingly, the required sample complexity grows with the rank of $\bA$ (the inherent model complexity of the underlying probabilities), rather than exponentially in $p$. Further, the de-noising step allows us to regress against the estimated latent probabilities rather than their ``noisy", binary outcomes.

\vspace{12pt}
\noindent \textit{Extension from Bernoulli to general categorical random variables.} Recall from above that a categorical random variable has support in $[G]$ for $G \in \mathbb{N}$. In this case, one can translate a categorical random variable to a a collection of binary random variables using the standard one-hot encoding method. It is worth highlighting that by using one-hot encoding, clearly $\eta_{ij_1}$ will not be independent of $\eta_{i j_2}$ for any $(j_1, j_2)$ pair, which encodes the same categorical variable. 
However, from Property \ref{prop:covariate_noise_structure}, we only require independence of the noise across rows, not within them. Thus this lack of independence is not an issue. Further, the generalization to multiple categorical variables, in addition to continuous covariates, is achieved by simply appending these features to each row and collectively de-noising the entire matrix before the regression step.



\section{Useful Theorems Known from Literature}\label{sec:useful_theorems}

\subsection{Bounding $\psi_{\alpha}$-norm}
\begin{lemma} \label{lemma:sum_of_subgaussians} {\bf Sum of independent sub-gaussians random variables.} \\
Let $X_1, \dots, X_n$ be independent, mean zero, sub-gaussian random variables. Then $\sum_{i=1}^n X_i$ is also a sub-gaussian random variable, and
\begin{align}
	\Big \| \sum_{i=1}^n X_i\Big \|_{\psi_2}^2 &\le C \sum_{i=1}^n \norm{X_i}_{\psi_2}^2
\end{align}
where $C$ is an absolute constant.
\end{lemma}

\begin{lemma} \label{lemma:subgauss_subexp} {\bf Product of sub-gaussians is sub-exponential.}\\
	Let $X$ and $Y$ be sub-gaussian random variables. Then $XY$ is sub-exponential. Moreover,
	\begin{align}
		\norm{XY}_{\psi_1} &\le \norm{X}_{\psi_2} \norm{Y}_{\psi_2}.
	\end{align}
\end{lemma}

\subsection{Concentration Inequalities for Random Variables}


\begin{lemma}\label{lem:general_bernsteins} {\bf Bernstein's inequality.}\\
Let $X_1, X_2, \dots, X_N$ be independent, mean zero, sub-exponential random variables. Let $S = \sum_{i=1}^n X_i$. Then for every $t > 0$, we have 
\begin{align}
\mathbb{P} \{ \abs{S} \ge t  \} \le 2 \exp ( -c \min \Bigg[\frac{t^2}{\sum^N_{i=1} \norm{X_i}^2_{\Psi_1}}, \frac{t}{\max_i \norm{X_i}_{\Psi_1}} \Bigg] )
\end{align}
\end{lemma}

\begin{lemma} \label{lem:mcdiarmid} {\bf McDiarmid inequality.}\\ 
	Let $x_1, \dots, x_n$ be independent random variables taking on values in a set $A$, and let $c_1, \dots, c_n$ be positive real constants. If $\phi: A^n \rightarrow \mathbb{R}$ satisfies
	\begin{align*}
		\sup_{x_1, \dots, x_n, x_i' \in A} \abs{\phi(x_1, \dots, x_i, \dots, x_n) - \phi(x_1, \dots, x'_i, \dots, x_n)} &\le c_i,
	\end{align*}
	for $1 \le i \le n$, then
	\begin{align*}
		\Pb \Big\{ \abs{\phi(x_1, \dots, x_n) - \Ex \phi(x_1, \dots, x_n)} \ge \epsilon \Big\} &\le \exp(\frac{-2\epsilon^2}{\sum_{i=1}^n c_i^2}). 
	\end{align*}
\end{lemma}

\subsubsection{Upper Bound on the Maximum Absolute Value in Expectation}

\begin{lemma} \label{lemma:max_subg} {\bf Maximum of sequence of random variables. } \\
	Let $X_1, X_2, \dots, X_n$ be a sequence of random variables, which are not necessarily independent, and satisfy $\Ex[X_i^{2p}]^{\frac{1}{2p}} \le K p^{\frac{\beta}{2}}$ for some $K, \beta >0$ and all $i$. Then, for every $n \ge 2$,
	\begin{align}
		\Ex \max_{i \le n} \abs{X_i} &\le C K \log^{\frac{\beta}{2}}(n).
	\end{align}
\end{lemma}

\begin{remark}\label{rem:max_psialpha}
Lemma \ref{lemma:max_subg} implies that if $X_1, \ldots, X_n$ are $\psi_{\alpha}$ random variables with 
$\| X_i \|_{\psi_{\alpha}} \leq K_{\alpha}$ for all $i \in [n]$, then 
\begin{align*}
	\Ex \max_{i \le n} \abs{X_i} &\le C K_{\alpha} \log^{\frac{1}{\alpha}}(n).
\end{align*}
\end{remark}

%

\subsection{Other Useful Lemmas}

\begin{lemma}  {\bf Perturbation of singular values (Weyl's inequality).} \label{lem:weyls}\\
	Let $\bA$ and $\bB$ be two $m \times n$ matrices. Let $k = m \wedge n$. Let $\lambda_1,\dots, \lambda_k$ be the singular values of $\bA$ in decreasing order and repeated by multiplicities, and let $\tau_1, \dots, \tau_k$ be the singular values of $\bB$ in decreasing order and repeated by multiplicities. Let $\delta_1, \dots, \delta_k$ be the singular values of $\bA - \bB$, in any order but still repeated by multiplicities. Then,
	\begin{align*}
	\max_{1 \le i \le k} \abs{ \lambda_i - \tau_i} &\le \max_{1 \le i \le k} \abs{ \delta_i}.
	\end{align*}
\end{lemma}

\section{Definitions} \label{sec:definitions}

\begin{definition} [$\psi_\alpha$-random variables/vectors] \label{def:psialpha}
For any $\alpha \geq 1$, we define the $\psi_{\alpha}$-norm of a random variable $X$ as $\norm{X}_{\psi_{\alpha}} = \inf \{ t > 0: \Ex \exp(|X|^{\alpha} /t^{\alpha}) \le 2 \}$.
If $\norm{X}_{\psi_{\alpha}} < \infty$, we call $X$ a $\psi_{\alpha}$-random variable. 
More generally, we say $X$ in $\mathbb{R}^n$ is a $\psi_{\alpha}$-random vector if all one-dimensional marginals $\langle X, v \rangle$ are $\psi_{\alpha}$-random variables for any fixed vector $v \in \mathbb{R}^n$. 
We define the $\psi_{\alpha}$-norm of the random vector $X \in \mathbb{R}^n$ as $\norm{X}_{\psi_{\alpha}} = \sup_{v \in \mathcal{S}^{n-1}} \norm{ \langle X, v \rangle }_{\psi_{\alpha}}$,
where $\mathcal{S}^{n-1} := \{ v \in \mathbb{R}^n: \norm{v}_2 = 1\}$, $\langle \cdot, \cdot \rangle$ usual inner product. Note that $\alpha = 2$ and $\alpha =1 $ represent the class of sub-gaussian and sub-exponential random variables/vectors, respectively. 
\end{definition}

\begin{definition} [Restricted Eigenvalue (RE) condition] \label{def:re}
For some $\alpha \ge 1$, and non-empty subset $S \in [p]$, let
\[ \mathcal{C}_\alpha(S) = \{ \Delta \in \Reals^p: \| \Delta_{S^c}\|_1 ~\le~ \alpha \| \Delta_S\|_1\},\]
where $S^c = [p] \setminus S$ and $\Delta_S = \{\Delta_j: j \in S\}$. 

We say that $\bX \in \Reals^{n \times p}$ satisfies the $\text{RE}(\alpha, \kappa)$ condition w.r.t. $S$ if 
\[ \frac{1}{n} \| \bX \Delta \|_2^2 ~\ge~ \kappa \|\Delta\|_2^2, \quad \forall \Delta \in \mathcal{C}_\alpha(S)\]
where $\kappa > 0$. 
\end{definition}

\section{Proof of Proposition \ref{prop:equivalence}: Equivalence between PCR and HSVT-OLS} \label{sec:appendix_equivalence}
\begin{proof}[Proof of Proposition \ref{prop:equivalence}]
Using the orthonormality of $\bU, \bV$, we obtain
\begin{align}
	\hY^{\pcr, k} &= \btZ \cdot \bV_k \cdot \beta^{\pcr, k}
	 \\ & = \btZ \cdot \bV_k \cdot \left( \bZ^{\pcr, k, \Omega} \right)^{\dagger} Y^{\Omega}\nonumber
	\\ &= \bU \cdot \bS \cdot \bV^T \cdot \bV_k \cdot \left( (\btZ \cdot \bV_k)^{\Omega} \right)^{\dagger} \cdot Y^{\Omega}
 	\\ & = \bU_k \cdot \bS_k \cdot \left( (\bU_k \cdot \bS_k)^{\Omega} \right)^{\dagger} \cdot Y^{\Omega}\nonumber
	\\ &= \bU_k \cdot \bS_k \cdot \left( \bU_k^{\Omega} \cdot \bS_k \right)^{\dagger} \cdot Y^{\Omega}
	\\ &= \bU_k \cdot \bS_k \cdot \bS_k^{-1} (\bU_k^{\Omega})^T \cdot Y^{\Omega} \nonumber
	\\ &= \bU_k \cdot (\bU_k^{\Omega})^T \cdot Y^{\Omega}. \label{eq:prop2.1}
\end{align}
Similarly, 
\begin{align}
	\hY^{\hsvt, k} &= \bZ^{\hsvt, k} \cdot \beta^{\hsvt, k}~=~ \bZ^{\hsvt, k} \cdot \left( \bZ^{\hsvt, k, \Omega} \right)^{\dagger} \cdot Y^{\Omega}
	\nonumber \\ &= \bU_k \cdot \bS_k \cdot \bV_k^T \cdot \left( (\bU_k \cdot \bS_k \cdot \bV_k^T)^{\Omega} \right)^{\dagger} \cdot Y^{\Omega} 
	\nonumber \\ &= \bU_k \cdot \bS_k \cdot \bV_k^T \cdot \left( \bU_k^{\Omega} \cdot \bS_k \cdot \bV_k^T \right)^{\dagger} \cdot Y^{\Omega} 
	\nonumber \\ &= \bU_k \cdot \bS_k \cdot \bV_k^T \cdot \bV_k \cdot \bS_k^{-1} \cdot (\bU_k^{\Omega})^{\dagger} \cdot Y^{\Omega}
	\nonumber \\ &=  \bU_k \cdot (\bU_k^{\Omega})^T \cdot Y^{\Omega}.  \label{eq:prop2.2}
\end{align}
%
From \eqref{eq:prop2.1} and \eqref{eq:prop2.2}, we obtain $\hY^{\pcr,k} = \hY^{\hsvt, k}$ for any $k \le N$. 
\end{proof}

\section{Proof of Theorem \ref{thm:training_pcr_generic} }\label{sec:appendix_noisy_regression_via_MCSE}
\subsection{Background} \label{ssec:matrix_norms}
Recall that the $(a,b)$-mixed norm of a matrix $\bB \in \mathbb{R}^{N \times p}$ is defined as 
\begin{align} \label{eq:mixed_norm}
	\| \bB \|_{a,b} = \left( \sum_{j=1}^p \|\bB_{\cdot, j} \|_{a}^b \right)^{1/b} = \left( \sum_{j=1}^p \left( \sum_{i =1}^N \bB_{ij}^a \right)^{b/a} \right)^{1/b}.
\end{align}
We are interested in the $(2, \infty)$-mixed norm, which corresponds to the maximum $\ell_2$ column norm:
\begin{align} \label{eq:column_norm}
	\| \bB \|_{2,\infty} = \max_{j \in [p]} \big\| \bB_{\cdot, j} \big\|_2 = \max_{j \in [p]} \left( \sum_{i=1}^N \bB_{ij}^2 \right)^{1/2}.\end{align}
\begin{lemma} \label{lemma:holder_general}
	Let $\bB$ be a real-valued $n \times p$ matrix and $x$ a real-valued $p$ dimensional vector. Let $q_1, q_2 \in [1, \infty]$ with $1/q_1 + 1/q_2 = 1$. Then,
	\begin{align}
		\norm{\bB x}_2 &\le \norm{x}_{q_1} \, \norm{\bB}_{2, q_2}. 
	\end{align}
\end{lemma}

\begin{proof}
	Using H\"older's Inequality, we have
	\begin{align*}
		\norm{\bB x}_2^2 &= \sum_{i=1}^n \langle \bB_{i, \cdot}, x \rangle^2
		\le\norm{x}_{q_1}^2 \sum_{i=1}^n \norm{\bB_{i, \cdot}}_{q_2}^2
		= \norm{x}_{q_1}^2 \cdot \norm{\bB}_{2, q_2}^2. 
	\end{align*}
\end{proof}

{ \subsection{Proof of Theorem \ref{thm:training_pcr_generic}}}

\begin{proof}
For simplicity of notation, let us define $\bhA = \bZ^{\text{HSVT}, k}$, $\bhAO = \bZ^{\text{HSVT}, k, \Omega}$. 
Due to the equivalence between PCR and performing linear regression using $\bhAO$ via Proposition \ref{prop:equivalence}, for the remainder of the proof we shall focus on linear regression using $\bhAO$. 

As per notation in Section \ref{sec:intro_pcr_ME}, 
let $\beta^{\text{HSVT}, k}$ be the solution of linear regression using $\bhAO$ and predicted response
variables $\hY^{\text{HSVT}, k} = \bZ^{\text{HSVT}, k} \beta^{\text{HSVT}, k}$; for simplicity, we will denote
$\widehat{\beta} = \beta^{\text{HSVT}, k}$ and $\hY = \hY^{\text{HSVT}, k} = \bhA \widehat{\beta}$. Recall, per our model specification in \eqref{eq:regression_model_general} , $\YO = \bAO\beta^* + \phi+ \epsilon$. Now observe
	\begin{align} \label{eq:linear_1}
		\| \bhAO \widehat{\beta} - \YO \|_2^2  &= \| \bhAO \widehat{\beta} - \bAO \beta^* + \phi\|_2^2  \,+\,  \|\epsilon\|_2^2 \,-\, 2 \epsilon^T (\bhAO \widehat{\beta} - \bAO \beta^*) - 2 \epsilon^T \phi. 
	\end{align}
	On the other hand, the optimality of $\widehat{\beta}$ (recall that $\widehat{\beta} \in \arg \min \| \bhAO \widehat{\beta} - \YO \|_2^2$) yields
	\begin{align} \label{eq:linear_2}
		\|\bhAO \widehat{\beta} - \YO \|_2^2  &\le \|\bhAO \beta^* - \YO\|_2^2  \nonumber
		\\ &= \|(\bhAO - \bAO) \beta^* + \phi\|_2^2 \,+\,  \|\epsilon\|_2^2 \,-\, 2 \epsilon^T (\bhAO - \bAO) \beta^* - 2 \epsilon^T \phi.
	\end{align}
	Combining \eqref{eq:linear_1} and \eqref{eq:linear_2} and taking expectations, we have
	\begin{align} \label{eq:linear_3}
		\Ex \| \bhAO \widehat{\beta} - \bAO \beta^* + \phi\|_2^2 &\le \Ex \|(\bhAO - \bAO) \beta^* + \phi\|_2^2  \,+\, 2 \Ex[\epsilon^T \bhAO (\widehat{\beta} - \beta^*)].
	\end{align}
	Let us bound the final term on the right hand side of \eqref{eq:linear_3}. Under our independence assumptions ($\epsilon$ is independent of $\bH$), 
	observe that
	\begin{align}
		\Ex[\epsilon^T \bhAO] \beta^* &= \Ex[\epsilon^T] \Ex[\bhAO] \beta^* = 0. 
	\end{align}
	Recall that $\widehat{\beta} = \big(\bhAO \big)^{\dagger} Y = \big(\bhAO\big)^{\dagger}\bAO \beta^* + \big(\bhAO \big)^{\dagger} \epsilon +  \big(\bhAO \big)^{\dagger}  \phi$. Using the cyclic and 
	linearity properties of the trace operator (coupled with similar independence arguments), we further have
	\begin{align} \label{eq:linear_trace}
		\Ex [\epsilon^T \bhAO \widehat{\beta}] &= \Ex[\epsilon^T \bhAO \big(\bhAO \big)^{\dagger}] \bAO \beta^* + \Ex[\epsilon^T \bhAO \big(\bhAO \big)^{\dagger} \epsilon]  + \Ex[\epsilon^T \big(\bhAO \big)^{\dagger} ] \phi   \nonumber
		\\ &= \Ex[\epsilon]^T \Ex [ \bhAO \big(\bhAO \big)^{\dagger}] \bAO \beta^* +  \Ex \Big[ \text{tr}\Big( \epsilon^T  \bhAO \big(\bhAO \big)^{\dagger} \epsilon \Big) \Big] + \Ex[\epsilon]^T \Ex[(\bhAO)^{\dagger}] \phi \nonumber
		\\ &= \Ex \Big[ \text{tr}\Big( \bhAO \big(\bhAO\big)^{\dagger} \epsilon \epsilon^T \Big) \Big] \nonumber
		= \text{tr}\Big( \Ex [ \bhAO \big(\bhAO\big)^{\dagger} ] \cdot \Ex [ \epsilon \epsilon^T ]  \Big) \nonumber
 		\le \sigma^2 \Ex \Big[ \text{tr} \Big( \bhAO \big(\bhAO \big)^{\dagger} \Big) \Big]	\nonumber
		\\ &= \sigma^2 \Ex [\text{rank}(\bhAO)]~ \leq ~ \sigma^2 k,
	\end{align}
	where the inequality follows from Property \ref{prop:observation_noise_structure} and the fact
	that rank of $\bhAO$ is at most that of $\bhA  = \bZ^{\text{HSVT}, k}$ and which by definition at most $k$.
Consider 
\begin{align}\label{eq:thm1.1}
\| \bhAO \widehat{\beta} - \bAO \beta^* + \phi \|_2^2 & = \|\bhAO \widehat{\beta} - \bAO \beta^* \|_2^2 \, + \, \|\phi \|_2^2  \,+\, 2 \phi^T (\bhAO \widehat{\beta} - \bAO \beta^*).
\end{align}
and 
\begin{align}\label{eq:thm1.2}
\|(\bhAO - \bAO) \beta^* + \phi \|_2^2 & = \|(\bhAO - \bAO) \beta^*\|_2^2 \,+\, \| \phi \|_2^2 \,+\, 2 \phi^T ((\bhAO - \bAO) \beta^*).
\end{align}
From \eqref{eq:linear_trace}, \eqref{eq:thm1.1} and \eqref{eq:thm1.2}, the \eqref{eq:linear_3} becomes 
\begin{align}
		\Ex \|\bhAO \widehat{\beta} - \bAO \beta^* \|_2^2 &\le 2 \sigma^2 k + \Ex \|(\bhAO - \bAO) \beta^* \|_2^2   \nonumber \\
		& \qquad + 2 \Ex |\phi^T (\bhAO \widehat{\beta} - \bAO \beta^*)| \, + \, 2 \Ex |\phi^T ((\bhAO - \bAO) \beta^*)|. \label{eq:linear_4}
\end{align}
By Cauchy-Schwartz, we have
\begin{align}
|\phi^T (\bhAO \widehat{\beta} - \bAO \beta^*)| & \leq \|\phi\|_2 \, \|\bhAO \widehat{\beta} - \bAO \beta^*\|_2, \label{eq:thm1.3a} \\
|\phi^T ((\bhAO - \bAO) \beta^*)| & \leq \|\phi\|_2 \, \|(\bhAO - \bAO) \beta^*\|_2. \label{eq:thm1.3b}
\end{align}
Using \eqref{eq:thm1.3a} and \eqref{eq:thm1.3b} in \eqref{eq:linear_4}, we obtain
\begin{align} 
		\Ex \| \bhAO \widehat{\beta} - \bAO \beta^*\|_2^2 &\le 
		2 \sigma^2 k + \Ex \| (\bhAO - \bAO) \beta^* \|_2^2 \, +\,  2 \| \phi \|_2 \, \Ex \| (\bhAO - \bAO) \beta^* \|_2  \nonumber 
		\\ & \qquad + 2 \| \phi \|_2 \, \Ex \| \bhAO \hbeta - \bAO \beta^* \|_2.
\end{align}
Applying Jensen's Inequality then gives
\begin{align}
	\Ex \| \bhAO \widehat{\beta} - \bAO \beta^*\|_2^2 &\le 
		2 \sigma^2 k + \Ex \| (\bhAO - \bAO) \beta^* \|_2^2 \, +\,  2 \| \phi \|_2 \, \sqrt{ \Ex \| (\bhAO - \bAO) \beta^* \|_2^2}  \nonumber 
		\\ & \qquad + 2 \| \phi \|_2 \, \sqrt{\Ex \| \bhAO \hbeta - \bAO \beta^* \|_2^2}.\label{eq:thm1.5}
\end{align}
Now, let 
\[ x = \Ex \| \bhAO\hbeta - \bAO \beta^* \|_2^2, ~~ y = 2 \sigma^2 k +  \Ex \| (\bhAO - \bAO) \beta^* \|_2^2 \, + \, 2 \| \phi\|_2 \sqrt{\Ex \| (\bhAO - \bAO) \beta^* \|_2^2}. \] 
Then, \eqref{eq:thm1.5} can be viewed as $x \leq y + 2 \|\phi\|_2 \, \sqrt{x}$ with both $x, y \ge 0$. Therefore, either $x \leq 4 \| \phi \|_2 \, \sqrt{x}$ or $x \le 2 y$, i.e., $x \le 2 y + 16 \|\phi\|_2^2$. Replacing the values of $x,y$ as above yields
\begin{align}
	 \Ex \| \bhAO\hbeta - \bAO \beta^* \|_2^2 &\le 4\sigma^2k + 2 \Ex \| (\bhAO - \bAO)\beta^* \|_2^2 \,+\, 4 \| \phi \|_2 \sqrt{ \Ex \| (\bhAO - \bAO)\beta^* \|_2^2} + 16 \|\phi\|_2^2
	 \\ &\le 4\sigma^2k + 2 \Ex \| (\bhAO - \bAO)\beta^* \|_2^2 \,+\, 4 \| \phi \|_2^2 \,+\, \Ex \| (\bhAO - \bAO)\beta^* \|_2^2 \,+\, 16 \|\phi\|_2^2
	 \\ &= 4 \sigma^2k + 3 \Ex \| (\bhAO - \bAO)\beta^* \|_2^2 \,+\, 20 \|\phi\|_2^2, \label{eq:thm1.6}
\end{align} 
where the second inequality uses the fact that for any $a,b \in \Reals$, $2ab \le a^2 + b^2$. 
We now apply Lemma \ref{lemma:holder_general} with $q_1 = 1$ and $q_2 = \infty$ to obtain
\begin{align}\label{eq.thm1.4}
		\| (\bhAO - \bAO) \beta^* \|_2^2 &\le \|\beta^*\|_1^2 \, \| \bAO - \bhAO\|_{2, \infty}^2.
\end{align}
Dividing by $n$ on both sides of \eqref{eq:thm1.6} gives the desired result: 
\[
	\frac{1}{n}  \Ex \| \bhAO\hbeta - \bAO \beta^* \|_2^2 ~\le \frac{4 \sigma^2 k}{n} + \frac{3 \| \beta^* \|_1^2 }{n} \Ex \| (\bhAO - \bAO)\|_{2, \infty}^2  \,+\, \frac{20 \|\phi\|_2^2}{n}. 
\]
\end{proof}

\section{Towards the Proof of Lemma \ref{lemma:mcse_hvst}: Spectral Norm Upper Bound of Random Matrices with Sub-Exponential Rows}

Here we state and derive bound on the spectral norm of random matrix whose rows (or columns) are generated independently per $\psi_\alpha$-distribution for $\alpha \geq 1$. 
This will be crucial in establishing the required ``de-noising'' properties of HSVT

\begin{theorem}\label{thm:spectral_norm_noise_matrix_bound}
Suppose Properties \ref{prop:bounded_covariates}, \ref{prop:covariate_noise_structure} for some $\alpha \ge 1$ hold.
Then for any $\delta_1 > 0$, 
	\begin{align}
		\norm{\bbZ - \rho \bbA} 
		&\leq 
		 \sqrt{N (1+ \sigma^2) (1+ \gamma^2) }
		+ C(\alpha) \sqrt{1+\delta_1}\sqrt{p} (K_{\alpha} + 1)
		\Big( 1 + \big(2 + \delta_1 \big) \log(Np) \Big)^{\frac{1}{\alpha}} \sqrt{ \log(Np) } 
	\end{align}
	with probability at least $ 1 - \frac{2}{N^{1 + \delta_1} p^{\delta_1}}$. Here, $C(\alpha)$ is an absolute constant that depends only on $\alpha$.
\end{theorem}

The upper bound stated in Theorem \ref{thm:spectral_norm_noise_matrix_bound} is not the sharpest possible. 
But they are sufficient for our purposes. Sharp bounds for $\alpha =1 $ and $\alpha \geq 2$ can be found in \cite{adamczak2011sharp} 
and \cite{vershynin2010introduction} for example. 

\subsection{Helper Lemmas for the Proof of Theorem \ref{thm:spectral_norm_noise_matrix_bound}}

We begin by presenting Proposition \ref{prop:spectral_upper_bound}, which holds for general random matrices $\bW \in \Reals^{N \times p}$. We note that this result depends on two quantities: (1) $\norm{ \Ex \bW^T \bW}$ and (2) $\norm{\bW_{i, \cdot}}_{\psi_\alpha}$ for all $i \in [N]$. We then instantiate $\bW := \bbZ - \rho \bbA$ and present Lemmas \ref{lemma:masked_noise_operator_norm} and \ref{lemma:masked_noise_row_norm}, which bound (1) and (2), respectively, for our choice of $\bW$. 





\begin{proposition}\label{prop:spectral_upper_bound}
	Let $\bW \in \mathbb{R}^{N \times p}$ be a random matrix whose rows $\bW_{i, \cdot}$ ($i \in [N]$) are independent 
	$\psi_{\alpha}$-random vectors for some $\alpha \geq 1$. Then for any $\delta_1 > 0$,
	\begin{align*}
		\norm{\bW} \leq \norm{ \Ex \bW^T \bW }^{1/2}
		+ C(\alpha) \sqrt{(1+\delta_1) p } \max_{i \in [N]}\norm{ \bW_{i, \cdot} }_{\psi_\alpha}
		\Big( 1 + \big(2 + \delta_1 \big) \log(Np) \Big)^{\frac{1}{\alpha}} \sqrt{ \log(Np) } 
	\end{align*}
	with probability at least $ 1 - \frac{2}{N^{1 + \delta_1}p^{\delta_1}}$. Here, $C(\alpha)>0$ is an absolute constant that depends only on $\alpha$.
\end{proposition}

\begin{proof}
We prove the proposition in four steps.
\paragraph{Step 1: picking the threshold value.} 
Let $e_1, \ldots, e_p \in \Reals^p$ denote the canonical basis\footnote{Column vector representation} of $\Reals^p$. Observe that 
$\norm{ \bW_{i, \cdot} }_2^2 =  \bW_{i, \cdot} \bW_{i, \cdot}^T = \sum_{j=1}^p \left( \bW_{i, \cdot} e_j \right)^2$\footnote{Recall that $\bW_{i, \cdot}$ is a row vector and hence $\bW_{i, \cdot} \bW_{i, \cdot}^T$ is a scalar.} . Therefore, for any $t \ge 0$,
\begin{align*}
	\mathbb{P} \Big\{ \norm{\bW_{i, \cdot}}_2^2 > t \Big\} 
	&= \mathbb{P}\bigg \{ \sum_{j=1}^p \left( \bW_{i, \cdot} e_j \right)^2 > t\bigg\} \\
	&\stackrel{(a)}\le \sum_{j=1}^p \mathbb{P}\bigg \{ \left( \bW_{i, \cdot} e_j \right)^2 > \frac{t}{p}\bigg\} \\
	&\le \sum_{j=1}^p \mathbb{P}\bigg\{ \left| \bW_{i, \cdot} e_j \right| \ > \sqrt{\frac{t}{p}}\bigg\} \\
	&\stackrel{(b)}\le 2p\exp( - C({\alpha}) \left( \frac{t}{p  \norm{ \bW_{i, \cdot} }_{\psi_\alpha}^2 } \right)^{\frac{\alpha}{2}} ),
\end{align*}
where (a) uses the union bound and (b) follows from the definition of $\psi_{\alpha}$-random vector ($C({\alpha})$ is an absolute constant 
which depends only on $\alpha \geq 1$). 
Choosing $t = C^{\frac{2}{\alpha}}  C({\alpha})^{-\frac{2}{\alpha}}  p \norm{ \bW_{i, \cdot} }_{\psi_\alpha}^2 \big(\log(2p) \big)^{\frac{2}{\alpha}}$ 
for some $C > 1$ gives 
\[
\mathbb{P}\Big \{ \norm{ \bW_{i, \cdot} }_{2}^2 
	> C^{\frac{2}{\alpha}}  C({\alpha})^{-\frac{2}{\alpha}}  p \norm{ \bW_{i, \cdot} }_{\psi_\alpha}^2 \big(\log(2p) \big)^{\frac{2}{\alpha}} \Big \} 
	\le \Big( \frac{1}{2p} \Big)^{C - 1}. 
\]
Applying the union bound, we obtain
\[
\mathbb{P}\bigg\{ \max_{i \in [N]} \norm{ \bW_{i, \cdot} }_{2}^2 
	> C^{\frac{2}{\alpha}}  C({\alpha})^{-\frac{2}{\alpha}}  p \max_{i \in [N]}  \norm{ \bW_{i, \cdot} }_{\psi_\alpha}^2 \big(\log(2p) \big)^{\frac{2}{\alpha}} \bigg\}
	\le N \Big( \frac{1}{2p} \Big)^{C - 1}.
\]
For $\delta_1 > 0$, we define $C(\delta_1) \triangleq 1 + \big(2 + \delta_1 \big) \log_{2p}(Np)$ and let $C = C(\delta_1)$. Also, we define 
\[
t_0(\delta_1) \triangleq  C(\delta_1)^{\frac{2}{\alpha}}  C({\alpha})^{-\frac{2}{\alpha}}  p \max_{i \in [N]}  \norm{ \bW_{i, \cdot} }_{\psi_\alpha}^2 \big(\log(2p) \big)^{\frac{2}{\alpha}}.
\]
We have
\begin{equation}\label{eqn:step1}
	\mathbb{P}\Big \{ \max_{i \in [N]} \norm{ \bW_{i, \cdot} }_{2}^2 > t_0(\delta_1)\Big \} 
	\le N \Big( \frac{1}{2p} \Big)^{\big(2 + \delta_1 \big) \log_{2p}(Np)} = \frac{1}{N^{1 + \delta_1} p^{2 + \delta_1}}.
\end{equation}
\medskip
\paragraph{Step 2: decomposing $\bW$ by truncation.} 
Next, given $\delta_1 > 0$, we decompose the random matrix $\bW$ as follows:
\[
	\bW = \bW^{\circ}(\delta_1) + \bW^{\times}	(\delta_1)
\]
where for each $i \in [N]$,
\[ 
\bW^{\circ}(\delta_1)_{i, \cdot} = \bW_{i, \cdot} \Ind{ \norm{\bW_{i, \cdot} }_2^2 \leq t_0(\delta_1) }  \quad\text{and}\quad 
 	\bW^{\times}(\delta_1)_{i, \cdot} = \bW_{i, \cdot} \Ind{ \norm{\bW_{i, \cdot} }_2^2 > t_0(\delta_1) }.
\]
Then it follows that
\begin{equation}\label{eqn:step2}
\norm{\bW} 
	\leq \norm{\bW^{\circ}(\delta_1) } + \norm{ \bW^{\times}(\delta_1) }		
	\leq \norm{\bW^{\circ}(\delta_1) } + \norm{ \bW^{\times}(\delta_1) }_F.		
\end{equation}
\medskip
\paragraph{Step 3: bounding $\norm{\bW^{\circ}(\delta_1) }$ and $\norm{ \bW^{\times}(\delta_1) }_F$.}
We define two events for conditioning:
\begin{align}
	E_1(\delta_1) &:= \left\{ \norm{\bW^{\circ} (\delta_1) } \leq  \norm{ \Ex \bW^T \bW }^{1/2} 
			+ \sqrt{ \frac{1 + \delta_1}{c} t_0(\delta_1) \log(Np) } \right\},		\label{eqn:step3.E1}\\
	E_2(\delta_1) &:= \left\{ \norm{ \bW^{\times} (\delta_1) }_F = 0 \right\}.				\label{eqn:step3.E2}
\end{align}

First, given $\delta_1 > 0$, we let $\Sigma^{\circ}(\delta_1) = \Ex \bW^{\circ}(\delta_1)^T \bW^{\circ}(\delta_1)$. 
By definition of $\bW^{\circ}(\delta_1)$, we have $\norm{\bW_{i, \cdot}}_2 \leq \sqrt{t_0(\delta_1)}$ for all $i \in [N]$. Then it follows 
that for every $s \geq 0$, 
\[	\norm{\bW^{\circ} (\delta_1) } \leq  \norm{ \Sigma^{\circ}(\delta_1) }^{1/2} + s \sqrt{t_0(\delta_1)}	\]
with probability at least $1 - p \exp(-c s^2 )$ (see Theorem 5.44 of \cite{vershynin2010introduction} and Eqs. (5.32) and (5.33) in reference, and replacing the common second moment $\bSigma = \Ex \bW_{i, \cdot}^T \bW_{i, \cdot}$ with the average second moment for all rows, $\bSigma = \frac{1}{N} \sum_{i=1}^N \Ex \bW_{i, \cdot}^T \bW_{i, \cdot}$, i.e., redefining $\bSigma$). 
Note that $\norm{ \Sigma^{\circ}(\delta_1) } =  \norm{\Ex \bW^{\circ}(\delta_1)^T \bW^{\circ}(\delta_1)} \leq \norm{\Ex \bW^T \bW}$. 
Now we define $\tilde{E}_1(s)$ parameterized by $s > 0$ as
\begin{equation}
	\tilde{E}_1(s; \delta_1) := \left\{ \norm{\bW^{\circ} (\delta_1) } >  \norm{ \Ex \bW^T \bW }^{1/2}  + s \sqrt{t_0(\delta_1)} \right\}.
\end{equation}
If we pick $s = \left( \frac{1 + \delta_1}{c} \log(Np) \right)^{1/2}$, then $E_1(\delta_1) = \tilde{E}_1(s; \delta_1)$ and
\[
	\Prob{ E_1(\delta_1)^c }	\leq p \exp(-c s^2 ) = p \exp \left( -(1 + \delta_1) \log(Np) \right) = \frac{1}{N^{1+\delta_1} p^{\delta_1}}.
\]
Next, we observe that $\norm{ \bW^{\times} (\delta_1) }_F = 0$ if and only if $\bW^{\times} (\delta_1) = 0$.
If $\bW^{\times} (\delta_1) \neq 0$, then $\max_{i \in [n]} \norm{ \bW_{i, \cdot} }_{2}^2 > t_0(\delta_1)$. Therefore, 
\[
	\Prob{E_2^c} \leq \frac{1}{N^{1+\delta_1} p^{2 + \delta_1}}
\]
by the analysis in Step 1; see \eqref{eqn:step1}.
\medskip
\paragraph{Step 4: concluding the proof.} 
For any given $\delta_1 > 0$,
\[
	\Prob{ \norm{\bW} > \norm{ \Ex \bW^T \bW }^{1/2} 
			+ \sqrt{ \frac{1 + \delta_1}{c} t_0(\delta_1) \log(Np) }  ~ \bigg|~ E_1(\delta_1) \cap E_2(\delta_1)} = 0.
\]
by \eqref{eqn:step2}, \eqref{eqn:step3.E1}, and \eqref{eqn:step3.E2}.
By the law of total probability and the union bound, 
\begin{align*}
	&\Prob{ \norm{\bW} > \norm{ \Ex \bW^T \bW }^{1/2} 
			+ \sqrt{ \frac{1 + \delta_1}{c} t_0(\delta_1) \log(Np) } }\\
	&\qquad\leq
		\Prob{ \norm{\bW} > \norm{ \Ex \bW^T \bW }^{1/2}  
			+ \sqrt{ \frac{1 + \delta_1}{c} t_0(\delta_1) \log(Np) }  ~ \bigg|~ E_1(\delta_1) \cap E_2(\delta_1)}\\
		&\qquad\quad+ \Prob{E_1(\delta)^c} + \Prob{E_2(\delta)^c}\\
	&\qquad\leq \frac{1}{N^{1+\delta_1} p^{\delta_1}} + \frac{1}{N^{1+\delta_1} p^{2 + \delta_1} }\\
	&\qquad\leq \frac{2}{N^{1+\delta_1} p^{\delta_1} }.
\end{align*}
This completes the proof. 
\end{proof}

\subsection{Lemmas  \ref{lemma:masked_noise_operator_norm} and \ref{lemma:masked_noise_row_norm}}

\subsubsection{Lemma  \ref{lemma:masked_noise_operator_norm} }
\begin{lemma} \label{lemma:masked_noise_operator_norm}
\begin{align}
	\norm{\Ex (\bbZ - \rho \bbA)^T(\bbZ - \rho \bbA)} &\le 
	\rho(1-\rho)
	\left( \max_{j \in [p] } \norm{ \bbA_{\cdot, j} }_2^2 
	    + \| \emph{diag}(\Ex[\bH^T \bH]) \| 
	\right) 
	+ \rho^2 \norm{ \Ex \bbH^T \bbH }.
\end{align}
\end{lemma} 

\begin{proof}
 We follow the proof of Lemma A.2 of \cite{shah2018learning} and state it here for completeness.
Throughout, for any matrix $\bQ \in \Rb^{N \times p}$, let $Q_\ell \in \Rb^n$ denote the $\ell$-th row of $\bQ$. 

\noindent 
To begin, observe that
\begin{align*}
    \Ex[(\bZ - \rho \bA)^T (\bZ - \rho \bA)] &= \sum_{\ell=1}^N \Ex[ (Z_\ell - \rho A_\ell) \otimes (Z_\ell - \rho A_\ell) ]. 
\end{align*}
Let $\bX = \bA + \bH$. Importantly, we highlight the following relations: for any $(\ell, i) \in [N]\times [p]$, 
\begin{align*}
    \Ex[Z_{\ell i}] &= \rho A_{\ell i}
    \\ \Ex[Z_{\ell i}^2] &= \rho \Ex[X_{\ell i}^2]. 
\end{align*}
Now, let us fix a row $\ell \in [N]$ and denote
\[ \bW^{(\ell)} = (Z_\ell - \rho A_\ell) \otimes (Z_\ell - \rho A_\ell). \]  
Using the linearity of expectations, the expected value of the $(i,j)$-th entry of $\bW^{(\ell)}$ can be written as 
\begin{align*}
    \Ex[W_{ij}^{(\ell)}] &= \Ex[Z_{\ell i} Z_{\ell j}] - \rho \Ex[ Z_{\ell i} A_{\ell j}] - \rho \Ex[ Z_{\ell j} A_{\ell i}] + \rho^2 \Ex[ A_{\ell i} A_{\ell j}].
\end{align*}
Suppose $i = j$, then 
\begin{align} \label{eq:ij.1} 
    \Ex[W_{ii}^{(\ell)}] &= \rho \Ex[X_{\ell i}^2] - \rho^2 A_{\ell i}^2
    = \rho (1-\rho) \Ex[X_{\ell i}^2] + \rho^2 \Ex[(X_{\ell i} - A_{\ell i})^2]. 
\end{align}
On the other hand, if $i \neq j$,   
\begin{align}\label{eq:ij.2}
    \Ex[W_{ij}^{(\ell)}] &= \rho^2 \Ex[ (X_{\ell i} - A_{\ell i}) (X_{\ell j} - A_{\ell j})].
\end{align}
Therefore, we can express $\bW^{(\ell)}$ as the sum of two matrices where the diagonal components are generated from \eqref{eq:ij.1} and the off-diagonal components are generated from \eqref{eq:ij.2}. That is, 
\begin{align*}
    \Ex[\bW^{(\ell)}] &= 
    \Ex\Big( \rho(1-\rho) \text{diag}(X_\ell \otimes X_\ell)
    + \rho^2 \text{diag}( H_\ell \otimes H_\ell) \Big) 
    + \Ex\Big( \rho^2 (H_\ell \otimes H_\ell)
    - \rho^2 \text{diag}( H_\ell \otimes H_\ell) \Big)
    \\ 
    &= \rho(1-\rho) \Ex[\text{diag}(X_\ell \otimes X_\ell)] + \rho^2 \Ex[H_\ell \otimes H_\ell]. 
\end{align*}
Taking the sum over all rows $\ell \in [N]$ yields
\begin{align} \label{eq:matrix.1} 
    \Ex[ (\bZ - \rho \bA)^T (\bZ - \rho \bA)] &= \rho (1-\rho) \text{diag}(\Ex[\bX^T \bX]) + \rho^2 \Ex[ \bH^T \bH]. 
\end{align}
To complete the proof, we apply triangle inequality to \eqref{eq:matrix.1} to obtain 
\begin{align*}
    \norm{ \Ex[(\bZ - \rho \bA)^T (\bZ - \rho \bA)]} &\le 
    \rho (1- \rho) \norm{ \text{diag}(\Ex[\bX^T \bX])} + \rho^2 \norm{ \Ex[\bH^T \bH]}. 
\end{align*}
Since $\bH$ is zero mean, we have
\begin{align}
    \norm{ \text{diag}(\Ex[\bX^T \bX])} 
    &=  \norm{ \text{diag}(\bA^T \bA) + \text{diag}(\Ex[\bH^T \bH])}
    \\ &\le \norm{ \text{diag}(\bA^T \bA)} + \norm{ \text{diag}(\Ex[\bH^T \bH])}.
\end{align}
Collecting terms completes the proof.
\end{proof}

%
%
%
%


\subsection{Lemma \ref{lemma:masked_noise_row_norm}}
\begin{lemma}\label{lem:technical_Ka}
	Suppose that $X \in \mathbb{R}^n$ and $P \in \{0, 1\}^n$ are random vectors. Then for any $\alpha \geq 1$,
	\[	\norm{ X \circ P }_{\psi_{\alpha}} \leq \norm{ X }_{\psi_{\alpha}}.	\]
\end{lemma}
\begin{proof}
	Given a deterministic binary vector $P_0 \in \{0, 1\}^n$, let $I_{P_0} = \{ i \in [n]: Q_i = 1 \}$. 
	Observe that 
	\[
	X \circ P_0 = \sum_{i \in I_{P_0}} e_i e_i^T X.
	\] 
	Here, $\circ$ denotes the Hadamard product (entrywise product) of two matrices.
	By definition of the $\psi_{\alpha}$-norm, 
	\begin{align*}
		\norm{X}_{\psi_{\alpha}} 
			&= \sup_{u \in \mathbb{S}^{n-1}} \norm{ u^T X } _{\psi_{\alpha}}
			= \sup_{u \in \mathbb{S}^{n-1}}\inf\left\{ t > 0: \Ex_X \Big[ \exp\big( | u^T X|^{\alpha} / t^{\alpha} \big) \Big] \leq 2 \right\}.
	\end{align*}
	Let $u_0 \in \mathbb{S}^{n-1}$ denote the maximum-achieving unit vector (such $u_0$ exists because $\inf\{\cdots\}$ 
	is continuous with respect to $u$ and $\mathbb{S}^{n-1}$ is compact). Then,
	\begin{align*}
		\norm{ X \circ P }_{\psi_{\alpha}}
			&= \sup_{u \in \mathbb{S}^{n-1}} \norm{ u^T X \circ P } _{\psi_{\alpha}}\\
			&= \sup_{u \in \mathbb{S}^{n-1}} \inf\left\{ t > 0: 
				\Ex_{X,P} \Big[ \exp \left(  \big| u^T X \circ P \big|^{\alpha} / t^{\alpha} \right) \Big] \leq 2 \right\}\\
			&= \sup_{u \in \mathbb{S}^{n-1}} \inf\left\{ t > 0: 
				\Ex_{P} \Big[ \Ex_{X} \Big[  \exp \left(  \big| u^T X \circ P \big|^{\alpha} / t^{\alpha} \right)
				 ~\Big|~ P \Big] \Big] \leq 2 \right\}\\
			&= \sup_{u \in \mathbb{S}^{n-1}} \inf\left\{ t > 0: 
				\Ex_{P} \bigg[ \Ex_{X} \bigg[  \exp \bigg(  \Big| u^T  \sum_{i \in I_P} e_i e_i^TX \Big|^{\alpha}
				 / t^{\alpha} \bigg) ~\bigg|~ P \bigg] \bigg] \leq 2 \right\}\\
			&= \sup_{u \in \mathbb{S}^{n-1}} \inf\left\{ t > 0: 
				\Ex_{P} \bigg[ \Ex_{X} \bigg[  \exp \bigg(  \bigg|  \Big(  \sum_{i \in I_P} e_i e_i^T u \Big)^T  X \bigg|^{\alpha}
				 / t^{\alpha} \bigg) ~\bigg|~ P \bigg] \bigg] \leq 2 \right\}.
	\end{align*}
	For any $u \in \mathbb{S}^{n-1}$ and $P_0 \in \{0, 1 \}^n$, observe that
	\begin{align*}
		\Ex_{X} \bigg[  \exp \bigg(  \bigg|  \Big(  \sum_{i \in I_P} e_i e_i^T u \Big)^T  X \bigg|^{\alpha}
				 / t^{\alpha} \bigg) ~\bigg|~ P = P_0 \bigg] 
			\leq \Ex_X \Big[ \exp\Big( | u_0^T X |^{\alpha}/t^{\alpha} \Big)  \Big].
	\end{align*}
	Therefore, taking supremum over $u \in \mathbb{S}^{n-1}$, we obtain
	\begin{align*}
		\norm{ X \circ P }_{\psi_{\alpha}}&\leq \norm{X}_{\psi_{\alpha}}.
	\end{align*}
\end{proof}

\begin{lemma}\label{lem:MGF_upper}
	Let $X$ be a mean-zero, $\psi_{\alpha}$-random variable for some $\alpha \geq 1$. 
	Then for $| \lambda | \leq \frac{1}{C \norm{X}_{\psi_{\alpha}}}$,
	\[	\Ex \exp\left( \lambda X \right) \leq \exp\left( C \lambda^2  \norm{X}_{\psi_{\alpha}}^2 \right).	\]
\end{lemma}
\begin{proof}
	See \cite{vershynin2018high}, Section 2.7.
\end{proof}

\begin{lemma}\label{lem:ind_sum}
	Let $X_1, \ldots, X_n$ be independent random variables with mean zero.
	For $\alpha \geq 1$, 
	\[	\norm{\sum_{i=1}^n X_i}_{\psi_{\alpha}} \leq C \left( \sum_{i=1}^n \norm{X_i}_{\psi_{\alpha}}^2 \right)^{1/2}.	\]
\end{lemma}

\begin{proof}
	Immediate by Lemma \ref{lem:MGF_upper}.
\end{proof}

\begin{lemma} \label{lemma:masked_noise_row_norm}
Assume Properties \ref{prop:bounded_covariates} and \ref{prop:covariate_noise_structure} hold. 
Then for any $\alpha \geq 1$ with which Property \ref{prop:covariate_noise_structure} holds, we have 
\begin{align}
	\norm{ \bbZ_{i,\cdot} -  \rho \bbA_{i, \cdot} }_{\psi_{\alpha}} \le C (K_{\alpha}	+1 )\qquad\text{for all }~i \in [N],
\end{align}
where $C > 0$ is an absolute constant. 
\end{lemma}

\begin{proof}
	Let $\bbP \in \{ 0, 1\}^{N \times p}$ denote a random matrix whose entries are i.i.d. random variables that take value $1$ 
	with probability $\rho$ and $0$ otherwise. Note that $\bbZ_{i, \cdot}$ can be written as  $\bbX_{i, \cdot} \circ \bbP_{i, \cdot}$ where $\star$ is identified with $0$. By triangle inequality,
	\begin{align*}
	\norm{\bbZ_{i, \cdot} - \rho\bbA_{i, \cdot}}_{\psi_{\alpha}}
	&= \norm{\bbX_{i, \cdot} \circ \bbP_{i, \cdot} - \rho\bbA_{i, \cdot}}_{\psi_{\alpha}} \\ 
	&= \norm{ (\bbX_{i, \cdot} \circ \bbP_{i, \cdot})  - (\bbA_{i, \cdot} \circ \bbP_{i, \cdot}) 
		- \rho\bbA_{i, \cdot} + (\bbA_{i, \cdot} \circ \bbP_{i, \cdot}) }_{\psi_{\alpha}} \\
	&\le \norm{(\bbX_{i, \cdot} - \bbA_{i, \cdot}) \circ  \bbP_{i, \cdot}}_{\psi_{\alpha}} 
		+ \norm{ (\bbA_{i, \cdot} \circ \bbP_{i, \cdot}) - \rho\bbA_{i, \cdot} }_{\psi_{\alpha}}.
	\end{align*}
	
	\noindent By definition of $\bX$, Property \ref{prop:covariate_noise_structure}, and Lemma \ref{lem:technical_Ka}, we have that
	\begin{align*}
		 \norm{(\bbX_{i, \cdot} - \bbA_{i, \cdot}) \circ  \bbP_{i, \cdot}}_{\psi_{\alpha}}
		 	&\leq \norm{\bbX_{i, \cdot} - \bbA_{i, \cdot}}_{\psi_{\alpha}} 
			= \norm{\eta_{i, \cdot}}_{\psi_{\alpha}} 
			\leq C K_{\alpha}.
	\end{align*}
	Moreover, Property \ref{prop:bounded_covariates} and the i.i.d. property of $\bbP_{ij}$ for different $j$ gives
	\begin{align*}
		\Big \| (\bbA_{i, \cdot} \circ \bbP_{i, \cdot}) - \rho\bbA_{i, \cdot} \Big \|_{\psi_{\alpha}}
			&= \sup_{u \in \mathbb{S}^{p-1}} \bigg \| \sum_{j=1}^p u_j \bbA_{i, j} \big( \bbP_{i, j} - \rho \big) \bigg \|_{\psi_{\alpha}}\\
			&\leq \sup_{u\in \mathbb{S}^{p-1}} \bigg( \sum_{j=1}^p u_j^2 \norm{ \bbA_{i,j} (\bbP_{i,j} - \rho)}_{\psi_{\alpha}}^2 \bigg)^{1/2}\\
			&\leq \bigg(\sup_{u\in \mathbb{S}^{p-1}} \sum_{j} u_j^2 \max_{j \in [p]} | \bbA_{i,j} |^2\bigg)^{1/2} \norm{\bbP_{1,1} - \rho}_{\psi_\alpha}\\
			&\leq  \norm{\bbP_{1,1} - \rho}_{\psi_\alpha}.
	\end{align*}
	The first inequality follows from Lemma \ref{lem:ind_sum}, the second inequality is immediate, and the last inequality follows from Property \ref{prop:bounded_covariates}. Lastly, $\norm{\bbP_{1,1} - \rho}_{\psi_\alpha} \leq C$ because $\bbP_{1,1} - \rho$ is a bounded random variable in $[-\rho, 1- \rho]$.
\end{proof}

\subsection{Proof of Theorem \ref{thm:spectral_norm_noise_matrix_bound}}
\begin{proof}[Proof of Theorem \ref{thm:spectral_norm_noise_matrix_bound}]
The proof follows by plugging the results of Lemmas \ref{lemma:masked_noise_operator_norm} and \ref{lemma:masked_noise_row_norm} into Proposition \ref{prop:spectral_upper_bound} for $\bW := \bbZ - \rho \bbA$ and applying Properties \ref{prop:bounded_covariates} and \ref{prop:covariate_noise_structure}. 
\end{proof}

\section{Proof of Lemma \ref{lemma:mcse_hvst}}\label{sec:appendix_mcse_hsvt}

To bound the error in estimation of HSVT, $\bbZ^{HSVT, k}$ with thresholding at $k$th singular value, 
and underlying covariate matrix $\bbA$ with respect to $\| \cdot \|_{2,\infty}$ matrix norm, we shall start by presenting 
Lemma \ref{lemma:column_error} which bounds $\| \bbZ^{HSVT, k} - \bbA\|_{2, \infty}$ as a function of few abstract quantities. 
Next, we bound these quantities with high probability in our setting through help of sequence of results including the 
spectral norm bound stated in Theorem \ref{thm:spectral_norm_noise_matrix_bound}. We conclude with the proof of Lemma \ref{lemma:mcse_hvst}.  

\paragraph{Notation.} Consider a matrix $\bB \in \Reals^{N \times p}$ such that 
$\bB = \sum_{i=1}^{N \wedge p} \sigma_i(\bB) x_i y_i^T$. With a specific choice of $\lambda \geq 0$, 
we can define a function $\varphi^{\bB}_{\lambda}: \mathbb{R}^{N} \to \mathbb{R}^{N}$ as follows: for any vector $w \in \mathbb{R}^N$,
\begin{align} \label{eq:prox_vector}
	\varphi^{\bB}_{\lambda}(w) &= \sum_{i = 1}^{N \wedge p} \mathbb{1}(\sigma_i (\bB) \ge \lambda) x_i x_i^T w.
\end{align}
Note that $\varphi^{\bB}_{\lambda}$ is a linear operator and it depends on the tuple $(\bB, \lambda)$; more precisely, 
the singular values and the left singular vectors of $\bB$, as well as the threshold $\lambda$. If $\lambda = 0$, 
then we will adopt the shorthand notation: $\varphi^{\bB} = \varphi_{0}^{\bB}$. 

\subsection{Lemma \ref{lemma:column_error}}\label{sec:key_lemma}

\subsubsection{Some Observations on HSVT Operator}\label{sec:more_svt}

Observe that the function $\varphi^{\bB}_{\lambda}: \mathbb{R}^{N} \to \mathbb{R}^{N}$ defined in \eqref{eq:prox_vector} 
is actually the operator acting on the column spaces, which is induced by HSVT.
\begin{lemma} \label{lemma:column_representation}
Let $\bB \in \mathbb{R}^{N \times p}$ and $\lambda \geq 0$ be given. Then for any $j \in [p]$,
\begin{align}
	\varphi^{\bB}_{\lambda} \big( \bB_{\cdot,j} \big) = \emph{HSVT}_{\lambda}\big(\bB \big)_{\cdot,j}.
\end{align}
\end{lemma}

\begin{proof}  
By \eqref{eq:prox_vector} and the orthonormality of the left singular vectors, 
\begin{align*}
	\varphi^{\bB}_{\lambda} \big( \bB_{\cdot,j} \big) 
		&= \sum_{i=1}^{N \wedge p} \mathbb{1}(\sigma_i (\bB) \ge \lambda) x_i x_i^T \bB_{\cdot, j} 
		= \sum_{i = 1}^{N \wedge p} \mathbb{1}(\sigma_i (\bB) \ge \lambda) x_i x_i^T 
			\Big( \sum_{i'=1}^{N \wedge p} \sigma_{i'} (\bB) x_{i'} y_{i'} \Big)_{\cdot, j} \\
		&= \sum_{i, i' = 1}^{N \wedge p} \sigma_{i'}(\bB)  \mathbb{1}(\sigma_i (\bB) \ge \lambda) x_i x_i^T x_{i'} (y_{i'})_{j} 
		= \sum_{i, i' = 1}^{N \wedge p} \sigma_{i'}(\bB)  \mathbb{1}(\sigma_i (\bB) \ge \lambda) x_i \delta_{i i'} (y_{i'})_{j} \\
		&= \sum_{i = 1}^{N \wedge p} \mathbb{1} (\sigma_i (\bB) \ge \lambda^*) \sigma_i x_i (y_i)_j  \\
		&= \text{HSVT}_{\lambda}(\bB)_{\cdot, j}.
\end{align*}
This completes the proof. 
\end{proof}


\begin{remark} Suppose we have missing data. Then the estimator $\bhA$ has the following representation:
\[
\bbhA = \frac{1}{\hrho}  \emph{HSVT}_{\lambda^*}(\bbZ) = \frac{1}{\hrho}  \sum_{i=1}^{N \wedge p} s_i \mathbb{1}(s_i \ge \lambda^*) \cdot u_i v_i^T.
\]
By Lemma \ref{lemma:column_representation}, we note that
\begin{align}\label{eq:repsentation_A_hat}
\bbhA_{\cdot, j} = \frac{1}{\hrho} \varphi^{\bZ}_{\lambda^*}(\bbZ_{\cdot, j}). 
\end{align}
\end{remark}

Lastly, we remark that the column operator induced by HSVT is a contraction.
\begin{lemma}\label{lemma:HSVT_contraction}
Let $\bB \in \mathbb{R}^{N \times p}$ and $\lambda \geq 0$ be given. Then for any $j \in [p]$, 
\[
\norm{\emph{HSVT}_{\lambda}\big(\bB  \big)_{\cdot,j}}_2 \le \norm{\bB_{\cdot,j} }_2. 
\]
\end{lemma}

\begin{proof}
By \eqref{eq:prox_vector} and Lemma \ref{lemma:column_representation}, we have
\begin{align*}
\norm{\text{HSVT}_{\lambda}\big( \bB \big)_{\cdot,j}  }_2^2
	&= \norm{\varphi^{\bB}_{\lambda} \big( \bB_{\cdot,j} \big) }_2^2 
	= \norm{\sum_{i=1}^{N \wedge p} \mathbb{1}(\sigma_i (\bB) \ge \lambda) \cdot x_i x_i^T \cdot \bB_{\cdot, j} }_2^2 \\
	&\stackrel{(a)}= \sum_{i=1}^{N \wedge p} \norm{\mathbb{1}(\sigma_i (\bB) \ge \lambda) \cdot x_i x_i^T \cdot \bB_{\cdot, j} }_2^2 
	\le \sum_{i=1}^{N \wedge p} \norm{ x_i x_i^T \cdot \bB_{\cdot, j} }_2^2 \\
	&\stackrel{(b)}=  \norm{ \sum_{i=1}^{N \wedge p} x_i x_i^T \cdot \bB_{\cdot, j} }_2^2 
	= \norm{\bB_{\cdot,j} }_2^2.
\end{align*}
Note that (a) and (b) use the orthonormality of the left singular vectors. 

\end{proof}

\begin{lemma} \label{lemma:column_error}
Suppose that (1) $\norm{ \bbZ - \rho \bbA} \leq \Delta $ for some $\Delta \geq 0$ and (2) $ \frac{1}{\varepsilon} \rho \le \widehat{\rho} \le \varepsilon \rho$ for some $\varepsilon \ge 1$.

Let $\bhA = \bbZ^{\text{HSVT}, k}$, $\bbA^k = \text{HSVT}_{\tau_k}(\bbA)$ and $\bbE = \bbA-\bbA^k$.  Then for any $j \in [p]$, 
\begin{align}
	\Big\| \bbhA_{\cdot, j} - \bbA_{\cdot, j} \Big\|_2^2 
		&\leq		\frac{4\varepsilon^2}{\rho^2} \frac{\Delta^2}{ \rho^2( \tau_k - \tau_{k+1})^2}	
				\big\| \bbZ_{\cdot, j} - \rho \bbA_{\cdot, j} \big\|_2^2	\nonumber\\
		&\quad	+ \frac{4\varepsilon^2}{\rho^2}  \Big \| \varphi^{\bbA^k}(\bbZ_{\cdot, j} - \rho \bbA_{\cdot, j})  \Big \|_2^2
				+ 2 (\varepsilon-1)^2 \| \bbA_{\cdot, j} \|_2^2. \nonumber \\
		&\quad	+ \frac{2 \Delta^2}{ \rho^2( \tau_k - \tau_{k+1})^2} \norm{ \bbA^k_{\cdot, j} }_2^2 
				+ 2\, \norm{ \bbE_{\cdot, j}}_2^2. \label{eq:main_MCSE_inequality}
\end{align}
\end{lemma}

\begin{proof}
First, we recall two conditions assumed in the Lemma that will be used in the proof: (1) $\norm{ \bbZ - \rho \bbA} \leq \Delta $ for some $\Delta \geq 0$, (2) $ \frac{1}{\varepsilon} \rho \le \widehat{\rho} \le \varepsilon \rho$ for some $\varepsilon \ge 1$. 

We will use notation $\lambda^* = s_k$, the $k$th singular value of $\bbZ$ for simplicity. We prove our Lemma in three steps. 
\paragraph{Step 1.}
Fix a column index $j \in [p]$. Observe that
\begin{equation*}
	\bbhA_{\cdot, j} - \bbA_{\cdot, j}
		= \Big( \bbhA_{\cdot, j} - \varphi_{\lambda^*}^{\bbZ} \big( \bbA_{\cdot, j} \big) \Big) + \Big( \varphi_{\lambda^*}^{\bbZ} \big( \bbA_{\cdot, j} \big) - \bbA_{\cdot, j} \Big).
\end{equation*}
By choice,  $\text{rank}(\bbhA) = k$. By definition (see \eqref{eq:prox_vector}), we have that $\varphi_{\lambda^*}^{\bbZ}: \Reals^N \to \Reals^N$ is 
the projection operator onto the span of the top $k$ left singular vectors of $\bbZ$, namely, $\text{span}\big\{ u_1, \ldots, u_k \big\}$. Therefore, 
\[
\varphi_{\lambda^*}^{\bbZ} (\bbA_{\cdot, j}) - \bbA_{\cdot, j} \in \text{span}\{u_1, \ldots, u_k \}^{\perp}
\]
and by \eqref{eq:repsentation_A_hat} (using Lemma \ref{lemma:column_representation}),
\[
\bbhA_{\cdot, j} -  \varphi_{\lambda^*}^{\bbZ} (\bbA_{\cdot, j}) = \frac{1}{\hrho}\varphi_{\lambda^*}^{\bbZ} (\bbZ_{\cdot, j}) - \varphi_{\lambda^*}^{\bbZ} (\bbA_{\cdot, j}) \in \text{span}\{u_1, \ldots, u_k \}.
\]
Hence, $ \langle \bbhA_{\cdot, j} -  \varphi_{\lambda^*}^{\bbZ} (\bbA_{\cdot, j}), \varphi_{\lambda^*}^{\bbZ} (\bbA_{\cdot, j}) - \bbA_{\cdot, j} \rangle = 0$ 
and
\begin{equation}\label{eq:column_error}
	\Big\| \bbhA_{\cdot, j} - \bbA_{\cdot, j} \Big\|_2^2
		= \Big\| \bbhA_{\cdot, j} - \varphi_{\lambda^*}^{\bbZ} \big( \bbA_{\cdot, j} \big) \Big\|_2^2		
			+ \Big\| \varphi_{\lambda^*}^{\bbZ} \big( \bbA_{\cdot, j} \big) - \bbA_{\cdot, j} \Big\|_2^2	
\end{equation}
by the Pythagorean theorem. It remains to bound the terms on the right hand side of \eqref{eq:column_error}. 

\paragraph{Step 2.}
We begin by bounding the first term on the right hand side of \eqref{eq:column_error}. 
Again applying Lemma \ref{lemma:column_representation}, we can rewrite
\begin{align*}
	\bbhA_{\cdot, j} - \varphi_{\lambda^*}^{\bbZ}(\bbA_{\cdot, j})
		&= \frac{1}{\hrho}\varphi_{\lambda^*}^{\bbZ}(\bbZ_{\cdot, j}) - \varphi_{\lambda^*}^{\bbZ}(\bbA_{\cdot, j})
		= \varphi_{\lambda^*}^{\bbZ} \Big(\frac{1}{\hrho} \bbZ_{\cdot, j} - \bbA_{\cdot, j} \Big)\\
		&=  \frac{1}{\hrho}  \varphi_{\lambda^*}^{\bbZ} (\bbZ_{\cdot, j} - \rho \bbA_{\cdot, j} ) 
			+ \frac{\rho - \hrho}{\hrho} \varphi_{\lambda^*}^{\bbZ}( \bbA_{\cdot, j} ).
\end{align*}
Using the Parallelogram Law (or, equivalently, combining Cauchy-Schwartz and AM-GM inequalities), we obtain
\begin{align}
\norm{\bbhA_{\cdot, j} - \varphi_{\lambda^*}^{\bbZ}(\bbA_{\cdot, j})}_2^2 
	&= \norm{\frac{1}{\hrho}  \varphi_{\lambda^*}^{\bbZ} (\bbZ_{\cdot, j} - \rho \bbA_{\cdot, j} ) 
		+ \frac{\rho - \hrho}{\hrho} \varphi_{\lambda^*}^{\bbZ}( \bbA_{\cdot, j}) }_2^2		\nonumber\\
	&\leq 2 \, \norm{\frac{1}{\hrho}  \varphi_{\lambda^*}^{\bbZ} (\bbZ_{\cdot, j} - \rho \bbA_{\cdot, j} ) }_2^2 
		+ 2 \, \norm{ \frac{\rho - \hrho}{\hrho} \varphi_{\lambda^*}^{\bbZ}( \bbA_{\cdot, j} )}_2^2		\nonumber\\
	&\leq \frac{2}{\hrho^2} \norm{\varphi_{\lambda^*}^{\bbZ}(\bbZ_{\cdot, j} - \rho \bbA_{\cdot, j})}_2^2
		+ 2 \Big( \frac{\rho - \hrho}{\hrho}\Big)^2 \| \bbA_{\cdot, j} \|_2^2		\nonumber\\
	&\leq \frac{2\varepsilon^2}{\rho^2} \norm{\varphi_{\lambda^*}^{\bbZ}(\bbZ_{\cdot, j} - \rho \bbA_{\cdot, j})}_2^2
		+ 2 (\varepsilon-1)^2 \| \bbA_{\cdot, j} \|_2^2.		\label{eqn:term.1a}
\end{align}
because Condition 2 implies $\frac{1}{\widehat{\rho}} \leq \frac{\varepsilon}{\rho}$ and 
$\left( \frac{\rho - \widehat{\rho}}{\widehat{\rho}} \right)^2 \leq (\varepsilon-1)^2$.

Note that the first term of \eqref{eqn:term.1a} can further be decomposed (using the Parallelogram Law and recalling 
$\bbA = \bbA^k + \bbE$, we have
\begin{align}
	&\norm{\varphi_{\lambda^*}^{\bbZ}(\bbZ_{\cdot, j} - \rho \bbA_{\cdot, j})}_2^2 \nonumber\\
	&\qquad\le 2  \,  \Big \| \varphi_{\lambda^*}^{\bbZ}(\bbZ_{\cdot, j} - \rho \bbA_{\cdot, j}) 
	 	- \varphi^{\bbA^k}(\bbZ_{\cdot, j} - \rho \bbA_{\cdot, j})  \Big \|_2^2 
	 	+2 \,  \Big \| \varphi^{\bbA^k}(\bbZ_{\cdot, j} - \rho \bbA_{\cdot, j})  \Big \|_2^2.		 \label{eq:tricky}
\end{align}

We now bound the first term on the right hand side of \eqref{eq:tricky} separately. First, we apply the Davis-Kahan 
$\sin \Theta$ Theorem (see \cite{davis1970rotation, wedin1972perturbation}) to arrive at the following inequality:
\begin{align}\label{eq:davis_kahan}
	\big\| \mathcal{P}_{u_1, \ldots, u_k} - \mathcal{P}_{\mu_1, \ldots, \mu_k} \big\|_2
		&\leq  \frac{\| \bbZ - \rho \bbA \|}{\rho \tau_k - \rho \tau_{k+1}}
		\leq  \frac{\Delta}{ \rho( \tau_k  - \tau_{k+1})}
\end{align}
where $\mathcal{P}_{u_1, \ldots, u_k}$ and $\mathcal{P}_{\mu_1, \ldots, \mu_k}$ denote the projection operators onto 
the span of the top $k$ left singular vectors of $\bbZ$ and $\bbA^k$, respectively. We utilized Condition 1 to bound 
$\| \bbZ - \rho \bbA \|_2 \leq \Delta$.
Then it follows that
\begin{align*} 
	 \Big \| \varphi_{\lambda^*}^{\bbZ}(\bbZ_{\cdot, j} - \rho \bbA_{\cdot, j}) 
	 	- \varphi^{\bbA^k}(\bbZ_{\cdot, j} - \rho \bbA_{\cdot, j})  \Big \|_2
	&\le \big\| \mathcal{P}_{u_1, \ldots, u_k} - \mathcal{P}_{\mu_1, \ldots, \mu_k} \big\|_2
		\big\| \bbZ_{\cdot, j} - \rho \bbA_{\cdot, j} \big\|_2	\nonumber\\
	&\le \frac{\Delta}{ \rho( \tau_k  - \tau_{k+1})}	\big\| \bbZ_{\cdot, j} - \rho \bbA_{\cdot, j} \big\|_2.	
\end{align*}
Combining the inequalities together, we have
\begin{align}
	\Big\| \bbhA_{\cdot, j} - \varphi_{\lambda^*}^{\bbZ} \big( \bbA_{\cdot, j} \big) \Big\|_2^2
		&\leq		\frac{4\varepsilon^2}{\rho^2} \frac{\Delta^2}{ \rho^2( \tau_k  - \tau_{k+1})^2}	
				\big\| \bbZ_{\cdot, j} - \rho \bbA_{\cdot, j} \big\|_2^2	\nonumber\\
		&\quad	+ \frac{4\varepsilon^2}{\rho^2}  \Big \| \varphi^{\bbA^k}(\bbZ_{\cdot, j} - \rho \bbA_{\cdot, j})  \Big \|_2^2
				+ 2 (\varepsilon-1)^2 \| \bbA_{\cdot, j} \|_2^2.	\label{eq:term_step2}
\end{align}

\paragraph{Step 3.}
We now bound the second term of \eqref{eq:column_error}. Recalling 
$\bbA = \bbA^k + \bbE$ and using \eqref{eq:davis_kahan}
\begin{align}
	\norm{\varphi_{\lambda^*}^{\bbZ} \big( \bbA_{\cdot, j} \big) - \bbA_{\cdot, j} }_2^2
		&= \norm{ \varphi_{\lambda^*}^{\bbZ} \big( \bbA^k_{\cdot, j} + \bbE_{\cdot, j} \big) - \bbA^k_{\cdot, j} - \bbE_{\cdot, j} }_2^2	\nonumber\\
		&\leq 2 \, \norm{\varphi_{\lambda^*}^{\bbZ} \big( \bbA^k_{\cdot, j} \big) - \bbA^k_{\cdot, j}  }_2^2 +  2 \,\norm{ \varphi_{\lambda^*}^{\bbZ} \big( \bbE_{\cdot, j} \big) - \bbE_{\cdot, j}}_2^2		\nonumber\\
		&= 2 \, \norm{\varphi_{\lambda^*}^{\bbZ} \big( \bbA^k_{\cdot, j} \big) - \varphi^{\bbA^k} \big( \bbA^k_{\cdot, j} \big)  }_2^2 +  2 \,\norm{ \varphi_{\lambda^*}^{\bbZ} \big( \bbE_{\cdot, j} \big) - \bbE_{\cdot, j}}_2^2		\nonumber\\
		&\leq 2 \, \norm{\mathcal{P}_{u_1, \ldots, u_k} - \mathcal{P}_{\mu_1, \ldots, \mu_k} }^2 \norm{ \bbA^k_{\cdot, j} }_2^2 + 2 \,\norm{ \bbE_{\cdot, j} }_2^2 \nonumber
		\\ &\le   \frac{2 \Delta^2}{ \rho^2( \tau_k  - \tau_{k+1})^2} \norm{ \bbA^k_{\cdot, j} }_2^2 + 2\, \norm{ \bbE_{\cdot, j} }_2^2.
		\label{eq:term_step3}
\end{align}

\noindent
Inserting \eqref{eq:term_step2} and \eqref{eq:term_step3} back to \eqref{eq:column_error} completes the proof.
\end{proof}

\subsection{High probability events for conditioning} 
We define the following four events: 
\begin{align}
\mathcal{E}_1 &:= \bigg\{ \norm{ \bbZ - \rho \bbA} \leq  
\sqrt{C_1} \left( \sqrt{N} + \sqrt{p} \log^{\frac 3 2}(Np) \right) \bigg\}	\label{eqn:cE_1}
	\\ \mathcal{E}_2&:= \Bigg\{ \bigg(1 - \sqrt{\frac{20 \log (Np)}{ Np \rho}}\bigg) \rho \le \widehat{\rho} 
					\le \frac{1}{1 - \sqrt{\frac{20 \log (Np)}{ Np \rho}}} \rho \Bigg\}		\label{eqn:cE_2}
	\\ \mathcal{E}_3 &:= \bigg\{ \max_{j \in [p]} \Big \| \bbZ_{\cdot, j} - \rho \bbA_{\cdot, j}  \Big \|_2^2 
							\leq 11 C K_{\alpha}^2 N \log^{\frac{2}{\alpha}}(Np)\bigg\}		\label{eqn:cE_3}
	\\ \mathcal{E}_4 &:= \bigg\{ \max_{j \in [p]} \Big \| \varphi^{\bbA^k}(\bbZ_{\cdot, j} - \rho \bbA_{\cdot, j})  \Big \|_2^2 \leq 
		11C K_{\alpha}^2 r \log^{\frac{2}{\alpha}}(Np)\bigg\}.		\label{eqn:cE_4}
\end{align}
Here, $C_1 = C(1+\sigma^2)(1+\gamma^2)(1+K_\alpha^2)$ for some constant $C > 0$.  

\paragraph{Observation 1: $\cE_1$ occurs with high probability.}

\begin{lemma}\label{lemma:E1}
Suppose that Properties \ref{prop:bounded_covariates}, \ref{prop:covariate_noise_structure} for $\alpha \ge 1$ hold. Then, $\Prob{\cE_1^c} \leq \frac{2}{N^{10} p^{10}}$. 
\end{lemma} 

\begin{proof}
The proof is complete by letting $\delta_1 = 10$ in Theorem \ref{thm:spectral_norm_noise_matrix_bound}.
\end{proof}


\paragraph{Observation 2: $\cE_2$ occurs with high probability.}
\begin{lemma}\label{lemma:E2}
For any $\varepsilon > 1$,
	\begin{align*}
		\Prob{  \frac{1}{\varepsilon} \rho \le \widehat{\rho} \le \varepsilon \rho } 
			\geq 1 - 2 \exp \left( - \frac{(\varepsilon - 1)^2}{2 \varepsilon^2} Np\rho \right).
	\end{align*}
\end{lemma}
\begin{proof}
Recall that $\hrho = \frac{1}{Np} \sum_{i=1}^{N} \sum_{j=1}^p \mathbb{1}(Z_{ij} \neq \star) \vee \frac{1}{Np}.$
By the binomial Chernoff bound, for $\varepsilon > 1$,
	\begin{align*}
		\Prob{ \widehat{\rho} > \varepsilon \rho }
			&\leq \exp\left( - \frac{(\varepsilon - 1 )^2}{\varepsilon + 1} Np \rho \right),	\quad\text{and}\\
		\Prob{ \widehat{\rho} < \frac{1}{\varepsilon} \rho  }
			&\leq \exp \left( - \frac{(\varepsilon - 1)^2}{2 \varepsilon^2} Np \rho \right).
	\end{align*}
	By the union bound,
	\[
		\Prob{  \frac{1}{\varepsilon} \rho \le \widehat{\rho} \le \varepsilon \rho }
			\geq 1 - \Prob{ \widehat{\rho} > \varepsilon \rho } -  \Prob{ \widehat{\rho} < \frac{1}{\varepsilon} \rho  }.
	\]
	Noticing $\varepsilon + 1 < 2 \varepsilon < 2 \varepsilon^2$ for all $\varepsilon > 1$ completes the proof.
\end{proof}

\begin{remark}\label{rem:E2}
	Let $\varepsilon = \left(1 - \sqrt{\frac{20 \log (Np)}{ Np \rho}} \right)^{-1} $ in Lemma \ref{lemma:E2}. 
	Then, $\Prob{\cE_2^c} \leq \frac{2}{N^{10}p^{10}}$. 
\end{remark}

\paragraph{Observation 3: $\cE_3$ and $\cE_4$ occur with high probability.}

\subsubsection{Two Helper Lemmas for $\cE_3$ and $\cE_4$}

\begin{lemma} \label{lemma:masked_noise_col_norm}
Assume Properties \ref{prop:bounded_covariates}, \ref{prop:covariate_noise_structure} hold. 
Then for any $\alpha \geq 1$, 
\[	\norm{ \bbZ_{\cdot, j} - \rho \bbA_{\cdot, j} }_{\psi_{\alpha}} \le C(K_{\alpha} + 1),	\qquad\forall j \in [p]	\] 
where $C > 0$ is an absolute constant. 
\end{lemma}

\begin{proof}
	Observe that
	\begin{align*}
		 \norm{ \bbZ_{\cdot, j} - \rho \bbA_{\cdot, j}}_{\psi_{\alpha}}
		 	&=  \sup_{u \in \mathbb{S}^{N-1}} \norm{ u^T \big( \bbZ_{\cdot, j} - \rho \bbA_{\cdot, j} \big)}_{\psi_{\alpha}}\\
			&=  \sup_{u \in \mathbb{S}^{N-1}} \norm{ u^T \big( \bbZ - \rho \bbA \big) e_j }_{\psi_{\alpha}}\\
			&=  \sup_{u \in \mathbb{S}^{N-1}} \norm{ \sum_{i=1}^n u_i \big( \bbZ_{i, \cdot} - \rho \bbA_{i, \cdot} \big) e_j }_{\psi_{\alpha}}\\
			&\stackrel{(a)}\leq C \sup_{u \in \mathbb{S}^{N-1}} \left( \sum_{i=1}^n  u_i^2  \norm{ \big( \bbZ_{i, \cdot} - \rho \bbA_{i, \cdot} \big)  e_j }_{\psi_{\alpha}}^2 \right)^{1/2}\\
			&\leq C \max_{i \in [N]} \norm{ \bbZ_{i, \cdot} - \rho \bbA_{i, \cdot} }_{\psi_{\alpha}},
	\end{align*}
	where (a) follows from Lemma \ref{lem:ind_sum}. Then the conclusion follows from Lemma \ref{lemma:masked_noise_row_norm}.
\end{proof}

\begin{lemma}\label{lem:norm_psi_alpha}
	Let $W_1, \ldots, W_n$ be a sequence of $\psi_{\alpha}$-random variables for some $\alpha \geq 1$. For any $t \geq 0$,
	\[
		\Prob{ \sum_{i=1}^n W_i^2 > t } 
		\leq 2 \sum_{i=1}^n \exp \left( - \left( \frac{t}{n \| W_i \|_{\psi_{\alpha}}^2 }\right)^{\alpha/2}  \right).
	\]
\end{lemma}

\begin{proof}
	Note that $ \sum_{i=1}^n W_i^2 > t$ implies that there exists at least one $i \in [n]$ with $W_i^2 > \frac{t}{n}$.
	By the union bound,
	\begin{align*}
		\Prob{ \sum_{i=1}^n W_i^2 > t } 
			&\leq \sum_{i=1}^n \Prob{W_i^2 > \frac{t}{n}}
			&\leq \sum_{i=1}^n \Prob{|W_i| > \sqrt{\frac{t}{n}}}
			\leq \sum_{i=1}^n 2 \exp \left( - \left( \frac{t}{n \| W_i \|^2_{\psi_{\alpha}} }\right)^{\alpha/2}  \right). 
	\end{align*}
\end{proof}

\begin{lemma}\label{lemma:E3}
	Suppose Properties \ref{prop:bounded_covariates}, \ref{prop:covariate_noise_structure} hold. Then,
	\[
	\Prob{\cE_3^c} \leq \frac{2}{N^{10}p^{10}}.
	\]
\end{lemma}

\begin{proof}
	Fix $j \in [p]$. Let $e_i \in \Reals^N$ denote the $i$-th canonical basis of $\Reals^N$ (column vector representation). 
	Note that
	\[
	\Big \| \bbZ_{\cdot, j} - \rho \bbA_{\cdot, j}  \Big \|_2^2 = \sum_{i=1}^N \Big( e_i^T \big( \bbZ_{\cdot, j} - \rho \bbA_{\cdot, j}  \big) \Big)^2
	\]
	and $e_i^T \big( \bbZ_{\cdot, j} - \rho \bbA_{\cdot, j}  \big)$ is a $\psi_{\alpha}$-random variable with 
	$\norm{ e_i^T \big( \bbZ_{\cdot, j} - \rho \bbA_{\cdot, j}  \big) }_{\psi_{\alpha}} \leq \norm{ \bbZ_{\cdot, j} - \rho \bbA_{\cdot, j} }_{\psi_{\alpha}} $. 
	By Lemma \ref{lemma:masked_noise_col_norm}, $ \norm{ \bbZ_{\cdot, j} - \rho \bbA_{\cdot, j} }_{\psi_{\alpha}}  \leq C(K_{\alpha} + 1)$ 
	for all $j \in [p]$. By Lemma \ref{lem:norm_psi_alpha} and the union bound,
	\begin{align*}
		\Prob{\mathcal{E}_3^c}
			&\leq \sum_{j=1}^p \Prob{ \Big \| \bbZ_{\cdot, j} - \rho \bbA_{\cdot, j}  \Big \|_2^2 > 11C^2 (K_{\alpha} + 1)^2 N \log^{\frac{2}{\alpha}}(Np) }\\
			&\leq 2 \sum_{j=1}^p \sum_{i=1}^N \exp\left( -11 \log(Np) \right)\\
			&= \frac{2}{N^{10}p^{10}}.
	\end{align*}
\end{proof}

\begin{lemma}\label{lemma:E4}
	Suppose properties \ref{prop:bounded_covariates}, \ref{prop:covariate_noise_structure} hold. Then,
	\[
	\Prob{\cE_4^c} \leq \frac{2}{N^{10}p^{10}}.
	\]
\end{lemma}

\begin{proof}
	Recall that $\text{rank}(\bbA^k) = k$. We write 
	\[
		 \Big \| \varphi^{\bbA^k}(\bbZ_{\cdot, j} - \rho \bbA_{\cdot, j})  \Big \|_2^2
		 	= \sum_{i=1}^{k}  \Big( u_i^T (\bbZ_{\cdot, j} - \rho \bbA_{\cdot, j}) \Big)^2,
	\]
	where $u_1, \ldots, u_{k}$ denote the left singular vectors of $\bbA^k$. The proof has the same structure 
	with that of Lemma \ref{lemma:E3} with $u_1, \ldots, u_{k}$ in place of $e_1, \ldots, e_n$. 
\end{proof}

\subsection{Completing Proof of Lemma \ref{lemma:mcse_hvst}}\label{sec:proof_thm_MCSE}

\begin{proof}[Proof of Lemma \ref{lemma:mcse_hvst}]
Recall that our goal is to establish
\begin{align}\label{eq:MCSE_train_bound.1}
\Ex[\|\bZ^{\text{HSVT}, k} - \bbA\|_{2,\infty}^2]		
&\le \frac{C (K_\alpha^2 + 1)}{\rho^2} \Big( k + \frac{N \Delta^2}{\rho^2 (\tau_k - \tau_{k+1})^2} \Big) \log^{\frac2\alpha} Np + 2 \| \bbA^k - \bbA \|_{2,\infty}^2,
\end{align}
where $C > 0$ is a universal constant. To that end, define $E \triangleq \cE_1 \cap \cE_2 \cap \cE_3 \cap \cE_4$. 
By Lemmas \ref{lemma:E1}, \ref{lemma:E2}, \ref{lemma:E3} and \ref{lemma:E4}, it follows that
\begin{align*}
	\Prob{E^c} 
		&\leq \Prob{\cE_1^c \cup \cE_2^c \cup \cE_3^c \cup \cE_4^c}~\leq~\frac{8}{N^{10}p^{10}}.
\end{align*} 
Observe (with $\bhA = \bZ^{\text{HSVT}, k}$), 
\begin{align}
\Ex[\|\bhA - \bbA\|_{2,\infty}^2]	&= \Ex  \max_{j \in [p]} \norm{\bhA_{\cdot, j} - \bA_{\cdot, j}}_2^2 	\nonumber\\ 
	&= \Ex \bigg[ \max_{j \in [p]} \norm{\bhA_{\cdot, j} - \bA_{\cdot, j}}_2^2  \cdot \mathbb{1}(E)\bigg] 
			+  \Ex \left[ \max_{j \in [p]} \norm{\bhA_{\cdot, j} - \bA_{\cdot, j}}_2^2  \cdot \mathbb{1}(E^c)\right].
		\label{eq:MCSE_decomp}
\end{align}
In the rest of the proof, we upper bound the two terms in \eqref{eq:MCSE_decomp} separately.

\paragraph{Upper bound on the first term in \eqref{eq:MCSE_decomp}.} Under event $E$, from 
Lemma \ref{lemma:column_error}, we have 
\begin{align*}
	\max_{j \in [p]} \Big\| \bbhA_{\cdot, j} - \bbA_{\cdot, j} \Big\|_2^2 
	&\leq	 \frac{C (K_{\alpha} + 1)^2}{\rho^2}
		\left( \frac{  \Delta^2 N }{ \rho^2( \tau_{r} - \tau_{r+1} )^2} + r \right) \log^{\frac{2}{\alpha}}(Np) 	
		+ 2\max_{j \in [p]} \norm{ \bbE_{\cdot, j} }_2^2.
\end{align*}
where $C > 0$ is an absolute constant. To see this, note that $\varepsilon^2 \leq 10$ since $\rho \geq \frac{64 \log(Np)}{Np}$;
$\|\bbA^k_j\|_2^2 \leq \|\bbA_j\|_2^2 \leq N$, again appealing to the contraction property of the HSVT operator (refer to Lemma \ref{lemma:HSVT_contraction} and Property \ref{prop:bounded_covariates}). Since $\Prob{E} \leq 1$,  it follows that 
\begin{align}
\Ex \bigg[ \norm{\bhA - \bA}_{2, \infty}^2  \cdot \mathbb{1}(E)\bigg] 
		&\leq	 \frac{C (K_{\alpha} + 1)^2}{\rho^2}
		\left( \frac{  \Delta^2 N }{ \rho^2( \tau_{r} - \tau_{r+1} )^2} + r \right) \log^{\frac{2}{\alpha}}(Np) 	
		+ 2\max_{j \in [p]} \norm{ \bbE_{\cdot, j} }_2^2.		\label{eqn:term1_upper}
\end{align}

\paragraph{Upper bound on the second term in \eqref{eq:MCSE_decomp}.}

To begin with, we note that for any $j \in [p]$,
\[
	\norm{\bhA_{\cdot, j} - \bA_{\cdot,j}}_2 
		\le\norm{ \bhA_{\cdot, j} }_2 +  \| \bA_{\cdot, j} \|_2 
\]
by triangle inequality. By the model assumption, the covariates are bounded (Property \ref{prop:bounded_covariates}) and $\norm{ \bA_{\cdot, j} }_2 \leq \sqrt{N}$ for all $j \in [p]$.
By definition, for any $j \in [p]$,
\[	 \bhA_{\cdot, j} =  \frac{1}{\hrho} \text{HSVT}_{\lambda}\big(\bZ  \big)_{\cdot,j} 	\]
for a given threshold $\lambda = s_k$, the $k$th singular value of $\bbZ$.
Therefore,
\[
	\| \bhA_{\cdot, j} \|_2 = \frac{1}{\hrho} \big\|  \text{HSVT}_{\lambda}\big(\bZ  \big)_{\cdot,j} \big\|_2 
		\stackrel{(a)}{\le} Np \big\| \text{HSVT}_{\lambda}\big(\bZ  \big)_{\cdot,j} \big\|_2
		\stackrel{(b)}{\le} Np \| \bZ_{\cdot,j} \|_2.
\]
Here, (a) follows from $\hrho \geq \frac{1}{Np}$; and 
(b) follows from Lemma \ref{lemma:HSVT_contraction} -- the $\text{HSVT}$ operator is a contraction on the columns. 
%
%
\begin{align} 
	 \max_{j \in [p]} \| \bhA_{\cdot, j} - \bA_{\cdot,j} \|_2 
		 &\le \max_{j \in [p]} \| \bhA_{\cdot, j} \|_2 + \max_{j \in [p]} \, \| \bA_{\cdot, j} \|_2 \nonumber
		 \\&\le Np ~ \max_{j \in [p]} \| \bZ_{\cdot,j} \|_2 +  \sqrt{N} \nonumber
		 \\&\le \big(N^{\frac{3}{2}} p + \sqrt{N} \big) + N^{\frac{3}{2}}p \max_{ij} \abs{\eta_{ij}}	\nonumber
		  \\&\le 2N^{\frac{3}{2}} p \Big( 1 + \max_{ij} \abs{\eta_{ij}} \Big)	\label{eq:proof_mcse_1}
\end{align}
because $\max_{j \in [p]} \| \bZ_{\cdot,j} \|_2 ~	\leq \sqrt{N} \max_{i,j} \abs{Z_{ij}}
\leq \sqrt{N} \max_{i,j} \abs{A_{ij} + \eta_{ij}} 	\leq \sqrt{N} \big( 1 +  \max_{i,j} \abs{\eta_{ij}} \big)$.
Now we apply Cauchy-Schwarz inequality on $  \Ex \big[ \max_{j \in [p]} \|\bhA_{\cdot, j} - \bA_{\cdot, j}\|_2^2
 \cdot \mathbb{1}(E^c)\big]$ to obtain 
\begin{align}
	\Ex \Big[ \max_{j \in [p]} \big\|\bhA_{\cdot, j} - \bA_{\cdot, j}\big\|_2^2 \cdot \mathbb{1}(E^c)\Big] 
		&\leq \Ex \Big[ \max_{j \in [p]} \big\|\bhA_{\cdot, j} - \bA_{\cdot, j}\big\|_2^4\Big]^{\frac{1}{2}} 
			\cdot \Ex \Big[ \mathbb{1}(E^c)\Big]^{\frac{1}{2}}	\nonumber\\
		&= \Ex \Big[ \max_{j \in [p]} \big\|\bhA_{\cdot, j} - \bA_{\cdot, j}\big\|_2^4\Big]^{\frac{1}{2}} 
			\cdot \Prob{ E^c }^{\frac{1}{2}}		\nonumber\\
		&\stackrel{(a)}{\leq} 4 N^3 p^2 \Ex \Big[ \Big( 1 + \max_{ij} \abs{\eta_{ij}} \Big)^4 \Big]^{\frac{1}{2}}
			\cdot \Prob{ E^c }^{\frac{1}{2}}		\nonumber\\
		&\stackrel{(b)}{\leq} 8\sqrt{2} N^3 p^2  \Big( 1 + \Ex \big[ \max_{ij} \abs{\eta_{ij}}^4 \big] \Big)^{\frac{1}{2}}
			\cdot \Prob{ E^c }^{\frac{1}{2}}		\nonumber\\
		&\stackrel{(c)}{\leq} 8\sqrt{2} N^3 p^2  \Big( 1 + \Ex \big[ \max_{ij} \abs{\eta_{ij}}^4 \big]^{\frac{1}{2}} \Big)
			\cdot \Prob{ E^c }^{\frac{1}{2}}.		\label{eq:proof_mcse_2}
\end{align}
Here, (a) follows from \eqref{eq:proof_mcse_1}; and (b) follows from Jensen's inequality:
\begin{align*}	
	\Ex \Big[ \Big( 1 + \max_{ij} \abs{\eta_{ij}} \Big)^4 \Big]
		&= \Ex \bigg[ \Big( \frac{1}{2} \big( 2  + 2 \max_{ij} \abs{\eta_{ij}} \big) \Big)^4  \bigg]
		\leq \Ex \bigg[ \frac{1}{2} \Big( 2^4 + \big( 2 \max_{ij} \abs{\eta_{ij}} \big)^4 \Big) \bigg]\\
		&= 8 \Ex \Big[ 1 + \max_{ij} \abs{\eta_{ij}}^4 \Big]
		= 8 \Big( 1 +  \Ex [  \max_{ij} \abs{\eta_{ij}}^4 ] \Big) ;
\end{align*}
and (c) follows from the trivial inequality: $\sqrt{A + B} \leq \sqrt{A} + \sqrt{B}$ for any $A, B \geq 0$.

Now it remains to find an upper bound for $\Ex \big[ \max_{ij} \abs{\eta_{ij}}^4 \big]$. Note that for any $\alpha >0$ and 
$\theta \geq 1$, $\eta_{ij}$ being a $\psi_{\alpha}$-random variable implies that $\big| \eta_{ij}\big|^{\theta}$ is a 
$\psi_{\alpha/\theta}$-random variable. With the choice of $\theta =4 $, we have that 
\begin{align} 
	\Ex \max_{ij} \abs{\eta_{ij}}^4 &\le C_1 K_\alpha^4 \log^{\frac{4}{\alpha}}(Np)	\label{eq:proof_mcse_3}
\end{align}
for some absolute constant $C_1 > 0$ by Lemma \ref{lemma:max_subg} (also see Remark \ref{rem:max_psialpha}).
Inserting \eqref{eq:proof_mcse_3} to \eqref{eq:proof_mcse_2} yields
\begin{align}
	\Ex \Big[ \max_{j \in [p]} \big\|\bhA_{\cdot, j} - \bA_{\cdot, j}\big\|_2^2 \cdot \mathbb{1}(\Econd^c)\Big] 
		&\leq 8\sqrt{2} N^3 p^2  \Big( 1 + {C_1'}^{1/2} K_\alpha^2 \log^{\frac{2}{\alpha}}(Np) \Big)
			\cdot \Prob{ E^c }^{\frac{1}{2}}	\nonumber\\
		&\stackrel{(a)}{\leq }
			32 \Big( 1 + {C_1}^{1/2} K_\alpha^2 \log^{\frac{2}{\alpha}}(Np) \Big) \frac{1}{N^2 p^{2}},		\label{eqn:term2_upper}
\end{align}
where (a) follows from recalling that $\Prob{E^c} \leq 8/N^{10}p^{10}$.

\paragraph{Concluding the Proof.}
Thus, combining \eqref{eqn:term1_upper} and \eqref{eqn:term2_upper} in \eqref{eq:MCSE_decomp} 
and noticing that term in \eqref{eqn:term2_upper} is smaller order term than that in \eqref{eqn:term1_upper}, by
defining appropriate constant $C > 0$, we obtain:
\begin{align}
\Ex[\|\bhA - \bbA\|_{2,\infty}^2] & \leq 
\frac{C (K_{\alpha} + 1)^2}{\rho^2}
		\left( \frac{  \Delta^2 N }{ \rho^2( \tau_{r} - \tau_{r+1} )^2} + r \right) \log^{\frac{2}{\alpha}}(Np) 	
		+ 2\max_{j \in [p]} \norm{ \bbE_{\cdot, j} }_2^2 \\
		& \qquad + \frac{C}{N^2 p^{2}} \Big(1 + K^2_\alpha \log^{\frac{2}{\alpha}}(Np) \Big),
\end{align}
with 
$
\Delta = \sqrt{C^*} \left( \sqrt{N} + \sqrt{p} \log^{\frac 3 2}(Np) \right)$
and $C^* = C(1+\sigma^2)(1+\gamma^2)(1+K_\alpha^2)$. 

\noindent The proof is complete by defining  $C' = C(1+\sigma^2)(1+\gamma^2)(1+K_\alpha^4)$ and simplifying the bound further in a straightforward manner.
\end{proof}



{
\section{Proof of Lemma \ref{lemma:mcse_hvst_LVM}}\label{sec:appendix_mcse_hsvt_LVM}
The proof of Lemma \ref{lemma:mcse_hvst_LVM} follows very closely the structure of the proof of Lemma \ref{lemma:mcse_hvst}. The key difference is Lemma \ref{lemma:column_error} no longer holds as is, and needs to be redefined for $\bbA^{\text{(lr)}}$ instead of $\bbA^k$.

\begin{lemma} \label{lemma:column_error_LVM}
Suppose that (1) $\norm{ \bbZ - \rho \bbA} \leq \Delta $ for some $\Delta \geq 0$ and (2) $ \frac{1}{\varepsilon} \rho \le \widehat{\rho} \le \varepsilon \rho$ for some $\varepsilon \ge 1$.

Let $\bbA = \bbA^{\emph{(lr)}} + \bbE^{\emph{(lr)}}$. 
Let $r = \text{rank}( \bbA^{\emph{(lr)}})$ and  $\tau_{r}$ denote the $r$-th singular value of $\bbA^{\emph{(lr)}}$.
Let $\bhA = \bbZ^{\emph{HSVT}, r}$.
Then for any $j \in [p]$, 
\begin{align}
	\Big\| \bbhA_{\cdot, j} - \bbA_{\cdot, j} \Big\|_2^2 
		&\leq		\frac{8\varepsilon^2}{\rho^4}\Big( \frac{\Delta^2}{  \tau^2_r} +  \frac{\|  \bbE^{\emph{(lr)}} \|_2^2}{\tau^2_r} \Big)	
				\Big(\big\| \bbZ_{\cdot, j} - \rho \bbA_{\cdot, j} \big\|_2^2 + \norm{ \bbA^{\emph{(lr)}}_{\cdot, j} }_2^2 \Big)	\nonumber \\
		&\quad	+ \frac{4\varepsilon^2}{\rho^2}  \Big \| \varphi^{\bbA^{\emph{(lr)}}}(\bbZ_{\cdot, j} - \rho \bbA_{\cdot, j})  \Big \|_2^2 + 2 (\varepsilon-1)^2 \| \bbA_{\cdot, j} \|_2^2 \\
		&\quad	+ 2\, \norm{ \bbE^{\emph{(lr)}}_{\cdot, j} }_2^2. \label{eq:main_MCSE_inequality_LVM}
\end{align}
\end{lemma}

The proof of Lemma \ref{lemma:column_error_LVM} is almost identical to the proof of Lemma \ref{lemma:column_error}, 
except the replacement of the subspace perturbation bound \eqref{eq:davis_kahan} with a new one in \eqref{eq:davis_kahan_LVM}. 
Roughly speaking, we control the principal angle between the top-$r$ left singular space of $\bbZ$ and the column space of $\bbA^{\text{(lr)}}$ 
by means of triangle inequality, using the column space of $\bbA^{k}$ as an intermeidary. Despite the similarity to the proof of 
Lemma \ref{lemma:column_error}, we present the full proof of Lemma \ref{lemma:column_error_LVM} for future reference in synthetic control literature.

\begin{proof}
We will use notation $\lambda^* = s_r$, the $r$th singular value of $\bbZ$ for simplicity. We prove our Lemma in three steps. 
\paragraph{Step 1.}
Fix a column index $j \in [p]$. Observe that
\begin{equation*}
	\bbhA_{\cdot, j} - \bbA_{\cdot, j}
		= \Big( \bbhA_{\cdot, j} - \varphi_{\lambda^*}^{\bbZ} \big( \bbA_{\cdot, j} \big) \Big) + \Big( \varphi_{\lambda^*}^{\bbZ} \big( \bbA_{\cdot, j} \big) - \bbA_{\cdot, j} \Big).
\end{equation*}
By choice,  $\text{rank}(\bbhA) = r$. By definition (see \eqref{eq:prox_vector}), we have that $\varphi_{\lambda^*}^{\bbZ}: \Reals^N \to \Reals^N$ is 
the projection operator onto the span of the top $r$ left singular vectors of $\bbZ$, namely, $\text{span}\big\{ u_1, \ldots, u_r \big\}$. Therefore, 
\[
\varphi_{\lambda^*}^{\bbZ} (\bbA_{\cdot, j}) - \bbA_{\cdot, j} \in \text{span}\{u_1, \ldots, u_r \}^{\perp}
\]
and by \eqref{eq:repsentation_A_hat} (using Lemma \ref{lemma:column_representation}),
\[
\bbhA_{\cdot, j} -  \varphi_{\lambda^*}^{\bbZ} (\bbA_{\cdot, j}) = \frac{1}{\hrho}\varphi_{\lambda^*}^{\bbZ} (\bbZ_{\cdot, j}) - \varphi_{\lambda^*}^{\bbZ} (\bbA_{\cdot, j}) \in \text{span}\{u_1, \ldots, u_r \}.
\]
Hence, $ \langle \bbhA_{\cdot, j} -  \varphi_{\lambda^*}^{\bbZ} (\bbA_{\cdot, j}), \varphi_{\lambda^*}^{\bbZ} (\bbA_{\cdot, j}) - \bbA_{\cdot, j} \rangle = 0$ 
and
\begin{equation}\label{eq:column_error_LVM}
	\Big\| \bbhA_{\cdot, j} - \bbA_{\cdot, j} \Big\|_2^2
		= \Big\| \bbhA_{\cdot, j} - \varphi_{\lambda^*}^{\bbZ} \big( \bbA_{\cdot, j} \big) \Big\|_2^2		
			+ \Big\| \varphi_{\lambda^*}^{\bbZ} \big( \bbA_{\cdot, j} \big) - \bbA_{\cdot, j} \Big\|_2^2	
\end{equation}
by the Pythagorean theorem. It remains to bound the terms on the right hand side of \eqref{eq:column_error_LVM}. 

\paragraph{Step 2.}
We begin by bounding the first term on the right hand side of \eqref{eq:column_error_LVM}. 
Again applying Lemma \ref{lemma:column_representation}, we can rewrite
\begin{align*}
	\bbhA_{\cdot, j} - \varphi_{\lambda^*}^{\bbZ}(\bbA_{\cdot, j})
		&= \frac{1}{\hrho}\varphi_{\lambda^*}^{\bbZ}(\bbZ_{\cdot, j}) - \varphi_{\lambda^*}^{\bbZ}(\bbA_{\cdot, j})
		= \varphi_{\lambda^*}^{\bbZ} \Big(\frac{1}{\hrho} \bbZ_{\cdot, j} - \bbA_{\cdot, j} \Big)\\
		&=  \frac{1}{\hrho}  \varphi_{\lambda^*}^{\bbZ} (\bbZ_{\cdot, j} - \rho \bbA_{\cdot, j} ) 
			+ \frac{\rho - \hrho}{\hrho} \varphi_{\lambda^*}^{\bbZ}( \bbA_{\cdot, j} ).
\end{align*}
Using the Parallelogram Law (or, equivalently, combining Cauchy-Schwartz and AM-GM inequalities), we obtain
\begin{align}
\norm{\bbhA_{\cdot, j} - \varphi_{\lambda^*}^{\bbZ}(\bbA_{\cdot, j})}_2^2 
	&= \norm{\frac{1}{\hrho}  \varphi_{\lambda^*}^{\bbZ} (\bbZ_{\cdot, j} - \rho \bbA_{\cdot, j} ) 
		+ \frac{\rho - \hrho}{\hrho} \varphi_{\lambda^*}^{\bbZ}( \bbA_{\cdot, j}) }_2^2		\nonumber\\
	&\leq 2 \, \norm{\frac{1}{\hrho}  \varphi_{\lambda^*}^{\bbZ} (\bbZ_{\cdot, j} - \rho \bbA_{\cdot, j} ) }_2^2 
		+ 2 \, \norm{ \frac{\rho - \hrho}{\hrho} \varphi_{\lambda^*}^{\bbZ}( \bbA_{\cdot, j} )}_2^2		\nonumber\\
	&\leq \frac{2}{\hrho^2} \norm{\varphi_{\lambda^*}^{\bbZ}(\bbZ_{\cdot, j} - \rho \bbA_{\cdot, j})}_2^2
		+ 2 \Big( \frac{\rho - \hrho}{\hrho}\Big)^2 \| \bbA_{\cdot, j} \|_2^2		\nonumber\\
	&\leq \frac{2\varepsilon^2}{\rho^2} \norm{\varphi_{\lambda^*}^{\bbZ}(\bbZ_{\cdot, j} - \rho \bbA_{\cdot, j})}_2^2
		+ 2 (\varepsilon-1)^2 \| \bbA_{\cdot, j} \|_2^2.		\label{eqn:term.1a_LVM}
\end{align}
because Condition 2 implies $\frac{1}{\widehat{\rho}} \leq \frac{\varepsilon}{\rho}$ and 
$\left( \frac{\rho - \widehat{\rho}}{\widehat{\rho}} \right)^2 \leq (\varepsilon-1)^2$.

Note that the first term of \eqref{eqn:term.1a_LVM} can further be decomposed (using the Parallelogram Law and recalling 
$\bbA = \bbA^{\text{(lr)}} + \bbE^{\text{(lr)}}$, we have
\begin{align}
	&\norm{\varphi_{\lambda^*}^{\bbZ}(\bbZ_{\cdot, j} - \rho \bbA_{\cdot, j})}_2^2 \nonumber\\
	&\qquad\le 2  \,  \Big \| \varphi_{\lambda^*}^{\bbZ}(\bbZ_{\cdot, j} - \rho \bbA_{\cdot, j}) 
	 	- \varphi^{\bbA^{\text{(lr)}}}(\bbZ_{\cdot, j} - \rho \bbA_{\cdot, j})  \Big \|_2^2 
	 	+2 \,  \Big \| \varphi^{\bbA^{\text{(lr)}}}(\bbZ_{\cdot, j} - \rho \bbA_{\cdot, j})  \Big \|_2^2.		 \label{eq:tricky_LVM}
\end{align}

We now bound the first term on the right hand side of \eqref{eq:tricky} separately. First, we apply the Davis-Kahan 
$\sin \Theta$ Theorem (see \cite{davis1970rotation, wedin1972perturbation}) to arrive at the following inequality:
\begin{align}\label{eq:davis_kahan_LVM}
	\big\| \mathcal{P}_{u_1, \ldots, u_r} - \mathcal{P}_{\mu_1, \ldots, \mu_r} \big\|_2
		&\leq  \frac{\| \bbZ - \rho \bbA^{\text{(lr)}} \|_2}{\rho \tau_r } \\
		&\leq  \frac{\| \bbZ - \rho \bbA \|_2}{\rho \tau_r} +  \frac{\| \rho \bbA - \rho \bbA^{\text{(lr)}} \|_2}{\rho \tau_r} \\
		&\leq  \frac{\Delta}{ \rho \tau_r} +  \frac{\|  \bbE^{\text{(lr)}} \|_2}{\rho \tau_r},
\end{align}
where $\mathcal{P}_{u_1, \ldots, u_r}$ and $\mathcal{P}_{\mu_1, \ldots, \mu_r}$ denote the projection operators onto 
the span of the top $r$ left singular vectors of $\bbZ$ and $\bbA^{\text{(lr)}}$, respectively. We utilized Condition 1 to bound 
$\| \bbZ - \rho \bbA \|_2 \leq \Delta$.
Then it follows that
\begin{align*} 
	 \Big \| \varphi_{\lambda^*}^{\bbZ}(\bbZ_{\cdot, j} - \rho \bbA_{\cdot, j}) 
	 	- \varphi^{\bbA^{\text{(lr)}}}(\bbZ_{\cdot, j} - \rho \bbA_{\cdot, j})  \Big \|_2
	&\le \big\| \mathcal{P}_{u_1, \ldots, u_r} - \mathcal{P}_{\mu_1, \ldots, \mu_r} \big\|_2
		\big\| \bbZ_{\cdot, j} - \rho \bbA_{\cdot, j} \big\|_2	\nonumber\\
	&\le \Big( \frac{\Delta}{ \rho \tau_r} +  \frac{\|  \bbE^{\text{(lr)}} \|_2}{\rho \tau_r} \Big)	\big\| \bbZ_{\cdot, j} - \rho \bbA_{\cdot, j} \big\|_2.	
\end{align*}
Combining the inequalities together, we have
\begin{align}
	\Big\| \bbhA_{\cdot, j} - \varphi_{\lambda^*}^{\bbZ} \big( \bbA_{\cdot, j} \big) \Big\|_2^2
		&\leq		\frac{8\varepsilon^2}{\rho^4}\Big( \frac{\Delta^2}{  \tau^2_r} +  \frac{\|  \bbE^{\text{(lr)}} \|_2^2}{\tau^2_r} \Big)	
				\big\| \bbZ_{\cdot, j} - \rho \bbA_{\cdot, j} \big\|_2^2	\nonumber\\
		&\quad	+ \frac{4\varepsilon^2}{\rho^2}  \Big \| \varphi^{\bbA^{\text{(lr)}}}(\bbZ_{\cdot, j} - \rho \bbA_{\cdot, j})  \Big \|_2^2
				+ 2 (\varepsilon-1)^2 \| \bbA_{\cdot, j} \|_2^2.	\label{eq:term_step2_LVM}
\end{align}

\paragraph{Step 3.}
We now bound the second term of \eqref{eq:column_error_LVM}. Recalling 
$\bbA = \bbA^{\text{(lr)}} + \bbE^{\text{(lr)}}$ and using \eqref{eq:davis_kahan_LVM}
\begin{align}
	\norm{\varphi_{\lambda^*}^{\bbZ} \big( \bbA_{\cdot, j} \big) - \bbA_{\cdot, j} }_2^2
		&= \norm{ \varphi_{\lambda^*}^{\bbZ} \big( \bbA^{\text{(lr)}}_{\cdot, j} + \bbE^{\text{(lr)}}_{\cdot, j} \big) -\bbA^{\text{(lr)}}_{\cdot, j} - \bbE^{\text{(lr)}}_{\cdot, j} }_2^2	\nonumber\\
		&\leq 2 \, \norm{\varphi_{\lambda^*}^{\bbZ} \big( \bbA^{\text{(lr)}}_{\cdot, j} \big) - \bbA^{\text{(lr)}}_{\cdot, j}  }_2^2 +  2 \,\norm{ \varphi_{\lambda^*}^{\bbZ} \big( \bbE^{\text{(lr)}}_{\cdot, j} \big) - \bbE^{\text{(lr)}}_{\cdot, j}}_2^2		\nonumber\\
		&= 2 \, \norm{\varphi_{\lambda^*}^{\bbZ} \big( \bbA^{\text{(lr)}}_{\cdot, j} \big) - \varphi^{\bbA^{\text{(lr)}}} \big( \bbA^{\text{(lr)}}_{\cdot, j} \big)  }_2^2 +  2 \,\norm{ \varphi_{\lambda^*}^{\bbZ} \big( \bbE^{\text{(lr)}}_{\cdot, j} \big) - \bbE^{\text{(lr)}}_{\cdot, j}}_2^2		\nonumber\\
		&\leq 2 \, \norm{\mathcal{P}_{u_1, \ldots, u_r} - \mathcal{P}_{\mu_1, \ldots, \mu_r} }^2 \norm{\bbA^{\text{(lr)}}_{\cdot, j} }_2^2 + 2 \,\norm{ \bbE^{\text{(lr)}}_{\cdot, j} }_2^2 \nonumber
		\\ &\le  4\Big( \frac{\Delta^2}{ \rho^2 \tau^2_r} +  \frac{\|  \bbE^{\text{(lr)}} \|_2^2}{\rho^2 \tau^2_r} \Big)  \norm{ \bbA^{\text{(lr)}}_{\cdot, j} }_2^2 + 2\, \norm{ \bbE^{\text{(lr)}}_{\cdot, j} }_2^2.
		\label{eq:term_step3_LVM}
\end{align}

\noindent
Inserting \eqref{eq:term_step2_LVM} and \eqref{eq:term_step3_LVM} back to \eqref{eq:column_error_LVM} completes the proof.
\end{proof}

\subsection{Completing Proof of Lemma \ref{lemma:mcse_hvst_LVM}}\label{sec:proof_thm_MCSE_LVM}
\begin{proof}[Proof of Lemma \ref{lemma:mcse_hvst_LVM}]
Proof follows in an identical fashion to that of Lemma \ref{lemma:mcse_hvst} (see Section \ref{sec:proof_thm_MCSE}) and using the bound in Lemma \ref{lemma:column_error_LVM} instead of the one in Lemma \ref{lemma:column_error}. 
\end{proof}
}

\section{Proof of Corollary \ref{cor:LVM-spectra-error}}\label{sec:proof_LVM_MSE}

\begin{proof}
From Proposition \ref{prop:lvm}, we have that $r \le C(\zeta, K) \delta^{-K}$ and $\| \bE^{\text{(lr)}} \|_\infty \le  \cL \cdot \delta^\zeta$.
So, $\tau^{2}_r \ge C Np / r \ge C Np /  (C(\zeta, K) \delta^{-K})$.
\begin{align} 
\emph{MSE}_{\Omega}(\hY) 
& \le \frac{4\sigma^2 r}{n}  
+ \frac{C' \| \beta^*\|_1^2}{\rho^4} \left( \frac{n \vee p \vee \| \bE^{\text{(lr)}} \|_2^2}{\tau_r^2} + \frac{r}{n} \right) \log^5(np) 
+ \frac{6 \|\beta^*\|_1^2}{n} \| \bE^{\text{(lr)}} \|^2_{2, \infty} 
\,+\,  \frac{20}{n} \| \phi \|_2^2 \\ 
& \le  
 \frac{C'  \|  \beta^*\|_1^2}{\rho^4} \left( \frac{n \vee p \vee \| \bE^{\text{(lr)}} \|_2^2}{\tau_r^2} + \frac{r}{n} + \frac{\| \bE^{\text{(lr)}} \|^2_{2, \infty}}{n} \right) \log^5(np) 
\,+\,  \frac{20}{n} \| \phi \|_2^2 \\
& \le  
 \frac{C'  \|  \beta^*\|_1^2}{\rho^4} \left( \frac{ r \| \bE^{\text{(lr)}} \|_2^2}{N p} + \frac{r}{n \wedge p} +\frac{\| \bE^{\text{(lr)}} \|^2_{2, \infty}}{n} \right) \log^5(np) 
\,+\,  \frac{20}{n} \| \phi \|_2^2 \\
& \le  
 \frac{C'  \|  \beta^*\|_1^2}{\rho^4} \left( \ r \| \bE^{\text{(lr)}} \|_\infty^2 + \frac{r}{n \wedge p} + \frac{\| \bE^{\text{(lr)}} \|^2_{2, \infty}}{n} \right) \log^5(np) 
\,+\,  \frac{20}{n} \| \phi \|_2^2 \\
& \le  
 \frac{C' \|  \beta^*\|_1^2}{\rho^4} \left( \ r \| \bE^{\text{(lr)}} \|_\infty^2 + \frac{r}{n \wedge p}  \right) \log^5(np) 
\,+\,  \frac{20}{n} \| \phi \|_2^2 \\
& \le  
 \frac{C'  \|  \beta^*\|_1^2}{\rho^4} \left( C(\zeta, K) \delta^{-K} \cL^{2} \cdot \delta^{2\zeta} + \frac{C(\zeta, K) \delta^{-K}}{n \wedge p}  \right) \log^5(np) 
\,+\,  \frac{20}{n} \| \phi \|_2^2 \\
& \le  
 \frac{C' C(\zeta, K) \cL^{2} \|  \beta^*\|_1^2}{\rho^4} \left( \delta^{-K}  \cdot \delta^{2\zeta} + \frac{ \delta^{-K}}{n \wedge p}  \right) \log^5(np) 
\,+\,  \frac{20}{n} \| \phi \|_2^2 
\end{align} 
Substituting $\delta = (1 / (n \wedge p))^{{1 / 2\zeta}}$ completes the proof.
%
%
\end{proof}


%

\section{Proof of Proposition \ref{prop:low_rank_sparsity}} \label{sec:low_rank_prop}
\begin{proof}
We have
\begin{align*}
M &= \sum_{i=1}^p v_i \bX_{\cdot, i}  
\\ &= \sum_{i=1}^k v_i \bX_{\cdot, i} + \sum_{j=k+1}^p v_j \bX_{\cdot, j}
\\ &= \sum_{i=1}^k v_i \bX_{\cdot, i} + \sum_{j=k+1}^p v_j \Big( \sum_{i=1}^k c_i(j) \bX_{\cdot, i} \Big)
\\ &= \sum_{i=1}^k v_i \bX_{\cdot, i} + \sum_{i=1}^k \bX_{\cdot, i} \Big( \sum_{j=k+1}^p c_i(j) v_j \Big)
\\ &= \sum_{i=1}^k \Big(v_i + \sum_{j=k+1}^p c_i(j) v_j \Big) \bX_{\cdot, i}. 
\end{align*}
Define $v^*_i = v_i + \sum_{j=k+1}^p c_i(j) v_j$ for $i \in [k]$ and $0$ for $i \notin [k]$. 
Then $\| v^* \|_{0} \le k$.
Further, 
\begin{align}
\| v^* \|_{1} 
= \sum_{i=1}^k  \left| \Big(v_i + \sum_{j=k+1}^p c_i(j) v_j \Big) \right|
\le    C'' \sum_{i=1}^k  \Big(|v_i| + \sum_{j=k+1}^p | v_j | \Big) 
\le   C'' k \| v \|_{1}
\end{align}
\end{proof}

\section{Proof of Theorem \ref{thm:test_pcr}}\label{sec:mse_test_hsvt}

The proof of Theorem \ref{thm:test_pcr} follows the standard approach in terms of establishing 
generalization error bounds using Rademacher complexity (see \cite{Bartlett_2003} and references therein). 
We note two important contributions: (1) relating our notion of generalization error to the standard definitions; 
(2) arguing that the Rademacher complexity of our matrix estimation regression algorithm (using HSVT) 
can be identified with the Rademacher complexity of regression with $\ell_0$-regularization. 


\subsection{Background}
\paragraph{Notation, Setup.} We consider PCR with parameter $k$ for some $k \geq 1$. 
Recall that the training sample set $\Omega \subset [N]$, with $|\Omega| = n$, is sampled uniformly at random and without replacement from $[N]$. 
Further, as argued in Proposition \ref{prop:equivalence}, PCR with parameter $k$ is equivalent to Linear Regression with pre-processing of noisy covariates using HSVT. 
Hence, we let $\bhA = \bbZ^{\text{HSVT}, k}$ and $\hbeta =  \beta^{\text{HSVT}, k}$. 

\paragraph{Generalization error and Rademacher complexity.} 

\noindent We measure the quality of our estimates through the following two quantities of error. For any hypothesis $\beta \in \mathbb{R}^p$ and training set $\Omega$, the empirical error is 
\begin{align} \label{eq:train_error.2}
	\hcE_{\Omega}(\beta) &= \frac1n\sum_{\omega \in \Omega} \Big( \bhA_{\omega, \cdot} \beta - \bA_{\omega, \cdot} \beta^* \Big)^2. 
\end{align}
Similarly, we define the overall error as
\begin{align} \label{eq:test_error.2}
	\cE(\beta) &= \frac{1}{N} \sum_{i=1}^N \Big( \bhA_{i, \cdot} \beta - \bA_{i, \cdot} \beta^* \Big)^2.  
\end{align}

\noindent 
For any linear hypothesis class $\F \subset \Rb^{p}$, define the generalization error as the supremum of the gap between \eqref{eq:train_error.2} and \eqref{eq:test_error.2} over $\F$. Precisely, for a given training set $\Omega$, 
\begin{align} \label{eq:generalization_error_general}
	\phi(\Omega) = \sup_{\beta \in \F} \left( \cE(\beta) -  \hcE_{\Omega}(\beta) \right). 
\end{align}

\noindent Next, we define the notion of Rademacher complexity, wich has been very effective to bound the generalization error. To begin with, the Rademacher complexity of a set $A \subset \Reals^n$ is defined as 
\begin{align}
R(A) & = \Ex_\sigma \left[\sup_{a \in A} \frac1n \sum_{i=1}^n \sigma_i a_i\right],
\end{align}
where $\sigma_1,\dots, \sigma_n$ are i.i.d. Rademacher variables, which are uniformly distributed on $\{-1, 1\}$, and the 
expectation above is taken with respect to their randomness. This has been naturally extended for the setting of prediction 
problems as follows: given a collection of real-valued response variables and covariates, say $(Y_i, X_i), ~i \in [n]$, a
collection of real-valued functions or hypotheses $\cG$ that map covariates to real values, and loss function $L: \Reals^2 \to [0, \infty)$ that measures the error or loss in prediction for a given function, define 
\begin{align}
R_S(\cG) & = \Ex_\sigma \left[\sup_{g \in \cG} \frac1n \sum_{i=1}^n \sigma_i g(X_i)\right],  
~~R_S(L\circ\cG) ~ = \Ex_\sigma \left[\sup_{g \in \cG} \frac1n \sum_{i=1}^n \sigma_i L(Y_i, g(X_i))\right].
\end{align}
In our setting, the covariates that the predictor uses are the denoised rows of $\bhA$, denoted as$\{\bhA_{1, \cdot}, \dots, \bhA_{N, \cdot}\}$. 
The loss function of interest is the quadratic function: $\ell(y, y') = (y-y')^2$. The ideal response variable of our interest
is $\bA_{i, \cdot} \beta^*$ for $i \in [N]$. Given that our algorithm observes (noisy) response variables in the index set
$\Omega$, we shall use the sample set $\{ (\bA_{\omega, \cdot} \beta^*, \bhA_{\omega, \cdot}): \omega \in \Omega\}$.

\noindent It turns out that the appropriate adaptation of the Rademacher complexity for our setting is as follows: 
Let $\D$ denote the distribution of the observations $Z_{ij}$ (i.e., the randomness in the measurements). 
Hence, $\bhA$ is a random matrix as it derived from $\bZ$. 
Then,
\begin{align}\label{eq:Rademacher}
	R_n(\F) &=  \Ex_{\sigma, \Omega | \D} \Bigg[ \sup_{\beta \in \F} \Bigg( \frac{1}{n} \sum_{\omega \in \Omega} \sigma_\omega 
	\bhA_{\omega, \cdot} \beta  \Bigg) \Bigg] \\ ~~
	R_n(\ell\circ \F) &= \Ex_{\sigma, \Omega | \D} \Bigg[ \sup_{\beta \in \F} \Bigg( \frac{1}{n} 
	\sum_{\omega \in \Omega} \sigma_\omega \ell(\bA_{\omega, \cdot} \beta^*, \bhA_{\omega, \cdot} \beta)  \Bigg) \Bigg],
\end{align}
where $\Ex_\Omega$ is taken with respect to selecting $\Omega \subset [N]$ uniformly at random from $[N]$ without replacement (with $| \Omega | = n$). 

\paragraph{Rademacher Class - Sparse Linear Models.}
Define 
$\F_{(a, b)} \subset \Rb^{p}$ for $a \in \Nb, b \in \Rb$ as $$\F_{(a, b)} \coloneqq \{\beta \in \Rb^{p} :  \| \beta \|_0 \le a, \| \beta \|_1 \le b \}.$$

\noindent 
We denote $\F_{(\cdot, b)}$ as the case where there is no restriction on $a$, i.e., $\beta \in \F_{(\cdot, b)}$ has no constraint in its $\| \cdot \|_0$-norm.
We then have the following proposition,
\begin{proposition}\label{prop:rademacher_containment}
Assume $\bhA$ satisfies \eqref{eq:test_error_condition} in Proposition \ref{prop:low_rank_sparsity}.
Then, 
$$R_n(\F_{(\cdot, \ \| \hbeta \|_1)}) \le R_n(\F_{(k, \ C'' k  \|\hbeta\|_1)}),  \quad R_n(\ell \circ \F_{(\cdot, \ \|\hbeta\|_1}) \le R_n(\ell \circ \F_{(k, \ C'' k  \|\hbeta\|_1)})$$
where $C''$ is defined as in Proposition \ref{prop:low_rank_sparsity}.
\end{proposition}

\begin{proof}
By definition, $\bhA$ has rank $k$. 
Then by Proposition \ref{prop:low_rank_sparsity} for $ \hbeta $, there exists an $k$-sparse vector $\beta' \in \Reals^p$ such that
\begin{align}\label{eq:sparse_eq_PCR}
 \bhA \cdot \hbeta = \bhA \cdot \beta', \ \text{s.t.} \ \| \beta' \|_1 \le C'' \| \hbeta \|_1.
\end{align}
Observe that due to the equality, we have,
\begin{align}
\hcE_{\Omega}(\hbeta) &= \hcE_{\Omega}(\beta^') \\
 \cE_{\Omega}(\hbeta) &= \cE_{\Omega}(\beta^').
\end{align}
\noindent Appealing to the definitions of $R_n(\cdot)$ and $R_n(\ell \circ \cdot)$ completes the proof.
\end{proof}
\noindent For the remainder of Section \ref{sec:mse_test_hsvt}, we define $B \coloneqq C'' \cdot k \cdot \| \hbeta \|_1$ and overload notation and define $\F \coloneqq \F_{(k, B)}$.

\subsection{Helper Lemmas \ref{lemma:rademacher_bound} and \ref{lemma:rademacher_final} to Prove Theorem \ref{thm:test_pcr}} 

\subsubsection{Lemma \ref{lemma:rademacher_bound}}
\begin{lemma} \label{lemma:rademacher_bound}
	Let $\phi(\Omega)$ be defined as in \eqref{eq:generalization_error_general}. Let $\Omega$ be random 
	subset of $[N]$ of size $n$ that is chosen uniformly at random without replacement. Then, 
	\begin{align*} 
		\Ex_{\Omega | \D} \left[ \phi(\Omega) \right] &\le  2 R_n(\ell \circ \F). 
	\end{align*} 
\end{lemma}

\begin{proof}
Let $\Omega = \{i_1, \dots, i_n\}$. Further, let $\Omega' = \{i'_1, \dots, i'_n\}$ be a ``ghost sample'', i.e., $\Omega'$ is an independent set of $n$ locations sampled uniformly at random and without replacement from $[N]$. 
Thus, 
\begin{align*} 
	\Ex_{\Omega | \D} [ \phi(\Omega) ]&= \Ex_{\Omega | \D} \left[ \sup_{\beta \in \F} \left(\cE(\beta) - \hcE_{\Omega}(\beta) \right) \right] 
	\\ &= \Ex_{\Omega | \D} \left[ \sup_{\beta \in \F} \left( \Ex_{\Omega'} \left[ \hcE_{\Omega'} (\beta)  -\hcE_{\Omega}(\beta) \right] \right) \right] 
	\\ &\le \Ex_{\Omega, \Omega' | \D} \left[ \sup_{\beta \in \F} \left( \hcE_{\Omega'} (\beta)  -\hcE_{\Omega}(\beta)  \right) \right] 
	\\ &= \Ex_{\Omega, \Omega' | \D} \left[ \sup_{\beta \in \F} \frac{1}{n} \sum_{k=1}^n \left(  \ell(\bA_{i'_k} \beta^*; \bhA_{i'_k} \beta) - \ell( \bA_{i_k} \beta^*; \bhA_{i_k} \beta) \right) \right], 
\end{align*}
where the inequality follows by the convexity of the supremum function and Jensen's Inequality. 

To proceed, we will use the ghost sampling technique. Recall that the entries of $\Omega$ and $\Omega'$ were drawn uniformly at random from $[N]$. As a result, 
$\ell(\bA_{i'_k} \beta^*; \bhA_{i'_k} \beta) - \ell( \bA_{i_k} \beta^*; \bhA_{i_k} \beta)$ and 
$\ell( \bA_{i_k} \beta^*; \bhA_{i_k} \beta) - \ell(\bA_{i'_k} \beta^*; \bhA_{i'_k} \beta)$
have the same distribution. Further, since $\sigma_k$ takes value $1$ and $-1$ with equal probability, we have
\begin{align*}
& \Ex_{\Omega, \Omega' | \D} \left[ \sup_{\beta \in \F} \frac{1}{n} \sum_{k=1}^n \left( \ell(\bA_{i'_k} \beta^*; \bhA_{i'_k} \beta) - \ell( \bA_{i_k} \beta^*; \bhA_{i_k} \beta) \right) \right]  \\
& \qquad = \Ex_{\sigma, \Omega, \Omega' | \D} \left[  \sup_{\beta \in \F}  \frac{1}{n}  \sum_{k=1}^n \sigma_k \left(\ell(\bA_{i'_k} \beta^*; \bhA_{i'_k} \beta) - \ell( \bA_{i_k} \beta^*; \bhA_{i_k} \beta))  \right)  \right]. 
\end{align*}

Combining the above relation with the fact that the supremum of a sum is bounded above by the sum of supremums, we obtain
\begin{align*}
	\Ex_{\Omega | \D} [\phi(\Omega)] &\le  \Ex_{\sigma, \Omega, \Omega' | \D} \left[  \sup_{\beta \in \F}  \frac{1}{n}  \sum_{k=1}^n \sigma_k \left( \ell(\bA_{i'_k} \beta^*; \bhA_{i'_k} \beta) - \ell( \bA_{i_k} \beta^*; \bhA_{i_k} \beta)  \right)  \right]
	\\ &\le \Ex_{\sigma, \Omega, \Omega' | \D} \left[ \sup_{\beta \in \F}  \frac{1}{n} \sum_{k=1}^n \sigma_k \ell(\bA_{i'_k} \beta^*; \bhA_{i'_k} \beta)  + \sup_{\beta \in \F} \frac{1}{n} \sum_{k=1}^n -\sigma_k \ell( \bA_{i_k} \beta^*; \bhA_{i_k} \beta)   \right] 
	\\ &= \Ex_{\sigma, \Omega | \D} \left[ \sup_{\beta \in \F}  \frac{1}{n} \sum_{k=1}^n \sigma_k \ell( \bA_{i_k} \beta^*; \bhA_{i_k} \beta)  \right] + \Ex_{\sigma, \Omega' | \D} \left[ \sup_{\beta \in \F}  \frac{1}{n} \sum_{k=1}^n \sigma_k \ell(\bA_{i'_k} \beta^*; \bhA_{i'_k} \beta)  \right]
	\\&= 2 \cdot R_n(\ell \circ \F),
\end{align*}
where the second to last equality holds because $\sigma_k$ is a symmetric random variable.

\end{proof}

\subsubsection{Lemma \ref{lemma:rademacher_final}}

\noindent To prove Lemma \ref{lemma:rademacher_final}, we first prove a series of helper lemmas. 

\begin{lemma} \label{lemma:worst_error}
	Let Property \ref{prop:bounded_covariates} hold. Then,
	for any $\beta \in \F$, 
	\begin{align*}
		\max_{i \in [N]} \ell ( \bA_{i, \cdot} \beta^*, \bhA_{i, \cdot} \beta ) &\le C(\bhA). 
	\end{align*}
	Here, $C(\bhA) = 2 \left[ (B \cdot \| \bhA \|_{\infty})^2 + (\| \beta^* \|_1)^2 \right]$.
\end{lemma}

\begin{proof}
	Observe that for any $i \in [N]$ and $\beta \in \F$, 
	\begin{align*}
		\ell(\bA_{i, \cdot} \beta^*, \bhA_{i, \cdot} \beta) &= ( \bhA_{i, \cdot} \beta - \bA_{i, \cdot} \beta^*)^2 \le 2(\bhA_{i, \cdot} \beta)^2 + 2(\bA_{i, \cdot} \beta^*)^2. 
	\end{align*} 
	Recall that every candidate vector $\beta \in \F$ has the following propery: $\|\beta\|_1 \le B$. Hence, it follows that for any $i \in [N]$, 
	\begin{align} \label{eq:max_error_bound.term1}
		|\bhA_{i, \cdot} \hbeta| & \leq   \| \beta \|_1 \cdot  \max_{j\in [p]} | \widehat{A}_{ij}| ~\le  B \cdot \| \bhA\|_{\infty}. 
	\end{align}
	Further, By Property \ref{prop:bounded_covariates} and Holder's inequality, we have for any $i \in [N]$, 
\begin{align} \label{eq:max_error_bound.term2}
		\abs{\bA_{i, \cdot} \beta^*} \le  \norm{\bA_{i, \cdot}}_{\infty} \, \norm{\beta^*}_1 \le  \norm{\beta^*}_1. 
\end{align}
	The desired result then follows from an immediate application of the above results. 
\end{proof}

\begin{lemma} \label{lemma:rademacher_complexity_linear_functions_sparse}
	Recall $\emph{rank}(\bhA) = k$. Then, 
	\begin{align}
		R_n(\F) &\le \frac{\sqrt{k} B}{\sqrt{n}} \cdot \| \bhA \|_{\infty}. 
	\end{align}
\end{lemma}

\begin{proof}
Let $I_{\beta} = \{i \in [p]: \beta_i \neq 0\}$ denote the index set for the nonzero elements of $\beta \in \F$; recall that $| I_\beta| \le k$ by the definition of $\F$. For any vector $v \in \Reals^p$, we denote $v_{I_{\beta}}$ as the vector that retains only its values in $I_{\beta}$ and takes the value $0$ otherwise. Then,
	\begin{align*}
		R_n(\F) &= \Ex_{\sigma, \Omega | \D} \Bigg[ \sup_{\beta \in \F} \Bigg( \frac{1}{n} \sum_{i=1}^n \sigma_i  \bhA_{i, \cdot}  \beta  \Bigg) \Bigg]
		\\ &= \frac{1}{n}  \Ex_{\sigma, \Omega | \D} \Bigg[ \sup_{\beta \in \F} \Bigg( \sum_{j \in I_{\beta}} \beta_j \Big( \sum_{i=1}^n \sigma_i \bhA_{i, \cdot} \Big)_j \Bigg) \Bigg]
		\\ &\stackrel{(a)} \le \frac{1}{n}  \Ex_{\sigma, \Omega | \D} \Bigg[ \sup_{\beta \in \F} \| \beta \|_2 \,\cdot \, \Big\| \Big(\sum_{i=1}^n \sigma_i  \bhA_{i, \cdot} \Big)_{I_{\beta}} \Big\|_2 \Bigg]
		\\ &\stackrel{(b)} \le \frac{ B}{n}  \Ex_{\sigma, \Omega | \D} \Bigg[  \Big\| \Big(\sum_{i=1}^n \sigma_i  \bhA_{i, \cdot} \Big)_{I_{\beta}} \Big\|_2 \Bigg]
		\\ &\stackrel{(c)} \le \frac{ B}{n} \left( \Ex_{\sigma, \Omega | \D} \Bigg[ \Big(\sum_{i=1}^n \sigma_i \bhA_{i, \cdot} \Big)_{I_{\beta}} \Big(\sum_{k=1}^n \sigma_k \bhA_{k, \cdot} \Big)_{I_{\beta}}^T \Bigg] \right)^{1/2}
		\\ &= \frac{ B}{n} \left( \Ex_{\Omega | \D} \Bigg[ \sum_{i=1}^n \Big\| (\bhA_{i, \cdot} )_{I_{\beta}} \Big\|_2^2 \Bigg] \right)^{1/2}
		\\ &\le \frac{B}{n} \left(n k \max_{i \in [n]} \norm{(\bhA_{i, \cdot} )_{I_{\beta}}}_{\infty}^2 \right)^{1/2}
		\\ &= \frac{\sqrt{k} B}{\sqrt{n}}  \cdot \| \bhA \|_{\infty}. 
	\end{align*}
Note that (a) makes use of the Cauchy-Schwartz Inequality, (b) follows from the boundedness assumption of the elements in $\F$ and noting the $\ell_2$-norm of a vector is less than the $\ell_1$-norm, and (c) applies Jensen's Inequality. 
\end{proof}

\begin{lemma} \label{lemma:rademacher_composition} {\bf Lipschitz composition of Rademacher averages. (\cite{lipschitz_rademacher})}  \\
Suppose $\{\phi_i\}, \{\psi_i\}$, $i=1, \dots, n$, are two sets of functions on $\Theta$ such that for each $i$ and $\theta, \theta' \in \Theta$, $\abs{ \phi_i(\theta) - \phi_i(\theta')} \le \abs{\psi_i(\theta) - \psi_i(\theta')}$. Then, for all functions $c: \Theta \rightarrow \Reals$, 
\begin{align*}
	\Ex \left[ \sup_{\theta \in \Theta} \left\{ c(\theta) + \sum_{i=1}^n \sigma_i \phi_i(\theta) \right\} \right] &\le \Ex \left[ \sup_{\theta \in \Theta} \left\{ c(\theta) + \sum_{i=1}^n \sigma_i \psi_i(\theta) \right\} \right],
\end{align*}
where $\sigma_i$ are Rademacher random variables. 
\end{lemma} 

\begin{proof}
The proof can be found in \cite{lipschitz_rademacher}. 
\end{proof}

\begin{lemma} \label{lemma:rademacher_final}
	Let Property \ref{prop:bounded_covariates} hold and recall $\emph{rank}(\bhA) = k$. Then, 
	\begin{align*}
		R_n(\ell\circ \F) \le C \frac{\sqrt{k} B^{2}}{\sqrt{n}} \cdot \| \bhA \|^{2}_{\infty} \cdot \| \beta^* \|_1,
	\end{align*}
	where  $C > 0$ is an absolute constant. 
\end{lemma}

\begin{proof}
	Using Lemma \ref{lemma:worst_error}, we have for any $\beta \in \F$, 
\begin{align*}
	\max_{i \in [N]} \, \, | \ell' (\bA_{i, \cdot} \beta^*, \bhA_{i, \cdot} \beta) | &\le 2 \sqrt{C(\bhA)},
\end{align*}
where $\ell'(\cdot, \cdot)$ denotes the derivative of the loss function with respect to our estimate. Since our loss function of interest has bounded first derivative, the Lipschitz constant of $\ell(\cdot, \cdot)$ is bounded by $2 C(\bhA)^{1/2}$; hence, applying Lemma \ref{lemma:rademacher_composition} for Lipschitz functions and using Lemma 
\ref{lemma:rademacher_complexity_linear_functions_sparse} yields the following inequality: 
	\begin{align*}
		R_n(\ell\circ \F) &\le 2 \sqrt{C(\bhA) } \cdot R_n (\F)
		\le  C \frac{\sqrt{k} B^{2}}{\sqrt{n}} \cdot \| \bhA \|^{2}_{\infty} \cdot \| \beta^* \|_1,
	\end{align*}
	for some absolute constant, $C > 0$.
	This concludes the proof. 
\end{proof}

\subsubsection{Proof of Theorem \ref{thm:test_pcr}}
Now we are ready to complete the proof of Theorem \ref{thm:test_pcr}. 
\begin{proof}[Proof of Theorem \ref{thm:test_pcr}]
The testing error, for PCR with parameter $k$ or, equivalently, Linear Regression with
covariate pre-processing via HSVT thresholded at the $k$-th singular value, is 
\begin{align}\label{eq:gen0.a}	
\text{MSE}(\widehat{Y}) &= \frac{1}{N} \Ex_{\D | \Omega} \left[ \sum_{i=1}^N \Big( \widehat{Y}_{i} - \bA_{i, \cdot} \beta^* \Big)^2  \right] 
=\Ex_{\D | \Omega} \left[\cE(\hbeta)  \right],
\end{align} 
where the expectation is taken with respect to the randomness in the data. 

\vspace{2mm}

\noindent And, for a given training set $\Omega$, the training error is 
\begin{align}\label{eq:gen0.b}	
\text{MSE}_{\Omega}(\widehat{Y}) &= \frac{1}{n} \Ex_{\D | \Omega} \left[ \sum_{i \in \Omega} \Big( \widehat{Y}_{i} - \bA_{i, \cdot} \beta^* \Big)^2  \right]=\Ex_{\D | \Omega} \left[\hcE_{\Omega}(\hbeta)  \right].
\end{align}
Recall that we shall consider the training set $\Omega$ being chosen uniformly at random amongst subsets of $[N]$ of size $n$. 
Given any $\Omega$, observe that
\begin{align} \label{eq:gen.1}
	\cE(\hbeta) &\le \hcE_{\Omega}(\hbeta) + \sup_{\beta \in \F} \Big( \cE(\beta) - \hcE_{\Omega}(\beta) \Big) =  \hcE_{\Omega}(\hbeta)  + \phi(\Omega),
\end{align}
where $\phi(\Omega)$ is as defined by \eqref{eq:generalization_error_general}. 
Taking expectations of the above inequality, we obtain
\begin{align}
	\Ex_{\D, \Omega} [\cE(\hbeta)] &\le \Ex_{\D, \Omega} [\hcE(\hbeta)] + \Ex_{\D, \Omega} [\phi(\Omega)] \nonumber
	\\ &= \Ex_\Omega \left[ \Ex_{\D | \Omega} [\hcE(\hbeta)]  \right] + \Ex_\D \left[ ~\Ex_{\Omega | \D} [ \phi(\Omega)] \right] \label{eq:test.1}. 
\end{align}	
We now bound each term on the right-hand side of \eqref{eq:test.1} separately. Beginning with the leftmost term, observe that, by definition, we have
\[ \Ex_\Omega \left[ \Ex_{\D | \Omega} [\hcE(\hbeta) ] \right] = \Ex_\Omega \left[ \text{MSE}_\Omega(\hY) \right]. \] 
Moreover, applying Lemmas \ref{lemma:rademacher_bound} and \ref{lemma:rademacher_final}, we obtain
\begin{align*}
	\Ex_\D \left[ ~\Ex_{\Omega | \D} [ \phi(\Omega)] \right]  &\le 2 \, \Ex_\D \left[  R_n (\ell \circ \F) \right]
	\\ &\le  C \Ex_\D \left[  \frac{\sqrt{k} B^{2}}{\sqrt{n}} \cdot \| \bhA \|^{2}_{\infty} \cdot \| \beta^* \|_1\right]
	\\ &=  C C'' \frac{ k^{5/2}}{\sqrt{n}} \Ex_\D \left[ \| \hbeta \|^{2} \cdot \| \bhA \|^{2}_{\infty} \cdot \right] \cdot \| \beta^* \|_1
\end{align*}
where we recall $B \coloneqq  C'' \cdot k \cdot \| \hbeta \|_1$.
Combining the above results completes the proof.
\end{proof}

\section{Proof of Propositions \ref{prop:gaussian_example}, \ref{prop:geo_decay_finite_sample} and \ref{lemma:dg_decaying_sv}: Examples} 

\subsection{Proof of Proposition \ref{prop:gaussian_example}: Embedded Random Gaussian Features}\label{sec:gaussian_features}
Recall that we let $\bA = \tilde{\bA} \tilde{\bR}$ where $\tilde{\bA} \in \Reals^{N \times r} $ is a random matrix whose entries are independent 
standard normal random variables, i.e., $\tilde{A}_{ij} \sim \mathcal{N}(0,1)$, and $\tilde{\bR} \in \Reals^{r \times p}$ is another random matrix 
with independent entries such that $\tilde{R}_{ij} = 1/\sqrt{r}$ with probability $1/2$ and $\tilde{R}_{ij} = - 1/\sqrt{r}$ with probability $1/2$ in 
Proposition \ref{prop:gaussian_example}. In this subsection, we show that $s_r(\bbA) = \Omega \Big(\sqrt{\frac{Np}{r}} \Big)$ and $\| \bA \|_\infty = 
O\big( \sqrt{ \log(Np) } \big)$ with high probability.

\subsubsection{Helper Lemmas}

\begin{lemma}\label{lem:quasi_isometry}
	Suppose that $r \leq \frac{ \sqrt{p} }{ 4\sqrt{2\log p}} + 1$ and let $\bR \in \Reals^{r \times p}$ be a random matrix with independent entries 
	such that $\bR_{ij} = \frac{1}{\sqrt{p}}$ with probability $\frac{1}{2}$ and $\bR_{ij} = - \frac{1}{\sqrt{p}}$ with probability $\frac{1}{2}$. 
	With probability at least $1 - \frac{1}{p^2}$, for all $v \in \Reals^r$,
	\[	\frac{1}{2} \| v \|_2^2 \leq \| \bR^T v \|_2^2 \leq \frac{3}{2} \| v \|_2^2.	\]
\end{lemma}

\begin{proof}
	For $i \in [r]$, let $\bR_i$ denote the $i$-th row of $\bR$. Observe that $\| \bR_i \|_2 = 1$ for all $i \in [r]$. Also, note that for $i \neq j \in [r]$,
	$\langle \bR_i, \bR_j \rangle = \frac{1}{p} \sum_{k=1}^p \tilde{\bR}_{ik} \tilde{\bR}_{jk}$ is a sum of $p$ independent binary random variables;
	$\tilde{\bR}_{ik} \tilde{\bR}_{jk} = 1$ with probability $\frac{1}{2}$ and $-1$ with probability $\frac{1}{2}$. Therefore, $\mathbb{E} \langle \bR_i, \bR_j \rangle = 0$. 
	By Hoeffding's inequality for bounded random variables, 
	\[	\Prob{ | \langle \bR_i, \bR_j \rangle | > t } \leq 2 \exp \left( - \frac{p t^2}{2} \right).	\]
	Letting $t = \frac{ 2 \sqrt{2 \log p } }{\sqrt{p} }$, we can conclude that for any pair of $i \neq j \in [r]$, $ | \langle \bR_i, \bR_j \rangle |  
	\leq \frac{ 2 \sqrt{2 \log p } }{\sqrt{p} }$ with probability at least $1 - \frac{2}{p^4}$. There are ${r \choose 2} \leq \frac{r^2}{2}$ such pairs and $r \leq p$. 
	Thus, applying the union bound, we know that $ | \langle \bR_i, \bR_j \rangle |  \leq \frac{ 2 \sqrt{2 \log p } }{\sqrt{p} }$ for all pairs $i \neq j$ 
	with probability at least $1 - \frac{1}{p^2}$.
	
	Now we observe that 
	\begin{align*}
		\| \bR^T v \|_2^2	&= \left\langle \sum_{i=1}^r v_i \bR_i, \sum_{i=1}^r v_i \bR_i, \right\rangle\\
			&= \sum_{i=1}^r v_i^2 \| \bR_i \|_2^2 + \sum_{i=1}^r \sum_{j \neq i} v_i v_j \langle \bR_i, \bR_j \rangle\\
			&\leq \sum_{i=1}^r v_i^2 \| \bR_i \|_2^2 + \sum_{i=1}^r \sum_{j \neq i } | v_i v_j | | \langle \bR_i, \bR_j \rangle |.
	\end{align*}
	With probability at least $1 - \frac{1}{p^2}$,
	\begin{align*}
		\| \bR^T v \|_2^2	&\leq   \sum_{i=1}^r v_i^2 \| \bR_i \|_2^2 + \sum_{i=1}^r \sum_{j \neq i } | v_i v_j | \frac{2\sqrt{2 \log p}}{\sqrt{p}}\\
			&\stackrel{(a)}{\leq}  \sum_{i=1}^r v_i^2 + (r-1)  \sum_{i=1}^r v_i^2 \frac{2\sqrt{2 \log p}}{\sqrt{p}}\\
			&\leq \| v \|_2^2  \bigg( 1 + \frac{2(r-1)\sqrt{2 \log p}}{\sqrt{p}} \bigg)
	\end{align*}
	where (a) follows from that $\| \bR_i \|_2^2 = 1$ for all $i \in [r]$ and the Cauchy-Schwarz inequality ($2|v_i v_j| \leq v_i^2 + v_j^2$). 
	By the same argument, $\| \bR^T v \|_2^2 \geq \| v \|_2^2  \Big( 1 - \frac{2(r-1)\sqrt{2 \log p}}{\sqrt{p}} \Big)$. 
	
	Lastly, we note that $\frac{2(r-1)\sqrt{2 \log p}}{\sqrt{p}} \leq \frac{1}{2}$ if and only if $r \leq \frac{\sqrt{p}}{4\sqrt{2 \log p}} + 1$ to complete the proof.
\end{proof}

\begin{remark}\label{rem:quasi_iso}
	Lemma \ref{lem:quasi_isometry} implies that given $r \leq 1 + \frac{\sqrt{p}}{4\sqrt{2 \log p }}$, the right multiplication of $\bR$ defines 
	a quasi-isometric embedding from $\Reals^r$ to $\Reals^p$ with high probability. More precisely, with probability at least $1 - \frac{1}{p^2}$, 
	the following inequalities are true:
	\begin{align*}
		\frac{1}{2} \| v \|_2^2 \leq \|  \bR^T v \|_2^2 \leq \frac{3}{2} \| v \|_2^2,	\quad \forall v \in \Reals^r,\qquad \text{and} \qquad
		\frac{1}{2} \| w \|_2^2 \leq \| \bR w \|_2^2 \leq \frac{3}{2} \| w \|_2^2,	\quad \forall w \in \text{rowspace}(\bR).
	\end{align*}
	The first inequality is just the conclusion of Lemma \ref{lem:quasi_isometry}; it implies that $\frac{1}{2} \leq \lambda_i( \bR \bR^T ) \leq \frac{3}{2}$ 
	for all $i \in [r]$ where $\lambda_i( \bR \bR^T)$ denotes the $i$-th largest eigenvalue of $\bR \bR^T$. Let $v_i$ be an eigenvector corresponding to 
	$\lambda_i( \bR \bR^T)$; $\{ v_1, \ldots, v_r \}$ forms an orthonormal basis of $\mathbb{R}^r$. 
	
	To see why the second inequality also holds, suppose that $w = \bR^T v_w$ for some $v_w \in \mathbb{R}^r$ (such a $v_w$ exists because $w \in \bR^T$). 
	Observe that $\| w \|_2^2 = w^T w = v_w^T \bR \bR^T v_w$ and that $\| \bR w \|_2^2 = w^T \bR^T \bR w = v_w^T \bR \bR^T \bR \bR^T v_w$. 
	We may write $v_w = \sum_{i=1}^r c_i v_i$ for some $c_i \in \mathbb{R}$. It follows that $\| w \|_2^2 = \sum_{i=1}^r c_i^2 \lambda_i( \bR \bR^T )$ and 
	$\| \bR w \|_2^2 = \sum_{i=1}^r c_i^2 \lambda_i^2( \bR \bR^T )$; therefore, $ \frac{1}{2} \leq \lambda_r( \bR \bR^T ) \leq \frac{\| \bR w \|_2^2}{ \| w \|_2^2 } 
	\leq \lambda_1( \bR \bR^T ) \leq \frac{3}{2}$.
\end{remark}

\begin{remark}\label{rem:sr_A}
	By Remark \ref{rem:quasi_iso}, with probability at least $1 - \frac{1}{p^2}$, 
	\begin{align*}
		s_r( \tilde{\bbA}\bR)
			&= \sup_{W \subset \Reals^p\atop\dim W = r} \inf_{w \in W} \frac{\| \tilde{\bbA} \bR w \|_2}{ \|w \|_2}
			= \inf_{w \in \text{\normalfont rowspace}{\bR}}\frac{\| \tilde{\bbA} \bR w \|_2}{ \|w \|_2}\\
			&\geq \sqrt{\frac{1}{2}} \inf_{w \in \text{\normalfont rowspace}{\bR}}\frac{\| \tilde{\bbA} \bR w \|_2}{ \| \bR w \|_2}
			= \sqrt{\frac{1}{2}} \inf_{v \in \Reals^r} \frac{\| \tilde{\bbA} v \|_2}{ \| v \|_2}
			= \sqrt{\frac{1}{2}} s_r(\tilde{\bbA}).
	\end{align*}
\end{remark}

\begin{lemma}[Spectral properties of $\tilde{\bbA}$]\label{lem:balance.1}
	Let $\tilde{\bbA} \in \Reals^{N \times r}$ be a random matrix whose entries are i.i.d. standard Gaussian random variable. 
	Then, 
	
	(1) with probability at least $1 - {2 \exp(-\frac{1}{2} \sqrt{Nr}) }$, $\rank(\tilde{\bbA}) = r$ and
	\[	\frac{s_1(\tilde{\bbA})}{s_r (\tilde{\bbA})} \leq { \frac{ 1 + (r/N)^{1/4} + (r/N)^{1/2} }{ 1 - (r/N)^{1/4} - (r/N)^{1/2} }; }	\]
	
	(2) with probability at least $1-\exp \left( -\frac{Nr}{8} \right)$, 
	\[	 \| \tilde{\bbA} \|_F^2  > \frac{Nr}{2}.	\]
\end{lemma}

\begin{proof}
	{\bf Proof of Claim 1}
	By \cite[Corollary 5.35]{vershynin2010introduction}, for any $t \geq 0$, we have
	\[	\sqrt{N} - \sqrt{r} - t \leq s_{\textrm{min}}(\tilde{\bbA}) \leq s_{\textrm{max}}(\tilde{\bbA}) \leq \sqrt{N} + \sqrt{r} + t, 	\]
	with probability at least $1 - 2 \exp(-t^2/2)$. { Choosing $t = (Nr)^{1/4}$ concludes the proof.} \\
	{\bf Proof of Claim 2}
	Observe that $\| \tilde{\bbA} \|_F^2 = \sum_{i,j} \tilde{\bbA}_{ij}^2$. We can easily observe that $\mathbb{E}\| \tilde{\bbA} \|_F^2 = Nr$. 
	By Bernstein's inequality, it follows that for every $t \geq 0$,
	\[	\mathbb{P}\{  \| \tilde{\bbA} \|_F^2  - \mathbb{E}  \| \tilde{\bbA} \|_F^2 \leq -t \}	\leq  \exp \left( -\frac{1}{2} \min\left\{ \frac{t^2}{Nr}, t \right\} \right).	\]
	With $t = \frac{Nr}{2}$, we have
	\[	\mathbb{P}\{  \| \tilde{\bbA} \|_F^2  \leq \frac{Nr}{2}\}	\leq  \exp \left( -\frac{Nr}{8} \right).	\]
\end{proof}

\begin{remark}\label{rem:sr_tildeA}
	Lemma \ref{lem:balance.1} implies that with probability at least $1 - 2 \exp(-2 \sqrt{Nr}) - \exp \left( -\frac{Nr}{8} \right)$, 
	\[	s_r(\tilde{\bbA})^2 \geq \left[ 1 + (r-1) \frac{s_1(\tilde{\bbA})^2}{s_r(\tilde{\bbA})^2} \right]^{-1} \| \tilde{\bbA} \|_F^2	
				\geq \left[ 1 + (r-1) { \bigg( \frac{ 1 + (r/N)^{1/4} + (r/N)^{1/2} }{ 1 - (r/N)^{1/4} - (r/N)^{1/2} } \bigg)^2 } \right]^{-1} \frac{Nr}{2}.	\]
\end{remark}

\begin{lemma}[Structural properties of $\bbA$]\label{lem:balance.2}
	Let $\bbA \in \Reals^{N \times p}$ be a matrix generated as above. With probability at least $1 - \frac{2}{N^2p}$, 
	\[	\max_{i,j} |A_{ij}| \leq  4 \sqrt{ \log(Np)}. \]
\end{lemma}

\begin{proof}
	By construction, $A_{ij} = \sum_{k=1}^r \tilde{A}_{ik} \tilde{R}_{kj}$ and $A_{ij} | \tilde{\bR} \sim \mathcal{N}(0, \sum_{k=1}^r \tilde{R}_{kj}^2 )$ 
	conditioned on $\tilde{\bR}$ and note $\sum_{k=1}^r \tilde{R}_{kj}^2 = 1$ regardless of $\tilde{\bR}$. 
	Therefore, for each fixed $j \in [p]$, $A_{\cdot j} | \tilde{\bR}  \sim \mathcal{N}(0, I_N)$. 
	Observe that $\max_{i} |A_{ij} |  \Big| \tilde{\bR} $ is the maximum absolute value of 
	$N$ i.i.d. standard Gaussians and $\mathbb{E}\Big[ \max_{i} |A_{ij} | \Big| \tilde{\bR}\Big] \leq 2 \sqrt{\log N }$. 
	Since this holds regardless of $\tilde{\bR}$, by tower law we can remove the conditioning on $\tilde{\bR}$.
	In addition, by the concentration of Lipschitz function 
	(note that $\max: \mathbb{R}^N \to \mathbb{R}$ is $1$-Lipschitz),
	\[	\Prob{ | \max_{i} |A_{ij} | - \mathbb{E}[ \max_{i} |A_{ij} | ] | \geq t} \leq 2 \exp\Big( - \frac{t^2}{2 } \Big).	\] 
	Letting $t = 2 \sqrt{\log(Np)}$, it follows for each $j \in [p]$ that $\Prob{ | \max_{i} |A_{ij} | \geq 4 \sqrt{\log (Np)} } \leq \frac{2}{N^2p^2}$. 
	Taking union bound over $j \in [p]$, we conclude that with probability at least $1 - \frac{2}{N^2p}$,
	\[	\max_{i,j} |A_{ij}| \leq 4 \sqrt{ \log(Np) }.	\] 
	
\end{proof}

\subsubsection{Completing the Proof of Proposition \ref{prop:gaussian_example}}

\begin{proof}[Proof of Proposition \ref{prop:gaussian_example}]
Observe that $\bA = \tilde{\bA} \tilde{\bR} = \sqrt{\frac{p}{r}} \tilde{\bA} \bR $. By Lemmas \ref{lem:quasi_isometry}, \ref{lem:balance.1} 
(along with Remarks \ref{rem:sr_A} and \ref{rem:sr_tildeA}), we have 
{
\begin{align*}
	s_r(\bbA) 
		&= s_r( \tilde{\bbA}\tilde{\bR}) =		\sqrt{\frac{p}{r}} s_r( \tilde{\bbA}\bR)
		\geq \sqrt{\frac{Np}{4r}} \left[ \Delta + \frac{1}{r} \left( 1 - \Delta \right)  \right]^{-1/2} 
\end{align*}
}
with probability at least $1 - 2 \exp(-2 \sqrt{Nr}) - \exp \left( -\frac{Nr}{8} \right)$, where $\Delta = \frac{ 1 + (r/N)^{1/4} + (r/N)^{1/2} }{ 1 - (r/N)^{1/4} - (r/N)^{1/2} }$.
Note that if $r \ll N$, then $| \Delta - 1| =  o(1)$.
This inequality combined with Lemma \ref{lem:balance.2} 
completes the proof. 

\end{proof}

\subsection{Proof of Proposition \ref{prop:geo_decay_finite_sample}: Geometrically Decaying Singular Values}\label{sec:appendix_geo_decay_finite_sample}

\begin{proof} [Proof of Proposition \ref{prop:geo_decay_finite_sample}]
Recall the (slightly simplified) bound of Corollary \ref{cor:training_pcr_generic} is 
\begin{align} \label{eq:mse_upper_generic_refined.rep}
	\emph{MSE}_{\Omega}(\hY) & \le
	 \frac{C' \| \beta^* \|_1^2 }{\rho^4} \left(\frac{k}{n} + \frac{n \vee p}{ (\tau_k - \tau_{k+1})^2} \right) \log^5(np) 
	 + \frac{3 \| \beta^* \|_1^2}{n}\| \bA^k - \bA \|_{2, \infty}^2
	+  \frac{20}{n} \| \phi \|_2^2,
\end{align} 
where $C' = C (1+\sigma^2)(1+\gamma^2)(1+K^4_\alpha)$ and $C>0$ is an absolute constant. 

\noindent Let us evaluate each of the first four terms in the right hand side of \eqref{eq:mse_upper_generic_refined.rep} to reach the desired \eqref{eq:geo_decay}.

\paragraph{First term.} Due to choice of $k$ we immediately have
follows that it is 
\begin{align}\label{eq.term1}
 \frac{C'  \| \beta^* \|_1^2 }{\rho^4} \log^5(np) \frac{k}{n}   & \le  \frac{C' C'(\theta) \| \beta^* \|_1^2 }{\rho^4} \frac{C_2 \log^6(np)}{n}.
\end{align}

\paragraph{Second term.}  
\begin{align}
 \frac{C' \| \beta^* \|_1^2 }{\rho^4} \frac{n \vee p}{ (\tau_k - \tau_{k+1})^2 }\log^5(np)
& \le  \frac{C' \| \beta^* \|_1^2 }{\rho^4}  \frac{n \vee p}{( \sqrt{Np}(\theta^{k-1} - \theta^{k}) )^{2}}\log^5(np)
\\&= \frac{C' \| \beta^* \|_1^2 }{\rho^4} \frac{n \vee p}{Np (  \theta^{k-1} (1 - \theta) )^{2}}\log^5(np)
\\& \le  \frac{C' \| \beta^* \|_1^2 }{\rho^4} C(\theta)\frac{1}{ n \wedge p} \frac{1}{ \theta^{2k} }\log^5(np)
\\& \le  \frac{C' \| \beta^* \|_1^2 }{\rho^4} C(\theta) \frac{1}{(n \wedge p)^{1/2}} \log^5(np)
\end{align}
where we have used the fact that $\tau_i = \tau_1 \theta^{i-1}$ for $i \geq 1$, $\tau_1 = C_1 \sqrt{Np}$, 
$n = \Theta(N)$ and $C(\theta) > 0$ is a term that depends only on $\theta$. 

\paragraph{Third term.} The goal is to bound $\|\bbA^k - \bbA\|_{2,\infty}^2$. With notation
$\bbE = \bbA - \bbA^k$, this is equivalent to bounding $\max_{j \in [p]} \|\bbE_{\cdot, j}\|_2^2$.  
With $\bbA = \sum_{i=1}^N \tau_i \mu_i \nu_i^T$ where $\mu_i \in \Reals^N$, $\nu_i \in \Reals^p$ for $i \in [N]$, 
for any $j \in [p]$, we have 
\begin{align*}
\frac{1}{n}\norm{\bbE_{\cdot, j}}^2
	&= \frac{1}{n} \bigg\| \bigg(\sum^N_{i = k + 1} \tau_i \mu_i \nu^T_i \bigg) e_j \bigg\|^2 
	= \frac{1}{n} \bigg\|\sum^N_{i = k + 1} \tau_i \mu_i (\nu^T_i e_j) \bigg\|^2 \\
	&\stackrel{(a)}= \frac{1}{n} \sum^N_{i = k + 1} \tau_i^2  (\nu^T_i e_j)^2 \\
	&\stackrel{(b)}\le \frac{1}{n} \sum^N_{i = k + 1} \tau^2_1 \theta^{2(i -1)}  (\nu^T_i e_j)^2  \\
	&\stackrel{(c)}\le \frac{C_1 Np}{n} \sum^N_{i = k + 1}  \theta^{2(i -1)}  (\nu^T_i e_j)^2 \\
	&\stackrel{(d)}\le \frac{C_1 Np}{np} \sum^N_{i = k + 1}  \theta^{2(i -1)} \\
	&\stackrel{(e)}\le C \theta^{2k} 
	~ \stackrel{(f)} \leq \frac{C}{(n \wedge p)^{1/2}}
 \end{align*}
Here, (a) follows from the orthonormality of the (left) singular vectors; (b) follows from $\tau_i = \tau_1 \theta^{i-1}$; 
(c) follows from $\tau_1 =C_1 \sqrt{Np}$; (d) `incoherence' property of singular vector, i.e. $\nu_i^T e_j = O(1/\sqrt{p})$
for all $i, j \in [p]$; (e) follows from property of geometric series for some absolute constant $C > 0$; and
(f) follows from choice of $k$. 
\paragraph{Concluding the proof.} The final term is repeat of $\frac{20}{n} \| \phi \|_2^2$. Therefore, putting all of the above together, 
the proof concludes. 
\end{proof}

\subsection{Geometrically Decaying Singular Values - Example from Signal Processing} \label{sec:proof_geo_decay}
%
%
%
As an illustration, we construct a matrix, popular in signal processing, which satisfies the conditions on the spectrum laid out in Proposition \ref{eq:geo_decay}.
We will construct an example based on the incoherence between the canonical basis and the Discrete Fourier Transform (DFT) basis. 

Suppose that $\bA = \bU \bSigma \bV^T $, where:
(i) $\bSigma$ is a diagonal matrix such that $\Sigma_{11} = C\sqrt{Np}$ for some $C > 0$ and the diagonal entries of $\bSigma$ satisfy $0 \leq \Sigma_{i+1, i+1}/\Sigma_{i,i} \leq \theta$ for all $i \in [ N \wedge p - 1]$ and for some $\theta \in (0,1)$; 
(ii) $\bU \in \Reals^{N \times N}$ is a DFT matrix such that $U_{ij} = (1/\sqrt{N})\cdot \exp( 2 \pi \imag (i-1)(j-1) / N)$ for all $i, j \in [N]$, where $\imag$ denotes the imaginary unit; 
(iii) $\bV \in \Reals^{p \times p}$ is a DFT matrix such that $V_{ij} = (1/\sqrt{p}) \cdot \exp(2 \pi \imag (i-1)(j-1) / p)$ for all $i, j \in [p]$.

The entries of the resulting matrix $\bA$ are complex numbers, but one could also construct $\bA$ by taking $\bU$ and $\bV$ as discrete cosine (or sine) transform matrices. Further, observe that $\bU$ and $\bV$ are orthogonal matrices; hence, $\sigma_i(\bA) = \sigma_i(\bSigma)$ for all $i \in [N \wedge p]$.
Finally, to show $\bA$ fits within our setting, we argue $\| \bA \|_\infty \, \leq C'$ for some constant $C' > 0$. 
%
\begin{proposition} \label{lemma:dg_decaying_sv}
Let $\bA$ be generated as above. 
Then, $\| \bA \|_\infty \, \leq C / (1-\theta)$.
Here, $C > 0$ and $\theta \in (0,1)$ are the constants that appear in the description of $\bSigma$.
Further, we have $v_i^T e_j = O(1/\sqrt{p})$ for all $i, j \in [p]$.
\end{proposition}
\noindent Proof of Proposition \ref{lemma:dg_decaying_sv} can be found in Appendix \ref{sec:proof_geo_decay}.

\begin{proof}[Proof of Proposition \ref{lemma:dg_decaying_sv}]
For $(i,j) \in [N] \times [p]$, we have $\bA_{ij} = \sum_{k=1}^{N \wedge p} \Sigma_{kk} U_{ik} V_{jk}$. Thus,
\begin{align*}
	| A_{ij} |	&= \left| \sum_{k=1}^{N \wedge p} \Sigma_{kk} U_{ik} V_{jk} \right|
		\leq	\sum_{k=1}^{N \wedge p} \Sigma_{kk} | U_{ik}| | V_{jk}|\\
		&\stackrel{(a)}{\leq}	\sum_{k=1}^{N \wedge p} \Sigma_{11} \theta^{k-1}\frac{1}{\sqrt{Np}}
		= \Sigma_{11} \frac{1 - \theta^{N \wedge p}}{1 - \theta} \frac{1}{\sqrt{Np}}\\
		&\stackrel{(b)}{\leq} \frac{C}{1 - \theta}.
\end{align*}
Here, (a) follows from that $| U_{ik} | = \frac{1}{\sqrt{N}}$, $| V_{jk} | = \frac{1}{\sqrt{p}}$, and $\Sigma_{kk} \leq \Sigma_{11} \theta^{k-1}$; 
and (b) follows from the assumption $\Sigma_{11} = C \sqrt{Np}$ and that $1 - \theta^{N\wedge p} \leq 1$.
\end{proof}

{

\section{Proof of Propositions \ref{prop:lvm}, \ref{prop:linear_comb}}\label{sec:syn_control_proofs}

\subsection{Proof of Proposition \ref{prop:lvm}}\label{sec:lvm_low_rank_proof}

This analysis is taken from \cite{xu2017rates} and is stated for completeness.

\medskip
\noindent \textbf{Step 1: Partitioning the space $[0,1)^{K}$.}
Let $\cE$ denote a partition of the cube $[0,1)^{K}$ into a finite number (denoted by $|\cE|$) of cubes $\Delta$. 
Let $\ell \in \Nb$. 
We say $P_{\cE,\ell}: [0,1]^K \to \Rb $ is a piecewise polynomial of degree $\ell$ if
\begin{align}\label{eq:def_piecewise_poly}
	P_{\cE,\ell} (\theta) = \sum_{\Delta \in \cE} P_{\Delta,\ell} (\theta) \mathbb{1}(\theta \in \Delta),  
\end{align}
where $P_{\Delta,\ell} (\theta): [0,1]^K \to \Rb $ denotes a polynomial of degree at most $\ell$. 

It suffices to consider an equal partition of $[0,1)^K$. 
More precisely, for any $k \in \Nb$, we partition the the set $[0,1)$into $1 / k$ half-open intervals of lengths $1/ k$,
i.e, 
$
[0, 1) = \cup_{i=1}^k \left[ (i-1)/k, i/k \right).
$
It follows that $[0,1)^{K}$ can be partitioned into $k^{K}$ cubes of forms $\otimes_{j=1}^{K}  \left[ (i_j-1)/k,   i_j/k \right)$ with $i_j \in [k]$.
Let $\cE_k$ be such a partition with $I_1, I_2, \ldots, I_{k^{K}} $ denoting all such cubes and  $z_1, z_2, \ldots, z_{k^{K}} \in \Rb^{K} $ denoting the centers of those cubes. 

\medskip
\noindent \textbf{Step 2: Taylor Expansion of $g(\cdot, \rho_j)$.}
For Step 2 of the proof, to reduce notational overload, we suppress dependence of $\rho_j$ on $g$, i.e.let $g(\cdot) = g(\cdot, \rho_j)$. 

For every $I_i$ with $1 \le i \le k^K$, define $P_{I_i, \ell} (\theta) $ as the degree-$\ell$ Taylor's series expansion of $g(\theta)$ at point $z_i$:
\begin{align}\label{eq:Taylor_series}
	P_{I_i, \ell} (\theta) =\sum_{\kappa: |\kappa| \le \ell} \frac{1}{\kappa ! } \left(\theta-z_i \right)^\kappa \nabla_\kappa g( z_i ), 
\end{align}
where $\kappa=(\kappa_1,\ldots, \kappa_K)$ is a multi-index with $\kappa!=\prod_{i=1}^d \kappa_i!$, and $\nabla_k g(z_i)$ is the partial derivative defined in \eqref{eq:def_partial_derivative}.
Note similar to $g$, $P_{I_i, \ell} (\theta)$ really refers to $P_{I_i, \ell} (x, \rho_j)$
Now we define a degree-$\ell$ piecewise polynomial as in \eqref{eq:def_piecewise_poly}, i.e., 
\begin{align}\label{eq:piece_wise_polynomial}
P_{\cE_k, \ell} (\theta) = \sum_{i=1}^{k^{K}}  P_{I_i, \ell} (\theta) \mathbb{1}(\theta \in I_i). 
\end{align}

\noindent For the remainder of the proof, let $\ell = \lfloor \zeta \rfloor$.
Since $g(\cdot, \rho_j) \in \cH(\zeta,L)$, it follows from the  that 
\begin{align*}
&\sup_{ \theta \in [0, 1)^{K} } \left| g(\theta) - P_{\cE_k, \ell} (\theta) \right| \\
&= \sup_{1 \le i \le k^{K}} \sup_{\theta \in I_i}  \left| g(\theta) - P_{I_i, \ell} (\theta) \right| \\
&\stackrel{(a)} = \sup_{1 \le i \le k^{K}} \sup_{\theta \in I_i} \left| \sum_{\kappa:|\kappa| \le \ell -1} \frac{\nabla_\kappa g ( z_i )}{\kappa !}(\theta - z_i)^\kappa 
			  + \sum_{\kappa:|\kappa| = \ell} \frac{\nabla_\kappa g ( z_i^' )}{\kappa !}(\theta - z_i)^{\kappa}
			  - P_{I_i, \ell} ( \theta ) \right| \\
&= \sup_{1 \le i \le k^{K}} \sup_{\theta \in I_i} \left| \sum_{\kappa:|\kappa| \le \ell -1} \frac{\nabla_\kappa g ( z_i )}{\kappa !}(\theta - z_i)^\kappa 
			\pm  \sum_{\kappa:|\kappa| = \ell} \frac{\nabla_\kappa g ( z_i )}{\kappa !}(\theta - z_i)^{\kappa} 
			+ \sum_{\kappa:|\kappa| = \ell} \frac{\nabla_\kappa g ( z_i^' )}{\kappa !}(\theta - z_i)^{\kappa} 
			- P_{I_i, \ell}  (\theta) \right| \\
&= \sup_{1 \le i \le k^{K}} \sup_{\theta \in I_i} \left| \sum_{\kappa:|\kappa| \le \ell} \frac{\nabla_\kappa g ( z_i )}{\kappa !}(\theta - z_i)^\kappa 
			+ \sum_{\kappa:|\kappa| = \ell} \frac{\nabla_\kappa g ( z_i^' ) - \nabla_\kappa g ( z_i )}{\kappa !}(\theta - z_i)^{\kappa} 
			- P_{I_i, \ell}  (\theta) \right| \\
&= \sup_{1 \le i \le k^{K}} \sup_{\theta \in I_i} \, \left| \sum_{\kappa:|\kappa| = \ell} \frac{\nabla_\kappa g ( z_i^' ) - \nabla_\kappa g ( z_i )}{\kappa !}(\theta - z_i)^{\kappa} 
			 \right| \\
&\stackrel{(b)} \le  \sup_{1 \le i \le k^{K}}  \sup_{\theta \in I_i} \|\theta - z_i \|_\infty^\ell \, \sup_{\theta \in I_i} 
\sum_{\kappa: |\kappa|=\ell} \frac{1}{\kappa!} \left| \nabla_\kappa g ( z_i^' ) -\nabla_{\kappa} g(z_{i} ) \right|  \\
&\stackrel{(c)} \le  \cL  \sup_{1 \le i \le k^{K}} \sup_{\theta \in I_i} \|\theta - z_i \|_\infty^{\zeta} =  \cL  k^{-\zeta}.
\end{align*}
where (a) follows from multivariate's version of Taylor's theorem (and using the Lagrange form for the remainder) and $z^'_i \in [0, 1)^K$ is a vector that can be represented as $z^'_i = (1-c) z_i + c x$ for $c \in (0, 1)$; (b) follows from Holder's inequality; (c) follows from Definition \ref{def:holder}.

\medskip
\noindent \textbf{Step 3: Construct Low-Rank Approximation of $\bA'$ Using $P_{\cE_k, \ell} (\cdot, \rho_j)$.}
Recall $\bA'_{ij} = g(\theta_i, \rho_j)$, and $g(\cdot, \rho_j) \in \cH(\zeta, \cL)$. 
We now construct a low-rank approximation of it using $P_{I_i, \ell} (\cdot, \rho_j)$. 
Define $\bA^{\text{(lr)}} \in \Rb^{N \times p}$, where $\bA^{\text{(lr)}}_{ij} = P_{\cE_k, \ell}  (\theta_i, \rho_j)$.

By Step 2, we have that for all $i \in [N], j \in [p]$, 
\[
	\Big| \bA'_{ij} - \bA^{\text{(lr)}}_{ij} \Big| \le \cL k^{-\zeta}
\]
It remains to bound the rank of $\bA^{\text{(lr)}}$. Note that since $P_{\cE_k, \ell}  (\theta_i, \rho_j)$ is a piecewise polynomial of degree $\ell = \lfloor \zeta \rfloor$, it has a decomposition of the form
\[
	\bA^{\text{(lr)}}_{ij} = P_{\cE_k, \ell}  (\theta_i, \rho_j) = 
	\sum_{i=1}^{k^{K}} \langle \Phi(\theta), \beta_{I_i, s} \rangle 
	\mathbb{1}(\theta \in I_i)
\]
where the vector 
\[
	\Phi(\theta) = \Big(1, \theta_1, \dots, \theta_K, \dots, \theta_1^\ell, \dots, \theta_K^\ell \Big)^T,
\]
i.e., is the vector of all monomials of degree less than or equal to $\ell$. The number of such monomials is easily show to be equal to $C(\zeta, K) := \sum^{\lfloor \zeta \rfloor}_{i=0} { i + K  - 1 \choose K - 1}$. 

\noindent Thus the rank of $\bA^{\text{(lr)}}$ is bounded by $k^K C(\zeta, K)$. Setting $k = 1/\delta$ completes the proof.

\subsection{Proof of Proposition \ref{prop:linear_comb}}\label{sec:linearity_exists}

\noindent Let $\bA^{\text{(lr)}}$ and $\beta^{*}$ be defined as in Property \ref{glm_linearity}. 
Then,
\begin{align*} 
	| A'_{i0} - \sum_{k=1}^{r} \beta^*_k \cdot A'_{ik}| &= | A'_{i0} \pm \bA^{\text{(lr)}}_{i0} - \sum_{k=1}^{r} \beta^*_k \cdot A'_{ik} \pm \sum_{k=1}^{r} \beta^*_k \cdot \bA^{\text{(lr)}}_{ik}|
	\\ &\le | A'_{i0} - \bA^{\text{(lr)}}_{i0} | + | \sum_{k=1}^{r} \beta^*_k \cdot A'_{ik} - \sum_{k=1}^{r} \beta^*_k \cdot \bA^{\text{(lr)}}_{ik}| + | \bA^{\text{(lr)}}_{i0} - \sum_{k=1}^{r} \beta^*_k \cdot \bA^{\text{(lr)}}_{ik}| 
	\\ &= | A'_{i0} - \bA^{\text{(lr)}}_{i0} | + | \sum_{k=1}^{r} \beta^*_k \cdot A'_{ik} - \sum_{k=1}^{r} \beta^*_k \cdot \bA^{\text{(lr)}}_{ik}| 
	\\ &\le  | A'_{i0} - \bA^{\text{(lr)}}_{i0} | + \sum_{k=1}^{r} |  \beta^*_k \cdot A'_{ik} -  \beta^*_k \cdot \bA^{\text{(lr)}}_{ik}| 
	\\ &\le C (r +1) \cL \cdot \delta^\zeta
\end{align*}
By Property \ref{glm_linearity}, we have $r \le C(\zeta, K) \Big(\dfrac{1}{\delta}\Big)^K$, which completes the proof.
}

\subsection{Proof of Theorem \ref{thm:mse_sc}}\label{sec:mse_sc_proof}
\begin{proof}
The bound in Theorem \ref{thm:mse_sc}, given by \eqref{eq:mse_upper_generic_refined_LVM_detailed-SC}, is a sum of the pre-intervention error term and the additional penalty paid for the generalization error in the post-intervention period.
The first term,
$$
\frac{C' C(\zeta, K) \cL^{2} \|  \beta^*\|_1^2}{\rho^4} \left(\frac{1}{(n \wedge p)^{{1 - \frac{K}{2\zeta}}}}  \right) \log^5(np) 
$$
comes due to the pre-intervention error and the it follows immediately from Corollary \ref{cor:LVM-spectra-error}.
The second term,
$$
\frac{ C'''  k^{5/2} } {\sqrt{n}} \| \beta^* \|_1
$$
comes due to the generalization error of RSC/PCR for the post-intervention period. 
The proof of this bound on the generalization error of RSC/PCR follows in an identical fashion to Theorem \ref{thm:test_pcr} -- the only change in the proof of Theorem \ref{thm:test_pcr} is that wherever an expectation over $\Omega$ was taken, we appropriately substitute it by taking an expectation over $\Theta$, the latent distribution from which $\theta_i$ is sampled.
\end{proof}

\end{appendix}

\end{document}